\documentclass{article}

\usepackage{natbib}
\usepackage[nonatbib, preprint]{neurips_2022}
\usepackage{bbm}
\usepackage{mathtools}
\usepackage{natbib}
\usepackage{verbatim}
\usepackage{comment}
\usepackage[linesnumbered, ruled,vlined, algo2e]{algorithm2e}
\usepackage{wrapfig}

\usepackage{microtype}
\usepackage{graphicx}
\usepackage{booktabs} %

\usepackage{collectbox}

\usepackage{hyperref}

\usepackage{tikz}
\usetikzlibrary{intersections,shapes.arrows}

\usepackage{amsmath,amssymb, amsfonts, amsthm}

\newtheorem{assumption}{Assumption}[section]

\usepackage{comment}
\usepackage{xcolor}
\usepackage{mathrsfs}
\usepackage{enumitem}
\usepackage{bbm}
\usepackage{mdframed}

\usepackage{pythonhighlight}

\usepackage{prettyref}
\newcommand{\pref}[1]{\prettyref{#1}}

\newcommand{\savehyperref}[2]{\texorpdfstring{\hyperref[#1]{#2}}{#2}}
\newrefformat{eq}{\savehyperref{#1}{\textup{(\ref*{#1})}}}
\newrefformat{eqn}{\savehyperref{#1}{Equation~\ref*{#1}}}
\newrefformat{con}{\savehyperref{#1}{Conjecture~\ref*{#1}}}
\newrefformat{lem}{\savehyperref{#1}{Lemma~\ref*{#1}}}
\newrefformat{def}{\savehyperref{#1}{Definition~\ref*{#1}}}
\newrefformat{line}{\savehyperref{#1}{line~\ref*{#1}}}
\newrefformat{thm}{\savehyperref{#1}{Theorem~\ref*{#1}}}
\newrefformat{cor}{\savehyperref{#1}{Corollary~\ref*{#1}}}
\newrefformat{sec}{\savehyperref{#1}{Section~\ref*{#1}}}
\newrefformat{app}{\savehyperref{#1}{Appendix~\ref*{#1}}}
\newrefformat{ass}{\savehyperref{#1}{Assumption~\ref*{#1}}}
\newrefformat{ex}{\savehyperref{#1}{Example~\ref*{#1}}}
\newrefformat{fig}{\savehyperref{#1}{Figure~\ref*{#1}}}
\newrefformat{tab}{\savehyperref{#1}{Table~\ref*{#1}}}
\newrefformat{alg}{\savehyperref{#1}{Algorithm~\ref*{#1}}}
\newrefformat{rem}{\savehyperref{#1}{Remark~\ref*{#1}}}
\newrefformat{conj}{\savehyperref{#1}{Conjecture~\ref*{#1}}}
\newrefformat{prop}{\savehyperref{#1}{Proposition~\ref*{#1}}}
\newrefformat{proto}{\savehyperref{#1}{Protocol~\ref*{#1}}}
\newrefformat{prob}{\savehyperref{#1}{Problem~\ref*{#1}}}
\newrefformat{claim}{\savehyperref{#1}{Claim~\ref*{#1}}}

\usepackage{cdann_definitions}

\newmdtheoremenv{test}{Test}

\definecolor{DarkRed}{rgb}{0.75,0,0}
\definecolor{DarkGreen}{rgb}{0,0.5,0}
\definecolor{DarkPurple}{rgb}{0.5,0,0.5}
\definecolor{DarkBlue}{rgb}{0,0,0.7}

\hypersetup{colorlinks=true, plainpages = false, citecolor = DarkBlue, urlcolor = DarkBlue, filecolor = black, linkcolor =DarkBlue}

\makeatletter
\newcommand{\pushright}[1]{\ifmeasuring@#1\else\omit\hfill$\displaystyle#1$\fi\ignorespaces}
\newcommand{\pushleft}[1]{\ifmeasuring@#1\else\omit$\displaystyle#1$\hfill\fi\ignorespaces}
\makeatother

\def\ddefloop#1{\ifx\ddefloop#1\else\ddef{#1}\expandafter\ddefloop\fi}
\def\ddef#1{\expandafter\def\csname bb#1\endcsname{\ensuremath{\mathbb{#1}}}}
\ddefloop ABCDEFGHIJKLMNOPQRSTUVWXYZ\ddefloop
\def\ddefloop#1{\ifx\ddefloop#1\else\ddef{#1}\expandafter\ddefloop\fi}
\def\ddef#1{\expandafter\def\csname b#1\endcsname{\ensuremath{\mathbf{#1}}}}
\ddefloop ABCDEFGHIJKLMNOPQRSTUVWXYZ\ddefloop
\def\ddef#1{\expandafter\def\csname c#1\endcsname{\ensuremath{\mathcal{#1}}}}
\ddefloop ABCDEFGHIJKLMNOPQRSTUVWXYZ\ddefloop
\def\ddef#1{\expandafter\def\csname h#1\endcsname{\ensuremath{\widehat{#1}}}}
\ddefloop ABCDEFGHIJKLMNOPQRSTUVWXYZabcdefghijklmnopqrsuvwxyz\ddefloop    %
\def\ddef#1{\expandafter\def\csname hc#1\endcsname{\ensuremath{\widehat{\mathcal{#1}}}}}
\ddefloop ABCDEFGHIJKLMNOPQRSTUVWXYZ\ddefloop
\def\ddef#1{\expandafter\def\csname t#1\endcsname{\ensuremath{\widetilde{#1}}}}
\ddefloop ABCDEFGHIJKLMNOPQRSTUVWXYZ\ddefloop
\def\ddef#1{\expandafter\def\csname tc#1\endcsname{\ensuremath{\widetilde{\mathcal{#1}}}}}
\ddefloop ABCDEFGHIJKLMNOPQRSTUVWXYZ\ddefloop

\usepackage{thmtools} 
\usepackage{thm-restate}

\newcommand{\poly}{\text{poly}}
\usepackage{ifthen}
\newcommand{\alg}[1][]{\ifthenelse{\equal{#1}{}}{\bbA}{\bbA_{#1}}}
\newcommand{\master}{\bbM}
\newcommand{\env}{\bbB}
\newcommand{\regret}{\operatorname{Reg}}

\newcommand{\pseudoregret}{\operatorname{PseudoReg}}

\newcommand{\maxrew}{\operatorname{MaxRew}}

\newcommand{\tildecrew}{\widetilde{\operatorname{CRew}}}

\allowdisplaybreaks[2]

\newcommand{\crew}{\mathrm{CRew}}

\newcommand{\indicator}[1]{\mathbf{1}\left\{ #1\right\}}
\newcommand{\conc}[1][]{\ifthenelse{\equal{#1}{}}{\mathrm{D}}{\mathrm{D}_{#1}}}

\newcommand{\postdeadline}[1]{#1} %

\newcommand\numberthis{\addtocounter{equation}{1}\tag{\theequation}}

\usepackage[utf8]{inputenc} %
\usepackage[T1]{fontenc}    %
\usepackage{hyperref}       %
\usepackage{url}            %
\usepackage{booktabs}       %
\usepackage{amsfonts}       %
\usepackage{nicefrac}       %
\usepackage{microtype}      %
\usepackage{xcolor}         %

\title{Best of Both Worlds Model Selection}

\author{%
 Aldo Pacchiano \\
  Microsoft Research, NYC \\
  \texttt{apacchiano@microsoft.com} \\
  \And
  Christoph Dann \\
  Google, NYC \\
  \texttt{cdann@cdann.net}\\
  \And
  Claudio Gentile \\
  Google, NYC \\
  \texttt{cgentile@google.com}
}

\begin{document}

\maketitle

\begin{abstract}
We study the problem of model selection in bandit scenarios in the presence of nested policy classes, with the goal of 
obtaining simultaneous adversarial and stochastic (``best of both worlds") high-probability regret guarantees. 
Our approach requires that each base learner comes with a candidate regret bound that may or may not hold, while our meta algorithm plays each base learner according to a schedule that keeps the base learner's candidate regret bounds balanced until they are detected to violate their guarantees. We develop careful mis-specification tests specifically designed to blend the above model selection criterion with the ability to leverage the (potentially benign) nature of the environment.
We recover the model selection guarantees of the CORRAL~\citep{agarwal2017corralling} algorithm for adversarial environments, but with the additional benefit of achieving high probability regret bounds, specifically in the case of nested adversarial linear bandits. More importantly, our model selection results also hold simultaneously in stochastic environments under gap assumptions. These are the first theoretical results that achieve best of both world (stochastic and adversarial) guarantees while performing model selection in (linear) bandit scenarios. 
\end{abstract}

\section{Introduction}
A fundamental challenge in sequential decision-making is the ability of the learning agent to adapt to the unknown properties of the environment they interact with. While adversarial environments may require some caution, we would also like to leverage situations where more benign scenarios may be disclosed as a result of this interaction. 

In the literature on multi-armed bandits, this adaptation capabilities has often taken two forms:\\
\textbf{Best-of-both worlds guarantees} which were pioneered by \cite{bs12} and subsequently studied by a number of authors \citep[e.g.][]{ss14,ac16,sl17,wl18,zs19,DBLP:conf/icml/LeeLWZ021}. Here, the goal is to design algorithms achieving both stochastic and adversarial environment guarantees simultaneously, without knowing the type of environment in advance.
Similar in spirit is the stream of literature on stochastic rewards with adversarial corruptions \citep{lmpl18,gkt19,zs21,DBLP:conf/icml/LeeLWZ021,wei2021model}, where an adversary is assumed to corrupt the stochastic rewards observed by the algorithm, and the regret guarantees are expected to degrade gracefully with the total amount of corruption, without knowing this amount in advance.\\
\textbf{Model selection guarantees} which were initiated by \cite{agarwal2017corralling}, and subject since then to intense investigations \citep[e.g.][]{foster2019model,abbasi2020regret,pacchiano2020model,arora2020corralling, ghosh2020problem,chatterji2020osom, bibaut2020rate,foster2020adapting,lee2020online,pacchiano2020regret,pmlr-v139-cutkosky21a}. Here, we assume we have access to a pool of $M$ base bandit algorithms each operating, say, within a different class of models or under different assumptions on the environment, and the aim is to design a bandit meta-algorithm that learns to simulate the best base algorithm in hindsight, without knowing in advance which one will be best for the environment at hand. This approach has often been used in the literature to capture bandit model selection problems. In fact, a natural instantiation of this framework is when we have a sequence of nested policy classes,
and the goal is to single out the best policy %
within this nested family, by paying as price only the complexity %
of the policy class 
the optimal policy 
falls into.
In some sense, Item 2 is more general than Item 1, since one may attempt to achieve a best-of-both-world performance by pooling a stochastic bandit algorithm with an adversarial bandit one, and expect the meta-algorithm on top of them to eventually learn to follow one of the two. Similarly, in the setting of stochastic rewards with adversarial corruptions, one can pool together stochastic algorithms operating with increasing guessed levels of corruptions, and let the meta-algorithm learn to single out the one corresponding to the true corruption level.\footnote
{
Recent relevant papers, working with model mis-specification instead of reward corruption, include \cite{pacchiano2020model,foster2020adapting}.
}

In this paper, we combine the two items above into a bandit algorithm that exhibits both model selection and best-of-both worlds regret guarantees simultaneously. Our framework encompasses in particular a well-known linear bandit model selection scenario, 
where the action set $\mathcal{A}$ is a (finite but large) subset of $\mathbb{R}^{d_M}$, for some maximal dimension $d_M>0$, and we deal with a hierarchy of possible dimensions
$d_1 < \ldots < d_M$. 
At time $t$, the learner plays an action $\mathbf{a}_t \in \mathcal{A}$ and receives as reward either $r_t = \mathbf{a}_t^\top \boldsymbol{\omega}_t$, where $\boldsymbol{\omega}_t \in \mathbb{R}^{d_{M}}$ is an adversarially-generated reward vector (adversarial setting), or $r_t = \mathbf{a}_t^\top \boldsymbol{\omega}$ + white noise, where $\boldsymbol{\omega} \in \mathbb{R}^{d_{M}}$ is a fixed but unknown reward vector (stochastic setting).
Associated with dimensions $d_1 < \ldots < d_M$ are a nested family of policy classes $\Pi_1 \subset \ldots \subset \Pi_M$, and $M$ base algorithms $\alg[1], \ldots, \alg[M]$, where the $i$-th algorithm $\alg[i]$ is a linear bandit algorithm that works under the assumption that only the first $d_{i}$ components of $\boldsymbol{\omega}_t$ (or $\boldsymbol{\omega}$) are nonzero. Hence $\alg[i]$ operates with dimension $d_i$ in that all policies $\pi \in \Pi_i$ are probability distributions which are projections over action set $\mathcal{A}_i$ of the set of probability distributions in $\Pi_M$. Here, 
$\mathcal{A}_i$ is the projection of the full-dimensional action set $\mathcal{A}$ onto its first $d_i$ dimensions. Notice that
this implies that the policy classes are nested: $\Pi_1 \subseteq \Pi_2 \ldots \subseteq \Pi_M$.
If only the first $d_{i_\star}$ dimensions ($i_\star$ being unknown to the learner) of each $\boldsymbol{\omega}_t$ (adversarial setting) or 
$\boldsymbol{\omega}$ (stochastic setting) are nonzero, we design an algorithm whose regret bound scales as ${\mbox{poly}}(d_{i_\star})\sqrt{T}$ in the adversarial case and, simultaneously, as $\frac{{\mbox{poly}}(d_M)\log T}{\Delta}$ if the environment happens to be stochastic.
Specific consideration is given to the dependence on $d_{i_\star}$. We show through a lower bounding argument that in the stochastic case a guarantee of the form $\frac{d_M\,\log T}{\Delta}$, that is, where $d_M$ replaces the more desirable factor $d_{i_\star}$, is {\em inevitable}, if we want to insist on obtaining a $\frac{\log T}{\Delta}$-like result. Thus our bounds reflect the best achievable regret rates depending on $d_{M}$ in the stochastic case and on ${\mbox{poly}}(d_{i_\star})$ in the adversarial case.

In order to achieve these best-of-both-world results one cannot easily rely on general corralling techniques, like the one contained in \citep{agarwal2017corralling}, since the granularity offered by adversarial aggregation algorithms is no better than $\sqrt{T}$, which is not adequate for stochastic settings; hence our choice towards the model selection technique known as {\em regret balancing}. Yet, even in this case, the literature does not provide off-the-shelf solutions:
We first need to extend regret balancing from stochastic rewards \citep{pmlr-v139-cutkosky21a, pacchiano2020regret, abbasi2020regret} and corrupted stochastic rewards \citep{wei2021model} to adversarial rewards. This involves adding extra actions to the action space at a given level of the hierarchy so as to enable the base learner operating at that level to compete with higher levels.
This seems to be needed because in the adversarial case, base algorithms may even incur a negative regret, thus making it hard to compare
the relative performance of algorithms operating at different levels during the regret balancing operations. Specific technical hurdles arise in the linear bandit case, where mis-specification may cause low-level base learners to behave in a maliciously erratic way. In this case, mis-specification tests have to be designed with care
so as to ensure that base learners that are mis-specified but not yet eliminated incur manageable regret. At the same time, these tests should also detect as soon as possible if the environment is stochastic through ad hoc regret balancing gap estimation procedures.

For our results to hold, some technical conditions are required on the base algorithms, like a notion of (high probability) stability and a notion of action space extendability, formally defined in Section \ref{sec:setup}. We show that these conditions are fulfilled by known algorithms, like an anytime variant of the GeometricHedge.P algorithm in \cite{bartlett2008high}, and a high-probability variant of the EXP4 algorithm in \cite{a+03}.

\noindent{\bf Related work. }
Among the references we mentioned above, those which are most relevant to our work, as directly related to model selection and best-of-both-worlds guarantees are perhaps \citep{arora2020corralling,pmlr-v139-cutkosky21a,DBLP:conf/icml/LeeLWZ021,wei2021model}.
In both \cite{arora2020corralling} and \cite{pmlr-v139-cutkosky21a}, model selection regret guarantees for stochastic contextual bandits 
are given which take the form ${\mbox{poly}}(d_{i_\star})\sqrt{T}$. Yet, no best-of-both world model selection results are contained in those papers. Closer in spirit are 
\cite{DBLP:conf/icml/LeeLWZ021} and \cite{wei2021model}. In \cite{DBLP:conf/icml/LeeLWZ021}, the authors consider model selection problems on top of adversarially corrupted stochastic linear bandit problems, where the total level of corruption $C$ is unknown in advance.
It is worth stressing that this is a model selection problem which is substantially different (and actually easier, see Section \ref{s:advbalancing}) than ours, as the selection applies to $C$ instead of the dimensions $d_i$.
In fact, the model selection procedure in \cite{DBLP:conf/icml/LeeLWZ021} looks very different from ours, and can be roughly seen as a robust version of a doubling trick applied to $C$.
These results are extended in \cite{wei2021model}, where the authors also consider RL settings, and more general forms of corruptions than \cite{DBLP:conf/icml/LeeLWZ021}. Like ours, their model selection algorithm follows the idea of regret balancing and elimination of \cite{abbasi2020regret,pmlr-v139-cutkosky21a}.
On the lower bound side, relevant papers include \cite{zn21,mz21}, where the authors investigate the Pareto frontier of model selection for (contextual) stochastic bandits. It is shown in particular that a regret upper bound of the form %
$\sqrt{d_{i_\star}T}$ 
cannot be achieved. However, these papers do not explicitly cover gap-dependent regret guarantees.

\section{Problem Setting and Assumptions}\label{sec:setup}
We start off by defining the adversarial scenario. An adversarial contextual bandit problem is a repeated game between a learner $\alg$  and an environment $\bbB$.
We consider the general decision-making scenario where the learner $\alg$ has at its disposal a class of policies $\Pi_{\alg}$ made up of functions $\pi$ of the form $\pi : \mathcal{X} \rightarrow \Delta_{\mathcal{A}}$, where $\Delta_{\mathcal{A}}$ denotes the set of probability distributions over ${\mathcal A}$. At each round $t$, the interaction between $\alg$ and $\bbB$ is as follows:
\begin{enumerate}
    \item Learner $\alg$ selects a policy $\pi_t \in \Pi_{\alg}$ 
    \vspace{-0.03in}
    \item Simultaneously, the environment $\bbB$ selects context $x_t \in \mathcal{X}$ and reward function $r_t\,:\, \mathcal{A} \times \mathcal{X} \rightarrow [0,1]$, and reveals $x_t$ to $\alg$
    \vspace{-0.03in}
    \item Learner $\alg$ takes an action  $a_t \sim \pi_t(\cdot\, |\, x_t)$
    and observes reward $o_t = r_t(a_t,x_t)$. 
\end{enumerate}
The regret $\regret_{\alg}(\cT, \Pi')$ of algorithm $\alg$ in rounds $\cT \subseteq \bbN$ against a policy class $\Pi'$ is the difference between the learner's accumulated reward in $\cT$ and the performance of the best {\em fixed} policy from $\Pi'$:
\begin{equation*}
   \regret_{\alg}(\cT, \Pi') = \max_{\pi \in \Pi'} \sum_{\ell \in \cT} \left[\bbE_{a \sim \pi}\left[ r_\ell( a,x_\ell ) \right] - r_\ell(a_\ell,x_\ell)\right]~,
\end{equation*}
where $\bbE_{a \sim \pi}[r_\ell(a, x_\ell)] = \sum_{a \in \cA} \pi(a | x_\ell) r_\ell(a, x_\ell)$ denotes the expectation over the action drawn from the given policy $\pi$ which may itself be a random quantity.
For ease of notation, we will omit the comparator policy class when $\Pi' = \Pi_{\alg}$ is the policy class of $\alg$ and replace $\cT$ by $t$ when $\cT = [t] := \{1,\ldots, t\}$. Hence, $\regret_{\alg}(t)$ is the regret of $\alg$ against its own policy class up to round $t$.

The stochastic scenario we consider is similar to the above, except in the way contexts $x_t$ and reward values $r_t(a_t,x_t)$ are generated. Specifically, $\env$ generates contexts $x_t$ in an i.i.d. fashion according to a fixed (but arbitrary and unknown) distribution ${\mathcal D}$ over ${\mathcal X}$, while the reward $r_t(a_t,x_t)$ is such that for all fixed $(a,x) \in {\mathcal A}\times {\mathcal X}$, we have $\bbE[r_t(a,x)] = r(a,x)$, for some fixed function $r(\cdot,\cdot)$ in some known class of reward functions. In this case, we measure performance through pseudo-regret
\begin{equation*}
   \pseudoregret_{\alg}(\cT, \Pi') = \max_{\pi \in \Pi'} \sum_{\ell \in \cT}  \left[\bbE_{x \sim \mathcal{D},a \sim \pi}\left[r( a,x)\right] - \bbE_{x \sim \mathcal{D},a \sim \pi_\ell}\left[r(a,x)\right]\right]~,
\end{equation*}
where $\bbE_{x \sim \mathcal{D},a \sim \pi}$ is the expectation over contexts $x$ and the action $a \sim \pi(\cdot | x)$ drawn from a policy $\pi$ that may itself be a random quantity.
We say an environment $\env$ is stochastic with gap $\Delta > 0$ if there is a policy $\pi_\star \in \Pi$ such that
\begin{equation*}
    \mathbb{E}_{x\sim \mathcal{D}, a \sim \pi_\star} \left[ r(a,x)  \right] \geq \max_{\pi \in \Pi \backslash \{ \pi^*\} } \mathbb{E}_{x \sim \mathcal{D}, a \sim \pi} \left[ r(a,x) \right] + \Delta~.
\end{equation*}
A notable example of the above is the following linear bandit scenario. In the adversarial case (``adversarial linear bandits'') the action space ${\mathcal A}$ is a subset of $\mathbb{R}^{d}$, for some dimension $d$, the context $x_t$ is irrelevant and, upon playing action $\mathbf{a}_t \in {\mathcal A}$, the environment generates a reward which is a {\em linear} function of the actions, $r_t(\mathbf{a}_t,x_t) = r_t(\mathbf{a}_t) = \mathbf{a}_t^\top \boldsymbol{\omega}_t$, where $\boldsymbol{\omega}_t$ is chosen adversarially at every round within some known class of $d$-dimensional vectors. In the stochastic case (``stochastic linear bandits''), the only difference is that we simply have $r_t(\mathbf{a}_t,x_t) = r_t(\mathbf{a}_t) = \mathbf{a}_t^\top \boldsymbol{\omega} + \epsilon_t$, where $\boldsymbol{\omega}$ is a fixed and unknown vector in the class of $d$-dimensional vectors, and $\epsilon_t$ is a sub-Gaussian noise. Moreover, we have gap $\Delta$ if there is $\pi^\star \in \Pi$ such that
$
\bbE_{a \sim \pi^\star} \left[ \mathbf{a}^\top \boldsymbol{\omega}  \right] \geq \max_{\pi \in \Pi \backslash \{ \pi^\star\} } \mathbb{E}_{a \sim \pi} \left[ \mathbf{a}^\top \boldsymbol{\omega} \right] + \Delta~.
$

Many algorithms $\alg$ in the literature exist that enjoy a regret bound against their policy class $\Pi_{\alg}$, holding in all environments $\env$ satisfying certain conditions. 
If this regret bound holds, we say that $\alg$ is \emph{adapted} to this environment. Here, we allow algorithms to also compete against policy classes $\Pi'$ other than $\Pi_{\alg}$, and focus on algorithms that come with regret bounds of the following form.
\begin{definition}[Adapted]\label{def:adapted}
We call $\alg$ adapted to environment $\bbB$ and policy class $\Pi'$ if, with probability at least $1 - \delta$, the following bound holds simultaneously for all $t \in \bbN$:\footnote
{
Here and throughout, the $\Ocal$-notation only hides absolute constants.
}
\begin{equation*}
    \regret_{\alg}(t, \Pi') =\Ocal\left( R(\Pi_{\alg}) \sqrt{t \ln{ \frac{t}{\delta}} }\right)~.
\end{equation*}
The term $R(\Pi_{\alg})$ is a (known) measure of complexity of the policy class $\Pi_{\alg}$ used by $\alg$.
\end{definition}

Examples of algorithms satisfying \pref{def:adapted} include a version of the Geometric Hedge algorithm in linear bandits \citep{bartlett2008high} and the EXP4 algorithm \citep{a+03} in contextual bandits with finite action sets. We will discuss them in detail below.
Before introducing the model selection questions addressed in this work, we present two stronger versions of the above definition, which we call \emph{high-probability stability} and \emph{extendability}. These will be useful for model selection among multiple learners. As we will show with examples later, these stricter conditions can be established for several common settings under the same conditions for which adaptivity from \pref{def:adapted} is guaranteed.

\noindent{\bf Additional conditions.}
The first condition, \emph{high-probability stability} (or \emph{h-stability}) generalizes adaptivity to the case where $\alg$ only observes a certain noisy version of the reward. To state the condition formally, consider a more general interaction protocol between $\alg$ and the environment $\env$, where Step 3 above is replaced by
\begin{enumerate}
\vspace{-0.05in}
\item[3a.] Let $b_t \sim \operatorname{Bernoulli}(\rho)$ for a fixed and known $\rho \in (0, 1]$.
Learner $\alg$ takes an action  $a_t \sim \pi_t(\cdot\, |\, x_t)$ and observes an importance-sampled version of the reward 
$
o_t = b_t \frac{r_t(a_t,x_t)}{\rho}~.
$
\end{enumerate}
This encompasses the original protocol with $\rho = 1$ as a special case.
An algorithm is \emph{h-stable} if it maintains its regret guarantee up to a $1 / \sqrt{\rho}$ penalty in this more general interaction protocol:

\begin{definition}[h-stability]\label{def:high_probability_stability}
An algorithm $\alg$ is high probability stable \textbf{(h-stable)} in an environment $\bbB$ against policy class $\Pi'$ if it satisfies for any constant $\rho \in (0, 1]$ a regret bound
\begin{equation*}
    \regret_{\alg}(t, \Pi') 
    =
    \Ocal\left(R(\Pi_{\alg}) \sqrt{\frac{t}{\rho} \ln{ \frac{t}{\delta}} }\right)
\end{equation*}
that holds with probability at least $1-\delta$ simultaneously for all $t \in \bbN$. 
\end{definition}
Note that $\alg$ is adapted to $\env$ and $\Pi'$ if it is h-stable for $\rho = 1$. Compared to the notion of stability proposed in \cite{agarwal2017corralling}, the one in Definition \ref{def:high_probability_stability} is more demanding, in that it requires the importance-weighted regret bound to hold with high probability, rather than in expectation. Since we aim for high-probability regret bounds in adversarial and stochastic settings, this stronger notion is natural.

The second condition, \emph{extendability}, generalizes h-stability to environments with additional actions, as specified next.
Let $\bbB$ be the original environment with action set $\cA$ and let $\bar \cA_k = \cA \cup \{a'_1, a'_2, \dots, a'_{k}\}$ be the action set extended by $k$ extra special actions $a'_1, a'_2, \dots, a'_{k}$. 
Further, let $\bar \Pi =\{\pi\, \colon\, \cX \rightarrow \Delta_{\bar \cA_k}\}$ be an extended version of the original policy set $\Pi = \{ \pi\,\colon\, \cX \rightarrow \Delta_{\cA}\}$ that contains all policies of the original policy set and the single-action policies $\indicator{a'_i}$ that always choose a certain special action $a'_i$, i.e., $\bar \Pi \supseteq \Pi \cup \{ \indicator{a'_1}, \indicator{a'_2}, \dots, \indicator{a'_k}\}$.\footnote{Policies $\pi \in \Pi$ are naturally extended from $\Delta_{\cA}$ to $\Delta_{\bar \cA_k}$ by assigning probability $0$ to all special actions $a'_i$.} 
We further allow the environment on the extended action space $\bar \cA_k$ to choose any values for rewards $r_t(a'_i,x_t)$, possibly depending on the entire history and all rewards $r_t(a,x_t)$ assigned to other actions $a \neq a'_i$ in that round. We denote the set of all such extended environments by $\cB_k(\bbB, \bar \cA_k)$. An algorithm $\alg$ is now extendable if we can run it with the extended policy set $\bar \Pi_k$ and it competes well against $\Pi'$ and policies $\indicator{a'_i}$ in all such extended environments.

\begin{definition}[Extendability]
\label{def:extendable}
Consider an algorithm $\alg$ with policy set $\Pi_{\alg}  \subseteq \cX \to \Delta_{\cA}$ in environment $\bbB$. 
We call $\alg$ $k$-extendable in $\bbB$ against $\Pi'$ if there is an extended policy set $\bar \Pi_k \supseteq \Pi_{\alg} \cup \{ \indicator{a'_1}, \dots \indicator{a'_k}\}$ such that $\alg$ equipped with the extended policy set $\bar \Pi_k$ is h-stable against $\Pi'  \cup \{ \indicator{a'_1}, \dots \indicator{a'_k}\}$ in all environments $\bbB' \in \cB(\bbB, \bar \cA_k)$ that extend $\bbB$ from action space $\cA$ to action space $\bar \cA_k = \cA \cup \{a'_1, \dots, a'_k\}$.  
\end{definition}

One relevant example of h-stable and extendable algorithm working in the adversarial linear bandit scenario is an anytime variant of the Geometric Hedge algorithm from \cite{bartlett2008high} with exploration ruled by John's ellipsoid (e.g., \cite{bubeck2012regret}) -- see \pref{app:geometric_hedge}. Another example is a high-probability variant of the EXP4 algorithm from \cite{a+03}, which operates with finite sets of policies.

\subsection{Model selection and best-of-both-worlds regret}\label{sec:ms_setting}
\vspace{-0.05in}
Our model selection for best-of-both worlds regret guarantees can be described as follows.
We have a nested family of policy classes $\Pi_1 \subseteq \ldots \subseteq \Pi_M$, with (known) complexities $R(\Pi_1) \leq \ldots \leq R(\Pi_M)$. A meta-learning algorithm \master\, has access to $M$ base algorithms $\alg[1], \ldots, \alg[M]$, algorithm $\alg[i]$ operating with policy class $\Pi_i$. These algorithms we sometimes refer to as {\em base learners}.
Let $i_\star$ be the smallest index of the base learner that competes against the largest policy class $\Pi_M$ in the following sense:
\[
    i_\star = \min\left\{i \in [M] \colon  \alg[i] \textrm{ is $(M-i)$-extendable and h-stable in } \bbB \textrm{ against } \Pi_{M} \right\}~.
\]
The goal of model selection is to devise a meta-algorithm $\master$ that has access to the base learners $\alg[1], \dots, \alg[M]$, and which achieves with probability $1-\delta$ a regret bound of the form\footnote
{
Here, $\operatorname{poly}(a,b,c)$ is a polynomial function of the three arguments separately.
}
\begin{align*}
\regret_{\master}(t, \Pi_M) = \Ocal\Bigl(\operatorname{poly}\bigl(M,  R(\Pi_{i_\star}) , \ln R(\Pi_M) \bigl) \sqrt{t \ln{ \frac{t}{\delta}} }\Bigl)~,
\end{align*}
holding for all $t$, whenever $\alg[M]$ is h-stable and the environment $\bbB$ is adversarial. Simultaneously, if $\bbB$ is stochastic with gap $\Delta$, then we must have
\[
\pseudoregret_{\master}(t, \Pi_M) = \Ocal\Bigl(\frac{\operatorname{poly}( M, R(\Pi_M))}{\Delta} \log\frac{t}{\delta}\Bigl)~.
\]
Notice that the above requirement on the pseudo-regret in stochastic environments only requires a dependence on $R(\Pi_M)$, instead of $R(\Pi_{i_\star})$. 
This is motivated by the fact that in stochastic environments with gap $\Delta$ it is generally impossible to obtain model selection guarantees of the form
\(
\frac{R(\Pi_{i_\star}) \log t}{\Delta}
\)
-- see \pref{app:lower_bound} for a proof of this claim.

\vspace{-0.05in}
\section{Adversarial Model Selection Using Regret Balancing}\label{s:advbalancing}
\vspace{-0.05in}
We now introduce our algorithm for model selection in adversarial bandit problems with high-probability regret guarantees. The algorithm is shown in \pref{alg:adversarial_epoch_balancing}, and follows the regret balancing principle. This principle has been applied successfully to model selection in bandit and RL problems with stochastic rewards \citep{pmlr-v139-cutkosky21a, pacchiano2020regret, abbasi2020regret} and with corrupted stochastic rewards \citep{wei2021model}. Our work is the first to extend this approach to adversarial rewards.

\begin{algorithm2e}[t]
\textbf{Input:} initial time $t_0$, index $s$ of smallest active base learner, failure probability $\delta$\\%, probabilities $\{ \rho^s_i \}_{i=s}^M$\\
Initialize base learners $\alg[s], \alg[s+1], \dots, \alg[M]$
with extended policy classes $\tilde \Pi_s, \tilde \Pi_{s+1}, \dots, \tilde \Pi_{M}$
\\
Set sampling probabilities $\rho_i = \frac{R(\tilde \Pi_i)^{-2}}{\sum_{j=s}^M R(\tilde \Pi_j)^{-2}}$ for all $i \in \{s,\ldots, M\}$, and $\rho_i = 0$ for $i < s$\\

\For{round $t=t_0 + 1, t_0 + 2, \dots $}{
Sample base learner index $b_t \sim \operatorname{Categorical}(\rho_1, \cdots, \rho_M)$\\
Get context $x_t$ and
compute $a_t^i \sim \pi_{t}^{i}(\cdot | x_t)$, the action each base learner $\alg[i]$ proposes for $x_t$\\
Play action $a_t = a^{b_t}_t$ (resolve linked actions if necessary) and receive reward $r_t(a_t,x_t)$\\
Update all base learners $\alg[i]$ with reward $\frac{\indicator{b_t = i} r_t(a_t,x_t)}{\rho_i}$\\ 
Set
\vspace{-0.15in}
   \begin{equation}
       \tildecrew_{i}(t_0, t) = \sum_{\ell=t_0 + 1}^t \frac{\indicator{b_\ell = i} r_\ell(a_\ell,x_\ell)}{\rho_i},\ \ \ \  
       \conc_i(t_0, t) = \Ocal \left(\sqrt{\frac{t - t_0}{\rho_i} \ln \frac{\ln t}{\delta}} + \postdeadline{\frac{1}{\rho_i} \ln \frac{\ln t}{\delta}}\right)
       \label{eqn:tildecrewdef}
   \end{equation}
    Test for all $i, j \in \{s,\ldots, M\}$ with $i < j$:
    \begin{align}\label{eqn:algorithm_misspecification_test_condition}
        \tildecrew_{j}(t_0, t) &> \tildecrew_{i}(t_0, t) + \conc_i(t_0, t) + \conc_j(t_0, t) +  R(\tilde \Pi_i) \sqrt{\frac{t - t_0}{\rho_i} \ln \frac{t}{\delta}}
    \end{align}\\
    {\bf if} test triggers for $\alg[i]$ {\bf then}
    \textbf{restart} algorithm by running \textsf{Arbe}$(\delta, i+1, t)$
}
\caption{\textsf{Arbe}$(\delta, s = 1, t_0 = 0)$ Adversarial Regret Balancing and Elimination}
\label{alg:adversarial_epoch_balancing}
\end{algorithm2e}

\noindent{\bf Regret Balancing.} In each round, we choose the index of a base learner $b_t$ by sampling from a categorical distribution with probabilities $\rho_s, \rho_{s+1}, \dots, \rho_M$ that remain fixed throughout the epoch. The policy of learner $\alg[b_t]$ is then used to sample the action $a_t$ that is passed to the environment. After receiving the reward $r_t(a_t,x_t)$, \pref{alg:adversarial_epoch_balancing} updates each base learner $\alg[i]$ with $r_t(a_t,x_t)$ importance-weighted by the probability that the learner was selected,
i.e., $\frac{\indicator{b_t = i} r_t(a_t,x_t)}{\rho_i}$. Thus, we update all base learners. This is closer to the \textsf{Corral} algorithm \cite{agarwal2017corralling} than to the regret balancing approaches for the stochastic setting, which only update the selected learner $\alg[b_t]$.

If the probabilities are set to $\rho_i \propto \frac{1}{R(\Pi_i)^2}$ then after $t$ rounds, each learner $\alg[i]$ is followed roughly $\rho_i t$ times. To see why the regrets of all learners are balanced if they are $h$-stable against $\Pi_M$ in this case, consider the following. 
\newcommand{\regr}{\operatorname{reg}}
Denote by $\regr_t(a) = \left[\bbE_{a \sim \pi^{\star}}\left[ r_t( a,x_t ) \right] - r_t(a, x_t)\right]$ the regret in round $t$ of action $a$ where $\pi^{\star} \in \argmax_{\pi \in \Pi_M} \sum_{t =1}^T \bbE_{a \sim \pi}\left[ r_t( a,x_t) \right]$ is the best policy after all rounds $T$.
The regret in rounds where $\alg[i]$ was selected can be bounded using standard concentration arguments as
\begin{align*}
    \sum_{t =1}^T \indicator{b_t = i} \regr_t(a_t)
    = \rho_i \sum_{t =1}^T \frac{\indicator{b_t = i}}{\rho_i} \regr_t(a_t^i)
    \approx \rho_i  \sum_{t =1}^T \regr_t(a_t^i) 
    \leq R(\tilde \Pi_{i}) \tilde O\left(\sqrt{\rho_i T}\right),
\end{align*}
where the final inequality holds because $\alg[i]$ is h-stable. From the definition of $\rho_i$ we have $R(\tilde \Pi_i)\sqrt{\rho_i} = \left(\sum_{j=s}^M R(\tilde \Pi_j)^{-2}\right)^{-1/2} \leq R(\tilde \Pi_s)$, which shows that the total regret in those rounds is bounded by $\tilde O(R(\tilde \Pi_s)\sqrt{T})$.
Thus, if base learners are all h-stable, then the regret incurred by each of them is comparable to the regret incurred by $\alg[s]$, the learner with the smallest complexity $R(\Pi_s)$.

\noindent{\bf Eliminating Base Learners. }
If a learner $\alg[i]$ is not h-stable and may have linear regret, then the probabilistic schedule above which plays this learner roughly $\rho_i t$ times yields linear regret. We therefore monitor the performance of each learner and terminate the epoch whenever a learner performs significantly worse than expected, and thus cannot be h-stable in the environment.
To identify such cases, we compare estimates of the rewards of all pairs of base learners as follows.
For each learner $\alg[i]$,
$\tildecrew_{i}(t_0, t)$ in \pref{eqn:tildecrewdef}
is an unbiased estimate of the learners reward sequence $\crew_{i}(t_0, t) = \sum_{\ell=t_0 + 1}^t r_\ell(a^i_\ell, x_\ell)$ (see the appendix for details about how these estimates are computed). Further, using confidence bounds $\conc[i](t - t_0)$ from \pref{eqn:tildecrewdef}, we can constructs a confidence interval for $\crew_{i}(t_0, t)$ as  $\left[\tildecrew_{i}(t_0, t) - \conc_{i}(t - t_0), \tildecrew_{i}(t_0, t) + \conc_{i}(t - t_0)\right]$. If the confidence intervals for two learners with indices $i < j$ are more than the h-stable regret bound of $i$ apart, see \pref{eqn:algorithm_misspecification_test_condition}, then  deem all learners with index up to $i$ not h-stable and restart the algorithm with a reduced set of base learners.

This kind of elimination condition has already been used in stochastic environments \citep{pacchiano2020model, pmlr-v139-cutkosky21a, wei2021model}, but it requires substantially more care in an adversarial setting.
In settings with stochastic rewards (even with corruption, e.g., \cite{wei2021model}), no learner can achieve rewards that are significantly higher than the optimal policy. In contrast, in the adversarial setting, it is possible to have negative regret against any fixed policy. Thus, if we were to use the elimination condition in \pref{eqn:algorithm_misspecification_test_condition} with our base learners as is, then we may eliminate an h-stable base learner $i$ when another base learner $j$ has negative regret against the best fixed policy $\pi^\star_t$. This could in turn lead to undesirable regret in subsequent epochs when only learners with $R(\Pi_i) \gg R(\Pi_{i_\star})$ are left.

\noindent{\bf Linking Base Learner Performances.} To address this issue, we link the performance of base learners. Instead of instantiating each learner $\alg[i]$ with its original policy set $\Pi_i$, we apply it to an extended problem with $M - i$ additional actions $\tilde a_{i+1}, \tilde a_{i+2}, \dots, \tilde a_{M}$, and an extended policy set $\tilde \Pi_i$ that includes all original policies, along with policies that only choose one of the additional actions $\tilde a_{i}$, that is
$\tilde \Pi_i \supseteq \Pi_{i} \cup \{ \indicator{\tilde a_{i+1}}, \dots \indicator{\tilde a_M}\}$.
Whenever a base learner $\alg[i]$ chooses one of the additional actions $\tilde a_{j}$, then the action proposed by $\alg[j]$ is followed. Essentially, running each base learner with such an extended action set allows it to choose to follow the actions proposed by any learner with higher index in the hierarchy. 
The benefit of linking base learners this way is that each base learner now not only competes against the best fixed policy in their set, but also against all learners above them in the hierarchy. Thus, if $\alg[i]$ is h-stable and extendable (see \pref{def:extendable}), then it satisfies a regret bound of the form\footnote{We here naturally extend the domain of $r_\ell$ to linked actions as $r_\ell(\tilde a_{i}, x_\ell) = r_\ell(a^i_\ell, x_\ell)$ for all $i \in [M]$.}
\vspace{-0.05in}
\begin{align}
    \sum_{\ell=t_0 + 1}^t \left[ \bbE_{a \sim \pi^{j}_\ell}\left[r_\ell(a, x_\ell) \right] -\bbE_{a \sim \pi^{i}_\ell}\left[ r_\ell(a, x_\ell)  \right]\right] 
    &=  \tilde \Ocal\left(R(\tilde \Pi_i) \sqrt{\frac{t - t_0}{\rho_i}}\right)
    \label{eqn:pseudoregret_linking}
\end{align}
against any learner $\alg[j]$ with $j > i$. As a result, since the LHS of \pref{eqn:pseudoregret_linking} is approximately $\tildecrew_{j}(t_0, t) - \tildecrew_{i}(t_0, t) \pm \conc[i](t_0, t) \pm \conc[j](t_0, t)$,
the test in \pref{eqn:algorithm_misspecification_test_condition} cannot trigger for such $\alg[i]$, and only base learners that are not h-stable or extendable can be eliminated.  

In Appendix \ref{sa:adv}, we show that our regret balancing algorithm achieves the following regret bound:
\begin{restatable}{theorem}{arberegret}
\label{thm:arbe_main_regret}
Consider a run of \pref{alg:adversarial_epoch_balancing} with $\textsf{Arbe}(\delta, 1,0)$ and $M$ base algorithms \postdeadline{with $1 \leq R(\wt \Pi_1) \leq \dots \leq R(\wt \Pi_M)$} where $\tilde \Pi_i$ is the extended version of policy class $\Pi_i$ with $(M-i)$ additional actions. Then with probability at least $1 - \operatorname{poly}(M)\delta$ the regret $\regret(t, \Pi_M)$ for all rounds $t \geq \postdeadline{i_\star}$ is bounded by
\begin{align}\label{eqn:arbe_regret}
    \Ocal\left(  \left(\frac{R(\widetilde{\Pi}_{i_\star})}{R(\widetilde{\Pi}_{1})}\postdeadline{\sqrt{i_\star}} + \postdeadline{M} \right) R(\widetilde{\Pi}_{i_\star})\sqrt{ i_\star  t \ln \frac{t}{\delta} } \right)~,
\end{align}
where $i_\star$ is the smallest index of the base algorithm that is $h$-stable. 
\end{restatable}

In most settings, the complexity $\frac{R(\widetilde{\Pi}_{i_\star})}{R(\widetilde{\Pi}_{1})} \geq \sqrt{M}$ and the first term dominates. This regret recovers the expected regret guarantees of \textsf{Corral} \citep{agarwal2016corralling} when used with learning rates that do not require knowledge of $i_\star$ but as stronger high-probability bounds. To the best of our knowledge, \pref{thm:arbe_main_regret} is the first high-probability regret bound for adversarial model selection. To illustrate our result, consider to common problem of model selection with nested model classes of dimensions $d_i = 2^i$ for $i \in [M]$, discussed in the introduction. We can use  GeometricHedge.P as base learners which are h-stable and extendable if adapted (see \pref{app:geometric_hedge}) and the complexity of extended policy classes $R(\widetilde{\Pi}_{i}) = R(\Pi_{i}) + M - i \leq  d_i + M$ is not much larger than those of the original policy classes. \textsf{Arbe} with such base learners achieves a regret bound of order $\tilde O(d_{i_\star}^2 \sqrt{t})$ up to factors of $M$ and $\log$-factors which are typically small.

\vspace{-0.05in}
\section{Adversarial Model Selection with Best of Both Worlds Guarantees}\label{sec:bbw}
\vspace{-0.05in}

In this section, we present our algorithm for adversarial model selection that preserves a logarithmic regret guarantee in case it interacts with a stochastic environment with gap $\Delta > 0$. 
In order to achieve such a best-of-both-worlds guarantee, we will combine our adversarial regret balancing technique with the algorithmic strategy of \citet{bubeck2012best}.  The main result for our algorithm is:
\begin{restatable}{theorem}{mainthmmodselbbw}
\label{thm:arbegap_main}
Consider a run of \pref{alg:gap_estimation_balanced} with inputs $t_0 = 0$, arbitrary policy policy $\wh \pi \in \Pi_M$ and $M$ base learners $\alg[1], \dots, \alg[M]$. Then with probability at least $1 - \poly(M)\delta$, the following two conditions hold for all $t \geq M^2$ simultaneously. In any adversarial or stochastic environment $\bbB$, the regret is bounded as
\begin{align*}
       \regret(t, \Pi_M) &= \Ocal\left(  \left(M  + \ln(t) + \frac{R(\wt \Pi_{i_\star})}{R(\wt \Pi_{1})}  \sqrt{ i_\star}\right) R(\wt \Pi_{i_\star}) \sqrt{t (\ln(t) + i_\star) \ln \frac{t}{\delta}} \right)~.
\end{align*}
If $\bbB$ is stochastic and there is a unique policy with gap $\Delta > 0$, then the pseudo-regret is bounded as,
\begin{align}
    \pseudoregret_{\bbM}(t, \Pi_M) &= 
    \Ocal\Bigg(\frac{R(\Pi_{M})^2}{\Delta} \ln(t) \ln \frac{t}{\delta} + 
    \frac{ R(\wt \Pi_{i_\star})^2 R(\wt \Pi_M)^2}{R(\wt \Pi_1)^2}\frac{M^2 i_\star}{\Delta} \ln^2 \left(\frac{M R(\wt \Pi_M)}{\Delta \delta }\right) \Bigg).
    \label{eqn:arbegap_pseudoregret}
\end{align}
\end{restatable}

The algorithm has the same regret bound as \textsf{Arbe} up to a $\sqrt{\ln t}$ factor. However, in addition, it also maintains poly-logarithmic pseudo-regret if the environment is stochastic and the best policy exhibits a positive gap. Our pseudo-regret bound depends polynomially on $R(\Pi_M)$ in contrast to the $R(\Pi_{i_\star})$ dependency for the adversarial regret rate. This may seem undesirable but, as mentioned in \pref{sec:ms_setting}, model selection in stochastic environments with $\poly(R(\Pi_{i_\star}), \ln(R(\Pi_M)) \frac{\ln (t)}{\Delta}$ regret is impossible (see \pref{app:lower_bound}), let alone while maintaining a best-of-both-worlds guarantee with $\sqrt{T}$ regret in adversarial environments. Hence, our algorithm achieves the best kind of guarantee we can hope for, up to improvements in the order of polynomial dependencies.

Our algorithm proceeds in two distinct phases. The first, $\textsf{Arbe-Gap}$ shown in \pref{alg:gap_estimation_balanced} is designed to identify a suitable candidate for the optimal policy and to estimate its gap. If such a policy that performs significantly better than any other policy emerges, the algorithm enters the second phase, $\textsf{Arbe-GapExploit}$ which hones in on this policy by playing it most of the time while monitoring its regret in case the environment turns out to be adversarial after all (in which case we simply run \textsf{Arbe}). We now describe both phases.

\begin{algorithm2e}[t]
\newcommand\mycommfont[1]{\textit{\textcolor{DarkBlue}{#1}}}
\SetCommentSty{mycommfont}
\textbf{Input:} failure probability $\delta$, timestep $t_0$, focus policy $\widehat{\pi}$, \postdeadline{number of (re)starts $n$,} learners $(\alg[i])_{i=s}^M$ \\% 
Set $\Pi_{M+1} = \Pi_M \backslash \{ \widehat{\pi}\}$ and $\alg[M+1]$ as a copy of $\alg[M]$ with policy class $\Pi_{M+1}$\\
Initialize base learners $\alg[s], \alg[s+1], \dots, \alg[M+1]$ with extended policy classes $\tilde \Pi_s, \tilde \Pi_{s+1}, \dots, \tilde \Pi_{M+1}$\label{lin:init_arbegap}\\
Set sampling probabilities $\rho_i = \frac{R(\tilde \Pi_i)^{-2}}{\sum_{j=s}^{M+1} R(\tilde \Pi_j)^{-2}}$ for all $i \in \{s,\ldots, M+1\}$, and $\rho_i = 0$ for $i < s$\\
\For{round $t= t_0+1,t_0+2, \dots $}{
    Sample base learner index $b_t \sim \operatorname{Categorical}(\rho_1, \cdots, \rho_M)$\\
    Get context $x_t$ and
    compute $a_t^i \sim \pi_{t}^{i}(\cdot | x_t)$, the action each base learner $\alg[i]$ proposes for $x_t$\\
    Play action $a_t = a^{b_t}_t$ (resolve linked actions if necessary) and receive reward $r_t(a_t,x_t)$\\
   Update all base learners $\alg[i]$ with reward $\frac{\indicator{b_t = i} r_t(a_t,x_t)}{\rho_i}$\\ 
    \If{\pref{eqn:algorithm_misspecification_test_condition} holds between $\alg[i],\alg[j]$ for $s \leq i < j \leq M+1$}{
        restart algorithm by running \textsf{Arbe-Gap}$(\delta, t, \wh \pi, \postdeadline{n+1}, (\alg[i+1], \cdots, \alg[M]))$ \label{lin:arbegap2}
    }
    \tcp{Gap Test: $\alg[M]$ better than $\alg[M+1]$?}
    Set\ \  $\textrm{W}(t_0, t) = \Theta\left( \postdeadline{\sqrt{\frac{R(\wt \Pi_M)^2}{\rho_M(t - t_0)}\ln \frac{n(t - t_0)}{\delta}} + \frac{\ln \frac{n\ln (t - t_0)}{\delta}}{\rho_M(t - t_0)}} \right)$\\
    Set $ \wh \Delta_t = \frac{\tildecrew_M(t_0, t) - \tildecrew_{M+1}(t_0,t)}{t - t_0} - \textrm{W}(t_0, t)$\label{lin:gap_estimate}\\
    \If{$\postdeadline{2}\textrm{W}(t_0, t) \leq \wh \Delta_t \postdeadline{\leq R(\wt \Pi_M)^2}$\label{lin:gap_test}}{
    Run \textsf{Arbe-GapExploit} (\pref{app:bestofboth}) with inputs $\delta, t_0=t, \alg[M], \wh \pi$ and  $\wh \Delta = \wh \Delta_t$\\
    \postdeadline{Run \textsf{Arbe} with inputs $t_0, s, \delta$} \label{lin:adv_fallback}
    }

    \tcp{New Candidate Policy Test}
    \If{a policy $ \pi \in \Pi_M \setminus \{ \wh \pi\}$ has been selected in more than $\frac{3t}{4}$ of all $t \postdeadline{\geq 9}$ rounds \label{lin:policy_test}}{
    restart algorithm by running \textsf{Arbe-Gap}$(\delta, t, \pi, \postdeadline{n+1}, (\alg[i])_{i=s}^M)$ 
}
}
\caption{\textsf{Arbe-Gap}$(\delta, t, \postdeadline{\wh \pi}, \postdeadline{n}, (\alg[i])_{i=s}^M)$}
\label{alg:gap_estimation_balanced}
\end{algorithm2e}

\vspace{-0.05in}
\subsection{First Phase: Candidate Policy Identification and Gap Estimation}
\vspace{-0.05in}
This goal of this phase is to always maintain the desired adversarial model selection regret guarantee and simultaneously identify the gap between the best policy and the rest if one exists. To achieve the first goal, we employ the regret balancing and elimination technique of \textsf{Arbe}, see Lines~\ref{lin:init_arbegap}--\ref{lin:arbegap2} of \pref{alg:gap_estimation_balanced} which are virtually identical to \pref{alg:adversarial_epoch_balancing}.

For the second goal, determining the gap, \textsf{Arbe-Gap} maintains a candidate $\wh \pi \in \Pi_M$ for the optimal policy and estimates its gap as follows. The learner hierarchy $\alg[1], \dots, \alg[M]$ is augmented at the top with an additional learner $\alg[M+1]$ operating on the policy class $\Pi_{M+1} = \Pi_M \setminus \{ \wh \pi\}$. Since $\alg[M+1]$ is identical to $\alg[M]$ except that it does not have access to $\wh \pi$, we can obtain a gap estimate for $\wh \pi$ by monitoring the difference in reward estimates $\tildecrew_{M+1}$ and $\tildecrew_{M}$ of $\alg[M+1]$ and $\alg[M]$. 
In fact, $\wh \Delta_t$ in Line~\ref{lin:gap_estimate} is a lower confidence bound on the difference between the best policies in $\wh \Pi_M$ and $\wh \Pi_M \setminus \{ \wh \pi\}$ and thus, the gap of $\wh \pi$. We test at every round whether $\wh \Delta_t$ exceeds its confidence width $\postdeadline{2} \textrm{W}(t_0, t)$. If this test triggers then $\wh \pi$ must have a positive gap $\Delta$ of order $\textrm{W}(t_0, t)$ and further, $\wh \Delta_t$ must be a multiplicative estimate of $\Delta$, that is, $\wh \Delta_t \leq \Delta \leq 2 \wh \Delta_t$. The latter holds, because $\wh \Delta_t + \postdeadline{2} \textrm{W}(t_0, t)$ is an upper-confidence bound on $\Delta$ and the test condition implies $2 \wh \Delta_t \geq \wh \Delta_t + \postdeadline{2} \textrm{W}(t_0, t) \geq \Delta$. Since in this case, we have determined that $\wh \pi$ is optimal with a gap of order $\wh \Delta_t$, we can move on to the second phase $\textsf{Arbe-GapExploit}$ discussed later.

Assume the candidate policy $\wh \pi$ is indeed optimal and exhibits a positive gap $\Delta$. Since $\wh \Delta_t$ concentrates around $\Delta$ at a rate of $\textrm{W}(t_0, t) \approx \frac{\poly(R_M(\wt \Pi_M))}{ \sqrt{t - t_0}}$, the condition of the gap test in Line~\ref{lin:gap_test} must trigger after at most $t - t_0 \lesssim \frac{\poly(R_M(\wt \Pi_M))}{\Delta^2}$ rounds. Finally, since $\textsf{Arbe-Gap}$ always maintains the $\poly(R(\wt \Pi_{i_\star}) \sqrt{t - t_0}$ adversarial regret rate, the total pseudo-regret incurred until the test triggers is of order $\frac{\poly(R_M(\wt \Pi_M))}{\Delta}$, leading to the second term in the bound in \pref{eqn:arbegap_pseudoregret}.

With the techniques above, we can reliably detect a positive gap if the candidate policy $\wh \pi$ exhibits one. It remains to identify a suitable candidate $\wh \pi$ of the optimal policy $\pi^\star$ in stochastic environments. To do so, we use the following two observations: $\textsf{Arbe-Gap}$ maintains a $\poly(R(\wt \Pi_{i_\star}) \sqrt{t}$ adversarial regret rate overall and each policy but $\pi^\star$ will incur on average at least a regret of $\Delta$ per round. Hence, in order to maintain that regret,  $\textsf{Arbe-Gap}$ must select $\pi^\star$ in the majority of all rounds when $t = \omega(\frac{\poly(R(\wt \Pi_{i_\star}))}{\Delta^2})$. Otherwise the regret grows as $\Omega(t\Delta) = \omega(\poly(R(\wt \Pi_{i_\star})) \sqrt{t})$ violating the adversarial rate.
To leverage this observation and identify a suitable candidate policy $\wh \pi$, Line~\ref{lin:policy_test} of \pref{alg:gap_estimation_balanced} always checks whether there is a policy other than the current candidate policy that has been selected in at least $3/4$ of all rounds.\footnote{Due to linking learners, there is slight ambiguity in defining the selected policy per round. We here determine the selected policy as the policy chosen by the learner that eventually picked $a_t$ \emph{after resolving linked actions}.} If so, the algorithm is restarted with this policy as candidate $\wh \pi$. One can show that there are at most $\Ocal(\ln t)$ restarts due to candidate policy switches, only increasing the adversarial regret rate of $\textsf{Arbe-Gap}$ by a factor of $\Ocal(\sqrt{\ln t})$ compared to $\textsf{Arbe}$'s.

Our candidate policy selection approach is similar to that by \citet{wei2021model} for the corrupted reward setting but they require each adapted base learner to achieve a logarithmic regret rate in the first place. Instead, our approach only requires an adversarial regret rate from the base learner. %

\vspace{-0.05in}
\subsection{Second Phase: Exploitation}
\vspace{-0.05in}
Since each base learner only needs to satisfy a $\sqrt{T}$ regret rate even in stochastic environments, we generally cannot hope to recover logarithmic pseudo-regret by only selecting among them. For logarithmic pseudo-regret, we need to ensure that the given policy $\wh \pi$ with gap estimate $\wh \Delta$ is played sufficiently often. However, we also need to monitor its regret against all other policies in case the environments turns out to be adversarial. If at any point, $\wh \pi$ fails to maintain a gap of order $\wh \Delta$, we can conclude that the environment is adversarial. Then \textsf{Arbe-GapExploit} returns and we simply play an instance of $\textsf{Arbe}$ (Line~\ref{lin:adv_fallback} of \pref{alg:gap_estimation_balanced}). 

We will present a brief summary of the main intuition behind our exploitation phase approach here and defer a longer discussion and the detailed pseudo-code to \pref{app:bestofboth}.
In each round, we play policy $\widehat \pi$ with probability approximately $1- \frac{\postdeadline{\poly(R(\Pi_M))}}{\wh \Delta^2 t}$, and with the remaining probability $\frac{\postdeadline{\poly(R(\Pi_M))}}{\wh \Delta^2 t}$ a version of base learner $\alg[M]$ with policy class $\Pi_M \setminus \{ \widehat \pi \}$. 

If the environment is indeed stochastic, then $\wh \pi = \pi^\star$ does not incur any pseudo-regret and the total regret in other $t' \approx t \cdot \frac{\postdeadline{\poly(R(\Pi_M))}}{\wh \Delta^2 t} \approx \frac{\postdeadline{\poly(R(\Pi_M))}}{\wh \Delta^2}$ rounds can be bounded as $t' \cdot \Delta + \regret_{\alg[M]}(t', \Pi_M \setminus \{ \widehat \pi \}) \approx \frac{\poly(R(\Pi_M))}{\Delta} \ln(t)$. The first term is the regret of the best policy in $\Pi_M \setminus \{ \widehat \pi \}$ and the second term is the regret of $\alg[M]$ against that policy. Since $\alg[M]$ is h-stable on $\Pi \setminus \{\wh \pi\}$, its regret against that policy is at most $\poly(R(\Pi_M)) \sqrt{t'}$  and we get the desired pseudo-regret.

To ensure good regret in the adversarial case, we need to detect quickly enough when $\wh \pi$ does not exhibit a performance gap anymore and fall back to a fully adversarial algorithm. Similar to $\wh \Delta_t$ in the first phase, we use a lower confidence bound on the average performance gap and continuously test whether it falls below $\frac{\wh \Delta}{2}$. This simple approach would give $\sqrt{t}$ adversarial regret but may exhibit a $\poly(R(\Pi_M))$ dependency. Fortunately, we can avoid this dependency and retain the desired model selection regret rates by extending the intuition above to also test an upper bound on the gap. For details, see \pref{app:bestofboth}.

\vspace{-0.05in}
\section{Conclusions}
\vspace{-0.05in}
We have described and analyzed a novel model selection scheme for bandit algorithms that benefits from best-of-both-worlds high probability regret guarantees. Our machinery can be specifically applied to adversarial/stochastic linear bandit tasks, where model selection is performed on the unknown dimensionality of the linear reward function. This has required extending the regret balancing technique of model selection from stochastic to adversarial rewards and a very careful handling of the associated mis-specification tests. The base learners aggregated by our meta-algorithm have to satisfy an anytime high probability regret guarantee in the adversarial case, along with regret stability and action space extendability properties which we have shown are satisfied by (variants of) known algorithm in the bandits literature. 

Our best-of-both world model selection regret guarantees cannot in general be improved, specifically in stochastic environments with gaps, where it is generally impossible to obtain $\log t$-like model selection bounds that only depend on the complexity of $\Pi_{i_\star}$. On the other hand, it would be nice to see in \pref{thm:arbegap_main} a better polynomial dependence on $M$ and $R(\Pi_M)$.

\bibliographystyle{abbrvnat}
\bibliography{manual}

\appendix
\newpage

\tableofcontents
\addtocontents{toc}{\protect\setcounter{tocdepth}{3}}

\section{Lower Bound on Model Selection for Stochastic Environments}\label{app:lower_bound}
Considering a simple multi-armed bandit problem suffices to prove that in stochastic environments with gap $\Delta$ it is generally impossible to obtain model selection guarantees of the form
\[
\frac{R(\Pi_{i_\star}) \log T}{\Delta}~.
\]
\begin{theorem}\label{lemma::lower_bound_logarithmic}
Let $K_1, K_2 \in \bbN$ with $K_1 < K_2$ and let $\Delta \in \bbR^+$ be fixed. Further, let $c : \mathbb{N} \rightarrow \mathbb{R}$ be an arbitrary function over the real numbers. If $c(K_1) < K_2 - K_1$, there is a class of $K_2$ multi-armed bandit problems with one optimal arm and gaps in $[\Delta, 2 \Delta]$ and a $T_0 \in \bbN$  such that the following holds.
For any algorithm $\bbA$ and for all $T \geq T_0$, there is a problem instance such that
\begin{align*}
\bbE[\regret_{\bbA}(T)] > 
    \begin{cases}
     \frac{c(K_1)}{\Delta} \ln T & \textrm{ if the optimal arm } a^\star \in [K_1]\\
     \frac{c(K_2)}{\Delta} \ln T & \textrm{ otherwise}~.
    \end{cases}
\end{align*}
\end{theorem}
\begin{remark}
The proof below actually shows the following stronger version
\begin{align*}
\bbE[\regret_{\bbA}(T)] > 
    \begin{cases}
     \frac{c(K_1)}{\Delta} \ln T & \textrm{ if the optimal arm } a^\star \in [K_1]\\
     \frac{\Delta}{16}T^{1-\frac{c(K_1)}{K_2 - K_1}}  & \textrm{ otherwise}~.
    \end{cases}
\end{align*}
This shows that if we aim to obtain an instance-dependent logarithmic regret bound that only scales with $K_1$ when $a^\star \in [K_1]$, then we cannot recover sublinear regret for $K_2 \gg K_1$ when $a^\star \notin [K_1]$.
\end{remark}

\begin{proof}
We will show the statement for $K_1 = 1$ but it can be trivially generalized to $K_1 > 1$. The rewards of all arms and in all instances are drawn from a Gaussian distribution with variance $1$ but different means. We denote by $\mu_i^{j}$ the mean reward of arm $j$ in instance $i$. We identify each bandit instance in the family by its mean rewards $\mu_i$. They are given by
\begin{align}
\mu_i &= [\Delta, 0, \dots, 0]&\textrm{ for $i = 1$} \label{equation::mean_definition_lower_bound1}\\
    \mu_i &= [\Delta, 0, \dots, \underset{i\textrm{'th pos}}{\underbrace{2 \Delta}}, 0, \dots, 0] &\textrm{ for $i > 1$} \label{equation::mean_definition_lower_bound2}
\end{align}
Thus, arm $i$ is optimal in instance $i$.
Consider any algorithm $\bbA$ and assume it violates the lower bound for the case where $a^\star \in [K_1]$, i.e., in the first problem instance,
\begin{align*}
    \bbE\left[\regret_{\mathbb{A}}(T, \mu_1) \right]\leq \frac{c(K_1)\log(T)}{\Delta}
\end{align*}
holds for some $T \geq T_0$ (we will specify $T_0$ later).
Otherwise, the statement is already true. We will now show that this algorithm has to satisfy the regret lower bound in one of the other problem instances for $T$.
By definition of $\mu_1$, we can write the expected regret as
\begin{equation*}
    \bbE \regret_{\mathbb{A}}( T, \mu_1) \geq 2\Delta  \sum_{i=K_1 + 1}^{K_2} \bbE_1\left[ T_i(T)\right]~,
\end{equation*}
where $T_i(T)$ is the (random) number of times $\mathbb{A}$ has pulled arm $i$ up to time $T$. $\bbE_j$ and $\bbP_j$ denote the expectation and the probability distribution induced by algorithm $\bbA$ in instance $j$, respectively. Consider now a problem instance $\wh i \in \argmin_{K_1 + 1 \leq i \leq K_2} \bbE_1 \left[ T_i(T)\right]$. As a consequence of the previous two observations,
\begin{equation*}
  2 \Delta (K_2 - K_1) \bbE_1 \left[ T_{\widehat{i}}(T)\right] \leq  \frac{c(K_1)\log(T)}{\Delta}~,
\end{equation*}
hence
\begin{equation*}
    \bbE_1 \left[ T_{\widehat{i}}(T)\right] \leq  \frac{c(K_1)\log(T)}{2(K_2 - K_1)\Delta^2}~.
\end{equation*}
By the divergence decomposition \citep[Lemma 15.1]{lattimore2018bandit}, 
\begin{align*}
    \mathrm{KL}\left( \bbP_1, \bbP_{\wh i} \right)
    &= \bbE_1 \left[ T_{\widehat{i}}(T) \right] \frac{(2\Delta)^2}{2} \leq \frac{c(K_1)\log(T)}{K_2 - K_1}~.
\end{align*}
Define the event $\mathcal{E} = \{ T_1(T) \leq \frac{T}{2}\}$. Notice that,
\begin{equation*}
   \frac{T}{2} \Delta   \bbP_1 \left( \mathcal{E} \right) 
   \leq   \bbE_1 \regret_{\bbA}(T, \mu_1) \leq \frac{c(K_1) \log(T)}{\Delta}
\end{equation*}
which implies,
\begin{equation}\label{equation::upper_bound_probability_1}
     \bbP_1 (\cE)  \leq \frac{2c(K_1) \log(T)}{\Delta^2 T}~.
\end{equation}
By the Bretagnolle-Huber inequality,
\begin{equation*}
    \bbP_1 \left( \mathcal{E} \right)  + \bbP_{\wh i} \left( \mathcal{E}^c \right)  \geq \frac{1}{2} \exp\left(-\frac{c(K_1)\log(T)}{K_2 - K_1}  \right) =\frac{1}{2} \left(\frac{1}{T} \right)^{\frac{c(K_1)}{K_2 - K_1}}~,
\end{equation*}
and by combining the last two inequalities, we can lower bound $\bbP_{\wh i} \left( \mathcal{E}^c \right) $ as 
\begin{equation*}
    \bbP_{\wh i} \left( \mathcal{E}^c \right) 
    \geq 
    \frac{1}{2} \left(\frac{1}{T} \right)^{\frac{c(K_1)}{K_2 - K_1}} - \frac{2c(K_1) \log(T)}{\Delta^2 T}~.
\end{equation*}
Since $ \bbE_{\wh i}\regret_{\bbA}(T, \mu_{\wh i})  \geq \frac{T \Delta}{2} \bbP_{\wh i} \left( \mathcal{E}^c \right)$ this implies,
\begin{equation*}
    \bbE_{\wh i}\regret_{\bbA}(T, \mu_{\wh i})  
    \geq 
    \frac{T\Delta}{2} \left( \frac{1}{2} \left(\frac{1}{T} \right)^{\frac{c(K_1)}{K_2 - K_1}} - \frac{2c(K_1) \log(T)}{\Delta T} \right) = \frac{\Delta}{4}T^{1-\frac{c(K_1)}{K_2 - K_1}} - c(K_1) \log(T)~.
\end{equation*}
By setting $T_0$ sufficiently large as a function of $\Delta, K_1, K_2$ and $c$, we can conclude that 
 \begin{equation*}
    \bbE_{\wh i}\regret_{\bbA}(T, \mu_{\wh i}) \geq 
    \frac{\Delta}{4}T^{1-\frac{c(K_1)}{K_2 - K_1}} - c(K_1) \log(T)
    \geq \frac{\Delta}{8}T^{1-\frac{c(K_1)}{K_2 - K_1}} 
    > \frac{c(K_2) \log(T)}{\Delta}~.
\end{equation*}
Specifically, it suffices to set $T_0$ as the smallest value of $T$ that satisfies the last two inequalities in the display.
This shows that the lower bound holds at time $T$ for problem instance $\wh i$.
\end{proof}

\section{Examples of h-Stability and Extendability}\label{app:geometric_hedge}
This appendix shows examples of h-stability and extendabilty. The first example, contained in Section \ref{ssa:geomhedge}, is a variant of the Geometric Hedge algorithm from \cite{bartlett2008high}. We sketch a second example in Section \ref{ssa:exp4}, where we deal with a high probability variant of the Exp4 algorithm from \cite{a+03}.

\subsection{Geometric Hedge for Adversarial Linear Bandits}\label{ssa:geomhedge}
Let us start by introducing a weighted version of the Geometric Hedge Algorithm that satisfies the h-stability condition of Definition~\ref{def:high_probability_stability}. We consider the setting where the action $\mathcal{A} \subset \mathbb{R}^d$ is finite (but potentially large). We denote by $\mathbf{a}_t$ and $\boldsymbol{\omega}_t$ the learner's action selection and the adversary's reward vector at time $t$, respectively. The associated reward $r_t = \mathbf{a}_t^\top\boldsymbol{\omega}_t$ is assumed to lie in the interval $[-1,1]$. Moreover, let $\delta \in (0,1)$ be a probability parameter and set for brevity $\delta' = \frac{\delta}{|\mathcal{A}|}$. 

The algorithm, an anytime variant of the Geometric Hedge algorithm from \cite{bartlett2008high} with John's ellipsoid exploration (e.g., \cite{bubeck2012regret}) is detailed in Algorithm \ref{Alg:AnytimeGeoHedge}. We call the algorithm Anytime Weighted Geometric Hedge.
The algorithm takes in input the set of actions $\Acal$, the failure probability $\delta$, and a weighting probability $\rho \in (0,1]$ which will play the role of an importance weight. As is standard, the algorithm maintains over time a distribution $p_t$ over $\Acal$, which is itself a mixture of an exponential weight distribution $q_t$ and an exploration distribution $p_E$. The exploration distribution $p_E$ is defined beforehand to be the John's ellipsoid distribution associated with $\Acal$ (see \cite{bubeck2012regret} for details). As in Geometric Hedge \cite{bartlett2008high}, the algorithm builds a covariance matrix $\boldsymbol{\Sigma}_t$ by computing the expectation of $\mathbf{a}\mathbf{a}^\top$ where $\mathbf{a}$ is drawn according to the current distribution $p_t$ over actions.\footnote
{
In this section, we denote by $\mathcal{F}_{t-1}$ the $\sigma$-algebra generated by all past random variables up to, but excluding, the random draw of $\mathbf{a}_t$ (so that $p_t$ is $\mathcal{F}_{t-1}$-measurable).
}
Then, the algorithm samples a Bernoulli random variable $b_t$, and computes an (importance-weighted) unbiased estimator $\widehat{\boldsymbol{\omega}}_t$ of $\boldsymbol{\omega}_t$, which is plugged into the (biased) reward estimator $\widetilde{r}_t(\mathbf{a})$ that the algorithm associates with every action $\mathbf{a} \in \Acal$. The factors $\widetilde{r}_t(\mathbf{a})$ are those that determine the exponential update of distribution $p_t$.

\begin{algorithm2e}
\textbf{Input:} Action set $\mathcal{A}$, failure probability $\delta$, weighting probability $\rho \in (0,1]$. \\
Initialize $w_1(\mathbf{a}) = 1$, $W_1 = |\mathcal{A}|$ and $q_1(\mathbf{a}) = \frac{1}{|\mathcal{A}|}$ for all $\mathbf a \in \Acal$\\
\For{$t=1, 2, \cdots $}{
Compute sampling distribution 
\begin{align}
p_t(\mathbf{a}) = (1-\gamma_t) q_t(\mathbf{a}) + \gamma_tp_E(\mathbf{a})\qquad
\textrm{where } q_t(\mathbf{a}) = \frac{w_t(\mathbf{a}) }{W_t}\label{equation::AnytimeGeoHedgeWeightProbDef}
\end{align}

Adversary generates reward vector $\boldsymbol{\omega}_t$\\
Sample action $\mathbf{a}_t \sim p_t$\\
Observe and gather reward $r_t = \mathbf{a}_t^\top\boldsymbol{\omega}_t$\\
Build covariance matrix $\boldsymbol{\Sigma}_t = \mathbb{E}_{\mathbf{a} \sim p_t}\left[ \mathbf{a} \mathbf{a}^\top | \mathcal{F}_{t-1} \right]$\\
Sample $b_t \sim \mathrm{Ber}(\rho)$\\
Compute unbiased reward vector estimator $\widehat{\boldsymbol{\omega}}_t = b_t \frac{r_t  \boldsymbol{\Sigma}_t ^{-1} \mathbf{a}_t}{\rho}$ \\
Compute the reward upper bounds $$
\widetilde{r}_t(\mathbf{a}) = \mathbf{a}^\top\widehat{\boldsymbol{\omega}}_t  + 2\mathbf{a}^\top \boldsymbol{\Sigma}_t^{-1} \mathbf{a} \sqrt{\frac{\ln(12t^2/\delta')}{\rho\, d\,t} }\ \ \
\forall \mathbf{a} \in \Acal
$$
Update distribution
\begin{equation}
w_{t+1}(\mathbf{a}) =  \exp\left( \eta_{t+1} \sum_{\ell=1}^t \widetilde{r}_\ell(\mathbf{a}) \right)\ \ \forall \mathbf{a} \in \Acal \label{equation::AnytimeGeoHedgeWeightUpdate}
\end{equation}
Update normalization factor $W_{t+1} = \sum_{\mathbf{a} \in \Acal} w_{t+1}(\mathbf{a})$\\
}
\caption{Anytime Weighted Geometric Hedge.}
\label{Alg:AnytimeGeoHedge}
\end{algorithm2e}

\begin{remark}
For simplicity, the algorithm is formulated for the case where the action space $\Acal$ is finite. When $\Acal$ is infinite, we can still formulate an algorithm that applies to a $\epsilon$-cover of $\Acal$ that is restarted at exponentially increasing time-steps $t_0 = 1, 2, 4, 8, \ldots,$. At the beginning of each epoch, the covering level $\epsilon$ is set to $\Ocal(1/t_0)$ so as to obtain an anytime algorithm. A very similar bound to the one in Theorem \ref{theorem::main_weighted_geomhedge_constantprob_simplified} is obtained, where $|\Acal|$ therein is replaced by $|\Acal_{1/t}|$, $\Acal_{1/t}$ being a $1/t$-cover of $\Acal$ w.r.t. the infinity norm.
\end{remark}

\begin{remark}
Notice that we are using the fact that $\boldsymbol{\Sigma}_t$ is invertible for all $t$. This is no loss of generality, as we can always assume that $\Acal$ spans the whole $d$-dimensional space (if this is not the case, we can project each $\mathbf{a}$ onto the space spanned by $\Acal$ and reduce to this case). Combined with the fact that, for all $t$, distribution $p_t$ assigns a nonzero probability to each action, this implies that the expectation $\boldsymbol{\Sigma}_t = \mathbb{E}_{\mathbf{a} \sim p_t}\left[ \mathbf{a} \mathbf{a}^\top | \mathcal{F}_{t-1} \right]$ must be full rank.
\end{remark}

\subsubsection{h-Stability}
Setting the mixture factor $\gamma_t$ and the learning rate $\eta_t$ appropriately, we can prove the following regret guarantee for Algorithm~\ref{Alg:AnytimeGeoHedge}.

\begin{theorem}\label{theorem::main_weighted_geomhedge_constantprob_simplified}
Let the Anytime Weighted Geometric Hedge Algorithm be run with 
$$
\eta_t = \Ocal\left(\frac{\rho \gamma_t}{d + \sqrt{\frac{d}{t}} \sqrt{\rho \ln |\mathcal{A}| \ln \frac{t}{\delta} }}\right)~,
\qquad 
\gamma_t =  \min\left\{\sqrt{\frac{d \ln|\mathcal{A}|\ln \frac{t}{\delta}}{\rho\,t}},\frac{1}{2}\right\}~,
$$
and $\rho \in (0,1]$. Then with probability at least $1-\delta(3+2d^2)$, simultaneously for all $t,$ the regret $\regret(t)$ after $t$ rounds can be bounded as
\begin{align*}
\mathrm{Reg}(t) = \Ocal\left( \sqrt{  \frac{ d\, t\, \ln |\Acal|}{\rho}  \ln \frac{t}{\delta} } + \rho \ln |\mathcal{A}| \ln \frac{t}{\delta} \ln t 
+ \left( \frac{\ln |\mathcal{A}|}{d} \ln \frac{t}{\delta}  \right)^{1/4} t^{1/4}\left(\frac{1}{\rho} \right)^{3/4} + \frac{1}{\rho} \ln \frac{t}{\delta}\right)~.
\end{align*}
In the above, the big-oh notation only hides absolute constants. 
\end{theorem}
Hence we have the following corollary.
\begin{corollary}
Let the complexity $R(\Pi)$ of the policy space $\Pi$ be defined as $R(\Pi) = \sqrt{d\,\log |\Acal|}$. Then with the same assumptions and setting as in Theorem \ref{theorem::main_weighted_geomhedge_constantprob_simplified}, the Anytime Weighted Geometric Hedge algorithm is h-stable in that, for $t \rightarrow \infty$ and constant $\rho$ independent of $t$, its regret $\regret(t)$ satisfies
\begin{equation*}
\regret(t) = \Ocal\left(R(\Pi)\,\sqrt{\frac{t}{\rho}  \ln \frac{t}{\delta} }\right)~,
\end{equation*}
with probability at least $1-\delta$,
where the big-oh hides terms in $t$ which are lower order than $\sqrt{t\log t}$ as $t \rightarrow \infty$.
\end{corollary}

\subsubsection{Extendability}
The extendability of the Anytime Weighted Geometric Hedge algorithm is easily obtained by simply observing that, since $\Acal \subseteq \mathbb{R}^d$, we can extend the dimensionality $d$ to $d+k$. Then we extend each $\mathbf{a} =(a_1, \ldots, a_d) \in \Acal$ to $\mathbf{a'} =(a_1, \ldots, a_d, \underbrace{0,\ldots,0}_{k\ \ {\mbox{zeros}}})$, and add $k$ new actions $\mathbf{b}_i$ of the form
\[
\mathbf{b}_i = (\underbrace{0,\ldots,0}_{d\ \   {\mbox{zeros}}},\underbrace{0,\ldots, 1,\ldots,0}_{{\mbox{position $i$}}})
\]
for $i = 1, \ldots, k$. Denote the extended policy space by $\Acal'$.
Consistent with this extension, any policy $\pi$ in the original policy space $\Pi$ has to be interpreted as a probability distribution in $\Delta_{\Acal'}$, whose last $k$ components are zero, while the $k$ extra indicator policies $\indicator{b_1}, \ldots,\indicator{b_k}$ are degenerate probability distributions in $\Delta_{\Acal'}$, where
$\indicator{b_i}$ places all its probability mass on the $(d+i)$-th component.
Finally the adversary reward vector $\boldsymbol{\omega}_t \in \mathbb{R}^d$ turns into the $(d+k)$-dimensional vector $\boldsymbol{\omega}'_t \in \mathbb{R}^{d+k}$, where the first $d$ components of the two vectors are the same, and the $(d+i)$-th component of $\boldsymbol{\omega}'_t$ is simply the regret generated for the $i$-th extra action $\mathbf{b}_i$.

\subsubsection{Removing an Individual Policy For Best of Both Worlds Regret}\label{section::removing_individual_policy} 

For the best of both worlds regret in \pref{sec:bbw} we adopt the following view. Here, each policy corresponds to a single action $a \in \cA$, i.e., the policy space of the learner is $\Pi = \{ \indicator{a} \colon a \in \cA\}$. In this case, the random draw of $\mathbf{a}_t \sim p_t$ is interpreted as a random choice of \emph{policy}. Note that this view does not impact the regret guarantee and is consistent with approach for extendability above. This view is necessary for a positive gap to be possible and removing a certain policy can be easily implemented by removing the corresponding action.
\subsubsection{Proofs}

We work under the assumption all the rewards $r_t$ have values between $-1$ and $1$:
\begin{assumption}[Boundedness]
\label{ass:boundedness}
The true rewards $r_t$ 
are bounded in that $\forall \mathbf{a} \in \mathcal{A}$ and $\forall t \in \mathbb{N}$ we have
$
|\mathbf{a}^\top \boldsymbol{\omega}_t| \leq 1$. %
\end{assumption}

For all $\mathbf{a} \in \mathcal{A}$, define 
$$
\widehat{r}_t(\mathbf{a}) = \mathbf{a}^\top \widehat{\boldsymbol{\omega}}_t\ \ \ {\mbox{where}}\ \ \ \widehat{\boldsymbol{\omega}}_t = b_t\frac{r_t \left( \boldsymbol{\Sigma}_t \right)^{-1} \mathbf{a}_t}{\rho}~.
$$ 
We will use the notation $\mathbb{E}[ \cdot | \mathcal{F}_t]$ to denote the conditional expectation where the sigma algebra $\mathcal{F}_t$ is generated by the random variables $\left( \boldsymbol{\omega}_1, b_1,  \mathbf{a}_1,  \cdots,\boldsymbol{\omega}_{t-1}, b_{t-1}, \mathbf{a}_{t-1}, \boldsymbol{\omega}_t, b_t, \mathbf{a}_t \right)$. Let $\mathcal{F}_t^-$ be the sigma algebra generated by $\left( \boldsymbol{\omega}_1, b_1, \mathbf{a}_1,  \cdots,\boldsymbol{\omega}_{t-1}, b_{t-1}, \mathbf{a}_{t-1}, \boldsymbol{\omega}_t \right)$. Observe that $\widehat{\boldsymbol{\omega}}_t$, $\widehat{r}_t(\cdot)$ and $\widetilde{r}_t(\cdot)$ are $\mathcal{F}_{t}$ measurable, $\boldsymbol{\Sigma}_t$ is $\mathcal{F}_{t-1}$ measurable, $\mathbb{E}\left[    \widehat{\boldsymbol{\omega}}_t | \mathcal{F}^-_{t}\right] = \boldsymbol{\omega}_t $ and $\mathbb{E}\left[b_t | \mathcal{F}_{t-1} \right]= \mathbb{E}\left[b_t | \mathcal{F}^{-}_{t} \right] = \rho$. When considering $\mathbb{E}\left[ \cdot | \mathcal{F}_{t}^{-} \right] $, the expectation is over $\mathbf{a}_t$, $b_t$ holding $\boldsymbol{\omega}_t$ fixed.  
Every time we consider an expectation of the form 
$\mathbb{E}_{\mathbf{a} \sim p_t} \left[ \widehat{r}_t(\mathbf{a}) |\mathcal{F}_t^{-}\right] $, 
$\mathbb{E}_{\mathbf{a} \sim p_t} \left[ \widetilde{r}_t(\mathbf{a}) |\mathcal{F}_t^{-}\right] $, $\mathbb{E}_{\mathbf{a} \sim p_t} \left[ \widehat{r}_t(\mathbf{a}) | \mathcal{F}_t\right] $ 
or 
$\mathbb{E}_{\mathbf{a} \sim p_t} \left[ \widetilde{r}_t(\mathbf{a}) | \mathcal{F}_t\right] $ the random variable $\mathbf{a}$ is a sample from $p_t$ conditionally independent from $\mathbf{a}_t$ given $\mathcal{F}_t$ or $\mathcal{F}_t^{-}$. 
Moreover, for notational simplicity, whenever possible, we will omit the absolute multiplicative constants, and instead resort to a big-oh notation.

\begin{lemma}\label{lemma::useful_properties}
Let 
\begin{equation}\label{equation::upper_bound_quadratic_form}
    \sup_{\mathbf{a}, \mathbf{b} \in \mathcal{A}} \mathbf{a}^\top \boldsymbol{\Sigma}_t^{-1}  \mathbf{b} \leq \frac{c(d)}{\gamma_t}~,
\end{equation}
for some function $c(\cdot)$ whose value will be detailed later on.
Then for any fixed $\mathbf{a} \in \mathcal{A}$ and $t \in \mathbb{N}$ the following holds:
\begin{enumerate}
\item $|\widehat{r}_t(\mathbf{a}) | \leq \frac{c(d)}{\rho\gamma_t}$~;
\item $\mathbf{a}^\top \left(  \boldsymbol{\Sigma}_t\right)^{-1} \mathbf{a} \leq \frac{c(d)}{\gamma_t}$~;
\item $\mathbb{E}_{\widetilde{\mathbf{a}} \sim p_t} \left[ \widetilde{\mathbf{a}}^\top \left( \boldsymbol{\Sigma}_t \right)^{-1} \widetilde{\mathbf{a}}    \Big|  \mathcal{F}_{t-1} \right] = d$\qquad {\mbox{ and}}\qquad $\mathbb{E}_{\widetilde{\mathbf{a}} \sim p_t} \left[ \widetilde{\mathbf{a}}^\top \left( \boldsymbol{\Sigma}_t \right)^{-1} \widetilde{\mathbf{a}}    \Big|  \mathcal{F}_{t}^{-} \right] = d$~; %
\item $\mathbb{E} \left[ \widehat{r}_t^2(\mathbf{a}) | \mathcal{F}^{-}_t \right] \leq \frac{\mathbf{a}^\top \left(  \boldsymbol{\Sigma}_t\right)^{-1} \mathbf{a}}{\rho}$~.
\end{enumerate}
\end{lemma}
\begin{proof}
Item 1 simply follows by recalling that $ \widehat{r}_t(\mathbf{a})  = b_t \frac{ r_t\, \mathbf{a}^\top \left(  \boldsymbol{\Sigma}_t\right)^{-1} \mathbf{a}_t}{\rho}$ with $|r_t| \leq 1$. The condition from Equation~\ref{equation::upper_bound_quadratic_form} then implies the result. 
Item 2 follows from the same condition. Item 3 follows by observing that
\begin{align*}
\mathbb{E}_{\widetilde{\mathbf{a}} \sim p_t} [\widetilde{\mathbf{a}}^\top \left(  \boldsymbol{\Sigma}_t\right)^{-1} \widetilde{\mathbf{a}} | \mathcal{F}_{t-1}] 
=
\mathbb{E}_{\widetilde{\mathbf{a}} \sim p_t} [\operatorname{tr}(\widetilde{\mathbf{a}}\widetilde{\mathbf{a}}^\top \boldsymbol{\Sigma}_t^{-1})| \mathcal{F}_{t-1}]
=
\operatorname{tr}\left(\mathbb{E}_{\widetilde{\mathbf{a}} \sim p_t} [\widetilde{\mathbf{a}}\widetilde{\mathbf{a}}^\top]\boldsymbol{\Sigma}_t^{-1}\right) = d~.
\end{align*}
Item 4 follows from the definition of $\widehat{r}_t(\mathbf{a})$. In fact, for any fixed $\mathbf{a}$, we can write
\begin{align*}
    \mathbb{E}\left[  \widehat{r}_t^2(\mathbf{a})\,|\, \mathcal{F}^{-}_{t} \right] &= \frac{1}{\rho^2} \mathbb{E}_{b_t \sim \mathrm{Ber}(\rho), \mathbf{a}_t \sim p_t}[ b_t^2 r_t^2 \mathbf{a}^\top \left(  \boldsymbol{\Sigma}_t\right)^{-1} \mathbf{a}_t \mathbf{a}_t^\top \left(  \boldsymbol{\Sigma}_t\right)^{-1} \mathbf{a} \,|\, \mathcal{F}^{-}_{t} ]\\
    &= \frac{1}{\rho} \mathbb{E}_{ \mathbf{a}_t \sim p_t}[  r_t^2 \mathbf{a}^\top \left(  \boldsymbol{\Sigma}_t\right)^{-1} \mathbf{a}_t \mathbf{a}_t^\top \left(  \boldsymbol{\Sigma}_t\right)^{-1} \mathbf{a}  \,|\, \mathcal{F}^{-}_{t} ]\\
    &\stackrel{(i)}{\leq} \frac{1}{\rho}\mathbb{E}_{\mathbf{a}_t \sim p_t}[  \mathbf{a}^\top \left(  \boldsymbol{\Sigma}_t\right)^{-1} \mathbf{a}_t \mathbf{a}_t^\top \left(  \boldsymbol{\Sigma}_t\right)^{-1} \mathbf{a} \,|\, \mathcal{F}^{-}_{t}  ]\\
    &= \frac{1}{\rho}\mathbf{a}^\top \left(  \boldsymbol{\Sigma}_t\right)^{-1} \mathbb{E}_t\left[ \mathbf{a}_t \mathbf{a}_t^\top\right] \left(  \boldsymbol{\Sigma}_t\right)^{-1} \mathbf{a}\\
    &= \frac{1}{\rho}\mathbf{a}^\top \left(  \boldsymbol{\Sigma}_t\right)^{-1} \mathbf{a}~,
\end{align*}
where $(i)$ holds because $|r_t| \leq 1$. 
\end{proof}

This allows us to prove the following version of Lemma 5 in~\cite{bartlett2008high},

\begin{lemma}\label{lemma::weighted_lower_bound_fixed_action_high_prob}
Let $\{ \alpha_\ell\}_{\ell=1}^\infty$ be a sequence of deterministic nonnegative weights satisfying $\alpha_\ell \leq 1$ for all $\ell \in \mathbb{N}$. Let $\delta' = \frac{\delta}{|\mathcal{A}|} $. Then with probability at least $1-\delta$, simultaneously for all $\mathbf{a} \in \mathcal{A}$ and all $t \in \mathbb{N}$,
\begin{align}
    \sum_{\ell=1}^t \alpha_\ell \widetilde{r}_\ell(\mathbf{a})\geq  \sum_{\ell = 1}^t  \alpha_\ell\, \mathbf{a}^\top \boldsymbol{\omega}_\ell -   \Ocal\left(\sqrt{\frac{dt}{\rho}\ln \frac{t}{\delta'}} + B_t \ln \frac{t}{\delta'}\right)~,
\end{align}
where $B_t = \max_{\ell \leq t} \frac{c(d)\alpha_\ell}{\rho \gamma_\ell} + \alpha_\ell$. 
\end{lemma}
\begin{proof}
Fix $\mathbf{a} \in \mathcal{A}$, and recall the definition of $\widetilde{r}_\ell(\mathbf{a})$ in Algorithm \ref{Alg:AnytimeGeoHedge}.
Define $M_t(\mathbf{a}) = \alpha_t\,  \mathbf{a}^\top \boldsymbol{\omega}_t  -  \alpha_t \widehat{r}_t(\mathbf{a}) $, and notice that $\{M_t(\mathbf{a})\}_{t=1,2,\ldots}$ is a martingale difference sequence. Using Lemma~\ref{lemma::useful_properties} (Item 1), along with Assumption \ref{ass:boundedness},
we see that
\[
| M_t(\mathbf{a})| \leq \frac{c(d) \alpha_t}{\rho\gamma_t}  + \alpha_t ~.
\] 
Let $V_t(\mathbf{a}) = \sum_{\ell=1}^t \mathrm{Var}[M_\ell(\mathbf{a})\,|\,\mathcal{F}_\ell^- ]$ be the sum of conditional variances of variables $M_{\ell}(\mathbf{a})$.
Using Lemma~\ref{lemma:simplified_freedman_norm} we see that with probability at least $1-\delta'$, simultaneously for all $t$,
\begin{align}\label{equation::applying_freedman1_weighted}
    \sum_{\ell=1}^t \alpha_\ell \widehat{r}_\ell(\mathbf{a}) \geq \sum_{\ell = 1}^t  \alpha_\ell\, \mathbf{a}^\top \boldsymbol{\omega}_\ell - \Ocal\left(\sqrt{V_t \ln \frac{t}{\delta'}} + B_t \ln \frac{t}{\delta'}\right)~.
\end{align}
Since 
$\mathrm{Var}[M_t(\mathbf{a})\,|\,\mathcal{F}_\ell^- ] 
\leq \mathbb{E}[M^2_t(\mathbf{a})\,|\,\mathcal{F}_\ell^- ]  
\leq 
\alpha_t^2\mathbb{E}_t\left[ \widehat{r}^2_t(\mathbf{a})\right]$,
by Lemma~\ref{lemma::useful_properties} (Item 4) we can write
\begin{align*}
\sqrt{V_t(\mathbf{a})} &\leq  \sqrt{ \sum_{\ell=1}^t \frac{\alpha^2_\ell  \mathbf{a}^\top \left(  \boldsymbol{\Sigma}_\ell\right)^{-1}  \mathbf{a}}{\rho} } \\
&\leq  
\sqrt{ \left( \frac{1}{\sqrt{dt}} \sum_{\ell=1}^t \frac{\alpha^2_\ell  \mathbf{a}^\top \left(  \boldsymbol{\Sigma}_\ell\right)^{-1}  \mathbf{a}}{\sqrt{\rho}} \right) \sqrt{\frac{dt}{\rho}} } \\
&\leq
\frac{1}{2}\left( \frac{1}{\sqrt{dt}} \sum_{\ell=1}^t \frac{\alpha^2_\ell \mathbf{a}^\top \left(  \boldsymbol{\Sigma}_\ell\right)^{-1}  \mathbf{a}}{\sqrt{\rho }} +  \sqrt{\frac{dt}{\rho}} \right)~,
\end{align*}

the last inequality being the arithmetic-geometric inequality $\sqrt{ab} \leq \frac{1}{2}(a+b)$. Substituting back into Eq.~(\ref{equation::applying_freedman1_weighted}) gives
\begin{align*}
\sum_{\ell=1}^t \alpha_\ell  \widehat{r}_\ell(\mathbf{a}) &\geq \sum_{\ell = 1}^t  \alpha_\ell\,  \mathbf{a}^\top \boldsymbol{\omega}_\ell - \Ocal\left( \left( \frac{ \sum_{\ell=1}^t \alpha^2_\ell \mathbf{a}^\top \left(  \boldsymbol{\Sigma}_\ell\right)^{-1}  \mathbf{a}     }{\sqrt{\rho dt}} + \sqrt{\frac{dt}{\rho}} \right)\sqrt{\ln \frac{t}{\delta'}} + B_t \ln \frac{t}{\delta'}\right)
\end{align*}
with probability at least $1-\delta'$ for all $t\in\mathbb{N}$.
Since the function $g(t) = \frac{\ln \frac{t}{\delta'}}{t}$
is decreasing for all $t \geq 1$ we conclude that  $\frac{\ln \frac{\ell}{\delta'}}{d\ell}$ is a decreasing function of $\ell$. Using this last fact together with the condition $\alpha_\ell \leq 1$ we see that 
\[
\alpha_\ell \frac{\mathbf{a}^\top \left( \boldsymbol{\Sigma}_\ell \right)^{-1}\mathbf{a}}{\sqrt{\rho d\ell} }\sqrt{\ln \frac{\ell}{\delta'} } \geq \alpha_\ell \frac{\mathbf{a}^\top \left( \boldsymbol{\Sigma}_\ell \right)^{-1}\mathbf{a}}{\sqrt{\rho dt} }\sqrt{\ln \frac{t}{\delta'} } \geq \alpha^2_\ell \frac{\mathbf{a}^\top \left( \boldsymbol{\Sigma}_\ell \right)^{-1}\mathbf{a}}{\sqrt{\rho dt} }\sqrt{\ln \frac{t}{\delta'} }~,
\]
and therefore
\begin{align*}
\sum_{\ell=1}^t \alpha_\ell \widehat{r}_\ell(\mathbf{a}) + \alpha_\ell \frac{\mathbf{a}^\top \left( \boldsymbol{\Sigma}_\ell \right)^{-1}\mathbf{a}}{\sqrt{\rho d\ell} }\sqrt{\ln \frac{\ell}{\delta'} } 
&\geq 
\sum_{\ell=1}^t \alpha_\ell \widehat{r}_\ell(\mathbf{a}) + \alpha_\ell \frac{\mathbf{a}^\top \left( \boldsymbol{\Sigma}_\ell \right)^{-1}\mathbf{a}}{\sqrt{\rho dt} }\sqrt{\ln \frac{t}{\delta'} } \\
&\geq 
\sum_{\ell = 1}^t  \alpha_\ell \mathbf{a}^\top \boldsymbol{\omega}_\ell -  \Ocal\left(\sqrt{\frac{dt}{\rho} \ln \frac{t}{\delta'}}  + B_t \ln \frac{t}{\delta'}\right)
\end{align*}
with probability at least $1-\delta'$ for all $t \in \mathbb{N}$. The result follows by taking a union bound over all $\mathbf{a} \in \mathcal{A}$.
\end{proof}

In particular when all weights $\alpha_{\ell} = 1$ Lemma~\ref{lemma::weighted_lower_bound_fixed_action_high_prob} implies the following.

\begin{corollary}\label{corollary::lower_bound_fixed_action_high_prob} 
With the same notation as in Lemma~\ref{lemma::weighted_lower_bound_fixed_action_high_prob}, with probability at least $1-\delta$ simultaneously for all $\mathbf{a} \in \mathcal{A}$ and all $t \in \mathcal{N}$,
\begin{equation*}
      \sum_{\ell=1}^t  \widetilde{r}_\ell(\mathbf{a}) \geq  \sum_{\ell = 1}^t  \mathbf{a}^\top \boldsymbol{\omega}_\ell -  \Ocal\left(\sqrt{\frac{dt}{\rho}\ln \frac{t}{\delta'}} + B_t \ln \frac{t}{\delta'}\right)~,
\end{equation*}
where $ B_t = \max_{\ell \leq t} \frac{c(d)}{\rho \gamma_\ell} + 1$.
\end{corollary}

We now proceed to upper bound $|\widetilde{r}_t(\mathbf{a})|$. This will inform our choice for learning rate $\eta_t$. 

\begin{lemma}
Let $\delta' = \frac{\delta}{|\mathcal{A}|} $. For all $\mathbf{a} \in \mathcal{A}$, 
$| \widetilde{r}_t(\mathbf{a}) | = \Ocal\left(  \frac{c(d)}{\rho \gamma_t} + \left(\frac{c(d)}{ \gamma_t \sqrt{\rho dt}} \sqrt{\ln \frac{t}{\delta'} }\right)\right)$.
\end{lemma}
\begin{proof}
For each $\mathbf{a} \in \mathcal{A}$, we can write
\begin{align*}
| \widetilde{r}_t(\mathbf{a} ) | 
&= \Ocal\left( |\widehat{r}_t(\mathbf{a}) | + \frac{\mathbf{a}^\top \left( \boldsymbol{\Sigma}_t \right)^{-1}\mathbf{a}}{\sqrt{\rho dt} }\sqrt{\ln \frac{t}{\delta'}}\right)    \\
&=\Ocal\left( \frac{c(d)}{\rho \gamma_t} + \left(\frac{c(d)}{\gamma_t \sqrt{\rho dt}} \sqrt{\ln \frac{t}{\delta'} } \right)\right)~, 
\end{align*}
the last inequality holding as a consequence of Lemma~\ref{lemma::useful_properties}. 
\end{proof}

For the analysis of exponential weights, we will insure that $|\eta_t \widetilde{r}_t(\mathbf{a}) | \leq 1 $ for all $\mathbf{a}$ and $t$. This imposes the restriction of the following form 
\begin{equation}\label{equation::constraint_eta}
\eta_t 
=
\Ocal\left( \frac{\rho}{\frac{c(d)}{\gamma_t} + \left(\frac{c(d)}{\gamma_t \sqrt{dt}} \sqrt{\rho \ln \frac{t}{\delta'} }\right)} \right)
= 
\Ocal\left(\frac{\rho \gamma_t}{c(d) + \frac{c(d)}{ \sqrt{dt}} \sqrt{\rho \ln \frac{t}{\delta'} }}\right)~.
\end{equation}

We are now ready to tackle the anytime high probability regret guarantees for Algorithm \ref{Alg:AnytimeGeoHedge}.

\begin{lemma}\label{l:basiclemma}
Let the condition in Eq. (\ref{equation::constraint_eta}) hold with a nonincreasing sequence of learning rates $\eta_t$. Then for all $\bar{\mathbf{a} }\in \mathcal{A}$ and $t \in \mathbb{N}$
\begin{equation*}
\sum_{\ell=1}^t \widetilde{r}_{\ell}(\bar{\mathbf{a}}) \leq 1+ \frac{\ln(\mathcal{A})}{\eta_{t}} + \sum_{\ell=1}^{t} \frac{1}{1-\gamma_\ell} \left(\mathbb{E}_{\mathbf{a} \sim p_{\ell}}\left[    \widetilde{r}_{\ell}(\mathbf{a}) + \eta_\ell \left( \widetilde{r}_{\ell}(\mathbf{a}) \right)^2 \Big|\mathcal{F}_{\ell} \right]  - \gamma_\ell \mathbb{E}_{\mathbf{a} \sim  p_{\mathrm{E}}(\mathbf{a}) }\left[ \widetilde{r}_\ell(\mathbf{a}) \Big| \mathcal{F}_{\ell} \right]\right)~.
\end{equation*}
\end{lemma}
\begin{proof}
Recall that
\begin{equation*}
    w_\ell(\mathbf{a}) = \exp\left(\eta_\ell \sum_{\ell'=1}^{\ell-1} \widetilde{r}_{\ell'}(\mathbf{a}) \right)  
\end{equation*}
and $W_\ell = \sum_{\mathbf{a} \in \mathcal{A}} w_\ell(\mathbf{a}) $. Let us also define
\begin{equation*}
    w_\ell^{-}(\mathbf{a}) = \exp\left(  \eta_{\ell-1}\sum_{\ell'=1}^{\ell-1} \widetilde{r}_{\ell'}(\mathbf{a})    \right)
\end{equation*}
and, $W_\ell^{-} = \sum_{\mathbf{a} \in \mathcal{A}} w_\ell^{-}(\mathbf{a})$. Moreover, set for brevity $A_\ell = \ln\left( \frac{W_{\ell+1}^{-}}{W_\ell} \right)$. We can write
\begin{align*}
\exp(A_\ell)  
&= 
\frac{W_{\ell+1}^{-}}{W_\ell} \\
&= 
\frac{\sum_{\mathbf{a} \in \mathcal{A}} \exp\left( \eta_{\ell}\sum_{\ell'=1}^{\ell} \widetilde{r}_{\ell'}(\mathbf{a})   \right)  }{ \sum_{\mathbf{a} \in \mathcal{A}} \exp\left(  \eta_{\ell}\sum_{\ell'=1}^{\ell-1} \widetilde{r}_{\ell'}(\mathbf{a})   \right) } \\
&= 
\sum_{\mathbf{a} \in \mathcal{A}} q_\ell(\mathbf{a}) \exp\left( \eta_{\ell} \widetilde{r}_{\ell}(\mathbf{a})  \right) \\
&\leq
1 + \sum_{\mathbf{a} \in \mathcal{A}} q_\ell(\mathbf{a}) \eta_{\ell} \widetilde{r}_{\ell}(\mathbf{a}) + q_\ell(\mathbf{a}) \eta^2_{\ell} \left( \widetilde{r}_{\ell}(\mathbf{a}) \right)^2  
\end{align*}
the last inequality holding because $e^{x} \leq 1+x + x^2$ whenever $|x| \leq 1$. Taking logs and using the fact that $\ln(1+x) \leq x$ yields
\begin{align*}
   A_\ell &\leq  \eta_\ell \sum_{\mathbf{a} \in \mathcal{A}} q_\ell(\mathbf{a})  \widetilde{r}_{\ell}(\mathbf{a}) + q_\ell(\mathbf{a}) \eta_{\ell} \left( \widetilde{r}_{\ell}(\mathbf{a}) \right)^2\\
   &\stackrel{(i)}{\leq}  \eta_\ell \sum_{\mathbf{a} \in \mathcal{A}} q_\ell(\mathbf{a})  \widetilde{r}_{\ell}(\mathbf{a}) + \frac{p_\ell(\mathbf{a})}{1-\gamma_\ell} \eta_{\ell} \left( \widetilde{r}_{\ell}(\mathbf{a}) \right)^2\\
   &\stackrel{(ii)}{=} \frac{\eta_\ell}{1-\gamma_\ell} \sum_{\mathbf{a} \in \mathcal{A}} p_\ell(\mathbf{a})  \widetilde{r}_{\ell}(\mathbf{a}) + p_\ell(\mathbf{a})\eta_{\ell} \left( \widetilde{r}_{\ell}(\mathbf{a}) \right)^2 -\gamma_\ell p_{\mathrm{E}}(\mathbf{a}) \widetilde{r}_\ell(\mathbf{a})~,
\end{align*}
where $(i)$ follows because $q_\ell(\mathbf{a}) \leq \frac{p_\ell(\mathbf{a})}{1-\gamma_\ell}$ and $(ii)$ because $p_\ell = (1-\gamma_\ell) q_\ell(\mathbf{a}) + \gamma_\ell p_{\mathrm{E}}(\mathbf{a})$. Hence
\begin{align}
\sum_{\ell=1}^{t} \frac{A_\ell}{\eta_\ell} &=    \sum_{\ell=1}^{t} \frac{1}{\eta_\ell} \ln\left(  \frac{W_{\ell+1}^{-}}{W_\ell} \right) \notag \\
&\leq \sum_{\ell=1}^{t} \frac{1}{1-\gamma_\ell} \sum_{\mathbf{a} \in \mathcal{A}} p_\ell(\mathbf{a})  \widetilde{r}_{\ell}(\mathbf{a}) + p_\ell(\mathbf{a})\eta_{\ell} \left( \widetilde{r}_{\ell}(\mathbf{a}) \right)^2 -\gamma_\ell p_{\mathrm{E}}(\mathbf{a}) \widetilde{r}_\ell(\mathbf{a})~. \label{equation::upper_bound_potential}
\end{align}
Now, define the following potential
$$
\Phi_\ell(\eta) = \frac{1}{\eta} \ln\left( \frac{1}{|\mathcal{A}|} \sum_{\mathbf{a} \in \mathcal{A} } \exp\left( \eta \sum_{\ell'=1}^{\ell-1} \widetilde{r}_{\ell'}(\mathbf{a})   \right)  \right)~. 
$$ 
Notice that by de L'Hopital's rule, this implies $\lim_{\eta \rightarrow 0}\Phi_{\ell}(\eta) = \Phi_{\ell}(0) = 1$ for all $\ell$. 
Let $\eta_0 = 0$. We have
\begin{align*}
   1 +  \sum_{\ell=1}^{t} \frac{1}{\eta_\ell} \ln\left(  \frac{W_{\ell+1}^{-}}{W_\ell} \right) &= \Phi_1(\eta_0) +  \sum_{\ell=1}^{t} \Phi_{\ell +1}(\eta_\ell) -   \Phi_{\ell }(\eta_\ell) \\
    &=  \left(\sum_{\ell = 1}^{t} \left( \Phi_{\ell}(\eta_{\ell-1} ) - \Phi_\ell(\eta_\ell)  \right) \right) + \Phi_{t+1}(\eta_{t})~. 
\end{align*}
Next, we now show that for all $\ell$ the function $\Phi_{\ell}(\eta)$ is an increasing function of $\eta$. To this effect, let $p_\ell^\eta(\mathbf{a}) = \frac{\exp\left( \eta \sum_{\ell'=1}^{\ell-1} \widetilde{r}_{\ell'}(\mathbf{a}) \right) }{\sum_{\mathbf{a}' \in \mathcal{A}} \exp\left( \eta \sum_{\ell'=1}^{\ell-1} \widetilde{r}_{\ell'}(\mathbf{a}') \right) } $. Observe that the following relationship holds,
\begin{align*}
    \Phi_{\ell}'(\eta) &= \frac{-1}{\eta^2} \ln\left( \frac{1}{|\mathcal{A}|} \sum_{\mathbf{a} \in \mathcal{A}} \exp\left( \eta \sum_{\ell'}^{\ell-1} \widetilde{r}_{\ell'}(\mathbf{a}) \right) \right)  + \frac{1}{\eta} \frac{  \sum_{\mathbf{a} \in \mathcal{A} }\left[  \sum_{\ell'=1}^{\ell-1} \widetilde{r}_{\ell'}(\mathbf{a})\right] \exp\left( \eta \sum_{\ell'=1}^{\ell-1} \widetilde{r}_{\ell'}(\mathbf{a}) \right)  }{ \sum_{\mathbf{a} \in \mathcal{A}} \exp\left( \eta \sum_{\ell'=1}^{\ell-1} \widetilde{r}_{\ell'}(\mathbf{a})  \right)     } \\
    &= \frac{1}{\eta^2}\sum_{\mathbf{a} \in \mathcal{A}} p_\ell^\eta(\mathbf{a}) \left( \eta \sum_{\ell'=1}^{\ell-1} \widetilde{r}_{\ell'}(\mathbf{a}) - \ln\left( \frac{1}{|\mathcal{A}|} \sum_{\mathbf{a}' \in \mathcal{A}} \exp\left(  \eta \sum_{\ell'=1}^{\ell-1} \widetilde{r}_{\ell'}(\mathbf{a}') \right)\right)         \right)\\
    &= \frac{1}{\eta^2}\mathrm{KL}\left( p_\ell^\eta , \mathrm{Uniform}(\mathcal{A}) \right) \\
    &\geq 0~,
\end{align*}
where KL($\cdot$,$\cdot$) denotes the Kullback Leibler divergence between the two distributions at arguments.

Since we are assuming $\eta_{\ell} \leq \eta_{\ell-1}$, this implies that $\Phi_\ell(\eta_{\ell-1}) \geq \Phi_{\ell}(\eta_\ell)$. Thus,
\begin{align}
1+ \sum_{\ell=1}^{t} \frac{1}{\eta_\ell} \ln\left( \frac{W_{\ell+1}^{-}}{W_\ell} \right) 
=  \left(\sum_{\ell = 1}^{t} \left( \Phi_{\ell}(\eta_{\ell-1} ) - \Phi_\ell(\eta_\ell)  \right) \right)+ \Phi_{t+1}(\eta_{t})
\geq 
\Phi_{t+1}(\eta_{t}).\label{equation::lower_bound_potential}
\end{align}
Combining (\ref{equation::upper_bound_potential}) with~(\ref{equation::lower_bound_potential}) gives
\begin{align*}
\Phi_{t+1}(\eta_{t}) 
&\leq   
1+\sum_{\ell=1}^{t} \frac{1}{\eta_\ell} \ln\left(  \frac{W_{\ell+1}^{-}}{W_\ell} \right) \\
&\leq 
1 + \sum_{\ell=1}^{t} \frac{1}{1-\gamma_\ell} \sum_{\mathbf{a} \in \mathcal{A}} p_\ell(\mathbf{a})  \widetilde{r}_{\ell}(\mathbf{a}) + p_\ell(\mathbf{a})\eta_{\ell} \left( \widetilde{r}_{\ell}(\mathbf{a}) \right)^2 -\gamma_\ell p_{\mathrm{E}}(\mathbf{a}) \widetilde{r}_\ell(\mathbf{a})~.
\end{align*}
Since for any $\bar{\mathbf{a}} \in \mathcal{A}$ the potential $\Phi_{t+1}(\eta_{t})$ satisfies,
\begin{equation*}
   \Phi_{t+1}(\eta_{t}) = \frac{1}{\eta_{t}} \ln\left( \frac{1}{|\mathcal{A}|} \sum_{\mathbf{a} \in \mathcal{A} } \exp\left( \eta_{t} \sum_{\ell=1}^{t} \widetilde{r}_{\ell}(\mathbf{a})   \right)  \right) \geq \sum_{\ell=1}^t \widetilde{r}_{\ell}(\bar{\mathbf{a}})- \frac{\ln(\mathcal{A})}{\eta_{t}} 
\end{equation*}
we have
\begin{equation*}
    \sum_{\ell=1}^t \widetilde{r}_{\ell}(\bar{\mathbf{a}}) \leq 1+ \frac{\ln(\mathcal{A})}{\eta_{t}} + \sum_{\ell=1}^{t} \frac{1}{1-\gamma_\ell} \sum_{\mathbf{a} \in \mathcal{A}} \left(p_\ell(\mathbf{a})  \widetilde{r}_{\ell}(\mathbf{a}) + p_\ell(\mathbf{a})\eta_{\ell} \left( \widetilde{r}_{\ell}(\mathbf{a}) \right)^2 -\gamma_\ell p_{\mathrm{E}}(\mathbf{a}) \widetilde{r}_\ell(\mathbf{a})\right)
\end{equation*}

The claimed result now follows by simply observing that
\begin{align*}
  \sum_{\mathbf{a} \in \mathcal{A}} p_\ell(\mathbf{a})  \widetilde{r}_{\ell}(\mathbf{a}) + p_\ell(\mathbf{a})\eta_{\ell} \left( \widetilde{r}_{\ell}(\mathbf{a}) \right)^2 &= \mathbb{E}_{\mathbf{a} \sim p_{\ell}}\left[    \widetilde{r}_{\ell}(\mathbf{a}) + \eta_\ell \left( \widetilde{r}_{\ell}(\mathbf{a}) \right)^2 \Big|\mathcal{F}_{\ell} \right]~.
\end{align*}
\end{proof}

In the sequel, we shall impose the restriction
\begin{equation}\label{e:gammarestriction}
\gamma_t \in (0,1/2]~,
\end{equation}
holding for all $t$,
so that $\frac{1}{1-\gamma_t} \leq 2$ and $\frac{\gamma_t}{1-\gamma_t} \leq 1$. 

To get a high probability anytime bound starting from Lemma \ref{l:basiclemma}, we are required to prove high probability bounds for each of the terms $\mathbf{I}, \mathbf{II}$, $\mathbf{III}$ and $\mathbf{IV}$ defined below: 
\begin{align*}
    \mathbf{I} &= \sum_{\ell=1}^t \widetilde{r}_\ell(\bar{\mathbf{a}} ).\\
    \mathbf{II} &= \sum_{\ell=1}^t \frac{1}{1-\gamma_\ell}   \mathbb{E}_{\mathbf{a} \sim p_\ell}\left[  \widetilde{r}_{\ell}(\mathbf{a}) | \mathcal{F}_\ell \right] . \\
    \mathbf{III} &= \sum_{\ell=1}^t \frac{\eta_\ell}{1-\gamma_\ell} \mathbb{E}_{\mathbf{a} \sim p_\ell} \left[  \left( \widetilde{r}_\ell(\mathbf{a})\right)^2 | \mathcal{F}_\ell \right]\\
    \mathbf{IV} &= -\sum_{\ell=1}^{t} \frac{\gamma_\ell}{1-\gamma_\ell} \mathbb{E}_{\mathbf{a} \sim p_E} \left[  \widetilde{r}_\ell(\mathbf{a}) | \mathcal{F}_\ell \right] 
\end{align*}

We proceed by (upper or lower) bounding the four terms above in turn.

\paragraph{Bounding term $\mathbf{I}$.}

By Corollary~\ref{corollary::lower_bound_fixed_action_high_prob} with probability at least $1-\delta$ for all $\bar{\mathbf{a}} \in \mathcal{A}$ and all $t \in \mathbb{N}$ simultaneously, 
\begin{align*}
    \mathbf{I} &\geq \sum_{\ell = 1}^t  \bar{\mathbf{a}}^\top \boldsymbol{\omega}_\ell -  \Ocal\left(\sqrt{\frac{dt}{\rho}\ln \frac{t}{\delta'}} - \left( \max_{\ell \leq t} \frac{c(d)}{\rho \gamma_\ell} + 1\right) \ln \frac{t}{\delta'}\right)~. 
\end{align*}
Let us denote the event where this bound holds by $\mathcal{E}_{\mathbf{I}}$. The preceding discussion implies $\mathbb{P}\left( \mathcal{E}_{\mathbf{I}} \right) \geq 1-\delta$.

\paragraph{Bounding term $\mathbf{II}$.}
Recalling the definition of $\widetilde{r}_\ell(\mathbf{a})$, we can write
\begin{align}
\mathbf{II} 
&= 
\sum_{\ell=1}^t \frac{1}{1-\gamma_\ell} \mathbb{E}_{\mathbf{a} \sim p_{\ell}} \left[ \widetilde{r}_\ell(\mathbf{a}) |\mathcal{F}_\ell   \right] \notag \\
&= 
\sum_{\ell=1}^t \frac{1}{1-\gamma_\ell} \mathbb{E}_{\mathbf{a} \sim p_{\ell}}\left[   \widehat{r}_\ell(\mathbf{a}) + \Ocal\left(\frac{\mathbf{a}^\top \Sigma_\ell^{-1} \mathbf{a} }{ \sqrt{\rho d\ell} } \sqrt{\ln\left( \frac{\ell}{\delta'}\right) }\right)    \Big|\mathcal{F}_\ell \right] \notag \\
&\stackrel{(i)}{=} 
\sum_{\ell=1}^t \frac{1}{1-\gamma_\ell} \mathbb{E}_{\mathbf{a} \sim p_{\ell}}\left[   \widehat{r}_\ell(\mathbf{a})  |\mathcal{F}_\ell    \right] + \Ocal\left(\frac{1}{(1-\gamma_\ell)}\sqrt{\frac{d}{\rho \ell}}\sqrt{\ln\left( \frac{\ell}{\delta'}\right) }\right) \notag \\
&\stackrel{(ii)}{\leq}
\sum_{\ell=1}^t \frac{1}{1-\gamma_\ell} \mathbb{E}_{\mathbf{a} \sim p_{\ell}}\left[   \widehat{r}_\ell(\mathbf{a})    |\mathcal{F}_\ell  \right]  + \Ocal\left(\sqrt{\frac{d t}{\rho}\ln\left(\frac{ t}{\delta'}\right)}\right)~, \label{equation::upper_bounding_II}
\end{align}

where $(i)$ follows from Item $3$ in Lemma~\ref{lemma::useful_properties}, and in $(ii)$ we have used $\frac{1}{1-\gamma_\ell} \leq 2$, $\ln\left(\frac{\ell}{\delta'}\right) \leq \ln\left(\frac{t}{\delta'}\right)$, along with $\sum_{\ell=1}^t \sqrt{\frac{d}{\ell}} \leq 2 \sqrt{dt}$ . 

We are left to prove a high probability upper bound for $\sum_{\ell=1}^t \frac{1}{1-\gamma_\ell} \mathbb{E}_{\mathbf{a} \sim p_{\ell}}\left[   \widehat{r}_\ell(\mathbf{a})   | \mathcal{F}_\ell  \right]$ which we achieve through the following Lemma.

\begin{lemma}\label{lemma::bounding_event_II}
With probability at least $1-\delta$ for all $t \in \mathbb{N}$,
\begin{equation*}
  \sum_{\ell=1}^t \frac{1}{1-\gamma_\ell} \mathbb{E}_{\mathbf{a} \sim p_{\ell}}\left[   \widehat{r}_\ell(\mathbf{a}) | \mathcal{F}_\ell    \right] \leq \sum_{\ell=1}^t \frac{r_\ell}{1-\gamma_\ell}  +  \Ocal\left( \sqrt{\frac{dt}{\rho} \ln \frac{t}{\delta} }  + \max_{\ell \leq t}\left( \frac{c(d)}{\rho \gamma_\ell} + 1\right) \ln \frac{t}{\delta} \right) ~.
\end{equation*}
\end{lemma}
\begin{proof}
Let $\bar{\mathbf{a}}_\ell =  \mathbb{E}_{\mathbf{a} \sim p_\ell } \left[ \mathbf{a} | \mathcal{F}_{\ell} \right]$ where the samples $\mathbf{a} \sim p_\ell$ are conditionally independent from $\mathbf{a}_\ell$. Observe that 
$$
\sum_{\ell=1}^t  \frac{1}{1-\gamma_\ell} \mathbb{E}_{\mathbf{a} \sim p_{\ell}}\left[   \widehat{r}_\ell(\mathbf{a})  | \mathcal{F}_{\ell}    \right] = \sum_{\ell=1}^t \frac{1}{1-\gamma_\ell}\mathbb{E}_{\mathbf{a} \sim p_{\ell}}\left[  \widehat{\boldsymbol{\omega}}_\ell^\top \mathbf{a}| \mathcal{F}_{\ell}    \right] = \sum_{\ell=1}^t \frac{1}{1-\gamma_\ell}  \widehat{\boldsymbol{\omega}}_\ell^\top \bar{\mathbf{a}}_\ell~. 
$$
The proof of this lemma follows closely the proof of Lemma 6 in~\cite{bartlett2008high}. Consider the martingale difference sequence $Y_\ell = \frac{\widehat{\boldsymbol{\omega}}_\ell^\top \bar{\mathbf{a}}_\ell - r_\ell}{1-\gamma_\ell}$ with respect to the filtration $\{ \mathcal{F}_{\ell}^{-}\}_{\ell=1}^\infty$, where we recall that $\widehat{\boldsymbol{\omega}}_\ell =  b_\ell \frac{r_\ell \left( \boldsymbol{\Sigma}_\ell \right)^{-1} \mathbf{a}_\ell}{\rho}$. The process $\{Y_\ell\}_{\ell = 1}^\infty$ is a martingale difference sequence w.r.t. the filtration $\{ \mathcal{F}_{\ell}^-\}_{\ell=1}^\infty$ since $\mathbb{E}\left[ \widehat{\boldsymbol{\omega}}_\ell^\top \bar{\mathbf{a}}_\ell | \mathcal{F}_\ell^{-}  \right] = \boldsymbol{\omega}_\ell^\top \bar{\mathbf{a}}_\ell =  \mathbb{E}[r_\ell | \mathcal{F}_\ell^-]$, and therefore $\mathbb{E}\left[ \widehat{\boldsymbol{\omega}}_\ell^\top \bar{\mathbf{a}}_\ell - r_\ell| \mathcal{F}_\ell^{-}  \right]  = 0$.

The conditional variance of $Y_\ell$ can be bounded as follows:

\begin{align*}
    \mathrm{Var}\left[ Y_\ell  | \mathcal{F}_{\ell}^- \right]  &=  \mathbb{E} \left[ \left(Y_\ell \right)^2 | \mathcal{F}_{\ell}^-  \right]\\
    &= \frac{\mathbb{E} \left[  \left(  \widehat{\boldsymbol{\omega}}_\ell^\top \bar{\mathbf{a}}_\ell - r_\ell   \right)^2   | \mathcal{F}_{\ell}^-\right] }{(1-\gamma_\ell)^2} \\
    &\stackrel{(i)}{\leq} 4 \mathbb{E} \left[ \left(   \widehat{\boldsymbol{\omega}}_\ell^\top \bar{\mathbf{a}}_\ell  \right)^2    \Big| \mathcal{F}_{\ell}^- \right] \\
    &\stackrel{(ii)}{\leq} \frac{4 \bar{\mathbf{a}}_\ell^\top \Sigma_\ell^{-1}  \bar{\mathbf{a} }_\ell }{\rho} \\
    &\stackrel{(iii)}{\leq} \frac{4 \mathbb{E}_{\mathbf{a} \sim p_{\ell} }\left[ \mathbf{a}^\top \Sigma_{\ell}^{-1} \mathbf{a}  | \mathcal{F}_{\ell-1} \right] }{\rho} \\
    &\stackrel{(iv)}{=} \frac{4d}{\rho}~,
\end{align*}
where:
$(i)$ holds because $\mathbb{E}\left[ \left( \widehat{\omega}_\ell^\top \bar{\mathbf{a}}_\ell - r_\ell \right)^2 | \mathcal{F}_\ell^{-} \right] \leq \mathbb{E}\left[ \left( \widehat{\omega}_\ell^\top \bar{\mathbf{a}}_\ell  \right)^2 | \mathcal{F}_\ell^{-} \right]$ and $\frac{1}{1-\gamma_\ell} \leq 2$; 
$(ii)$ is a consequence of Item 4 of Lemma~\ref{lemma::useful_properties} (treating $\bar{\mathbf{a}}_\ell$ as a fixed vector); 
$(iii)$ holds by Jensen's inequality; 
$(iv)$ holds by Item 3 of Lemma~\ref{lemma::useful_properties}. Thus, $\mathrm{Var}\left[ Y_\ell  | \mathcal{F}_\ell^{-} \right] \leq \frac{4d}{\rho}$

As a consequence, $\sum_{\ell=1}^t \mathrm{Var}\left[ Y_\ell  | \mathcal{F}_\ell^{-} \right] \leq \frac{4td}{\rho}$. Furthermore, $| Y_\ell| \leq \frac{2c(d)}{\rho \gamma_\ell} + 2$  by Item 1 in Lemma~\ref{lemma::useful_properties} and because $\frac{1}{1-\gamma_\ell} \leq 2$. 

We are in a position to apply Lemma~\ref{lemma::super_simplified_freedman_norm} (setting therein $V_t = \frac{4td}{\rho}$ and $B_t = \frac{2c(d)}{\rho \gamma_t} + 2$ ) to the martingale differences sequence $\{ Y_\ell\}_{\ell=1}^\infty$. Rearranging terms this gives the claimed bound.
\end{proof}

Denote by $\mathcal{E}_{\mathbf{II}}$ the event where the bound of Lemma~\ref{lemma::bounding_event_II} holds. By the previous result we have $\mathbb{P}\left( \mathcal{E}_{\mathbf{II}} \right) \geq 1-\delta$. Lemma~\ref{lemma::bounding_event_II} along with $|r_\ell| \leq 1$ for all $\ell$ (see Assumption~\ref{ass:boundedness}) together imply the following. 

\begin{corollary}\label{corollary::bounding_term_II}
If $\mathcal{E}_{\mathbf{II}}$ holds then
\begin{align*}
 \sum_{\ell=1}^t \frac{1}{1-\gamma_\ell} \mathbb{E}_{\mathbf{a} \sim p_{\ell}}\left[   \widehat{r}_\ell(\mathbf{a})   | \mathcal{F}_\ell  \right] \leq \sum_{\ell=1}^t r_\ell + 2\gamma_\ell  +  \Ocal\left( \sqrt{\frac{dt}{\rho} \ln \frac{t}{\delta}}  + \max_{\ell \leq t} \left( \frac{c(d)}{\rho \gamma_\ell} + 1\right) \ln \frac{t}{\delta} \right)~.
\end{align*}
\end{corollary}
\begin{proof}
Since $\frac{1}{1-\gamma_\ell} - 1 = \frac{\gamma_\ell}{1-\gamma_\ell}$ we immediately see that
\begin{align*}
    \sum_{\ell=1}^t \frac{r_\ell}{1-\gamma_\ell} = \sum_{\ell=1}^t r_\ell + \frac{\gamma_\ell r_\ell}{1-\gamma_\ell} 
\end{align*}

And from $|r_\ell| \leq 1$ and $\frac{1}{1-\gamma_\ell} \leq 2$,
\begin{equation*}
      \sum_{\ell=1}^t \frac{r_\ell}{1-\gamma_\ell} = \sum_{\ell=1}^t r_\ell + 2\gamma_\ell~,
\end{equation*}
thereby concluding the proof.
\end{proof}

Finally, (\ref{equation::upper_bounding_II}) and Corollary~\ref{corollary::bounding_term_II} imply that, in case $\mathcal{E}_{\mathbf{II}}$ holds,
\begin{align*}
\mathbf{II} 
\leq 
\sum_{\ell=1}^t r_\ell + 2\gamma_\ell  +  \Ocal\left(\sqrt{\frac{dt}{\rho}  \ln \frac{t}{\delta}}  + \max_{\ell \leq t} \left( \frac{c(d)}{\rho \gamma_\ell} + 1\right) \ln \frac{t}{\delta} + \sqrt{\frac{d t}{\rho}\ln\left(\frac{t}{\delta'}\right)}\right)~.
\end{align*}

\paragraph{Bounding term $\mathbf{III}$.}

By definition of $\widetilde{r}_\ell(\mathbf{a})$, the fact that $\frac{1}{1-\gamma_\ell} \leq 2$, along with the inequality $(a+b)^2 \leq 2a^2+2b^2$, we can write
\begin{align*}
    \sum_{\ell=1}^t \frac{\eta_\ell}{1-\gamma_\ell} \mathbb{E}_{\mathbf{a} \sim p_\ell}\left[ \widetilde{r}^2_\ell(\mathbf{a}) | \mathcal{F}_\ell \right]  &\leq \sum_{\ell=1}^t  4\eta_\ell   \mathbb{E}_{\mathbf{a} \sim p_\ell} \left[ \left(\widehat{r}_\ell(\mathbf{a})\right)^2 + \frac{4\ln\left(\frac{12\ell^2}{\delta'}\right)}{\rho d\ell } \left(\mathbf{a}^\top \Sigma_\ell^{-1} \mathbf{a}\right)^2 \Big| \mathcal{F}_\ell \right]  \\ 
    &= \underbrace{\sum_{\ell=1}^t  4\eta_\ell   \mathbb{E}_{\mathbf{a} \sim p_\ell} \left[ \left(\widehat{r}_\ell(\mathbf{a})\right)^2 \Big| \mathcal{F}_\ell\right] }_{\mathbf{A}} \\
    &\quad \underbrace{ +\Ocal\left(\sum_{\ell=1}^t \eta_\ell   \mathbb{E}_{\mathbf{a} \sim p_\ell} \left[  \frac{\left(\mathbf{a}^\top \Sigma_\ell^{-1} \mathbf{a}\right)^2 \ln\left(\frac{\ell}{\delta'}\right)}{\rho d\ell }  \Big| \mathcal{F}_\ell \right]\right) }_{\mathbf{B}}~.
 \end{align*}

We proceed to upper bound $\mathbf{A}$ and $\mathbf{B}$ separately. Let us start from term $\mathbf{B}$. We have
\begin{align}
 \mathbf{B}    
&\stackrel{(i)}{=} 
\Ocal\left( \sum_{\ell=1}^t \frac{\eta_\ell c(d) \ln\left(\frac{\ell}{\delta'}\right)  }{\rho \ell \gamma_\ell }   \mathbb{E}_{\mathbf{a} \sim p_\ell} \left[ \mathbf{a}^\top \Sigma_\ell^{-1} \mathbf{a} \right]\right)
\stackrel{(ii)}{=} 
\sum_{\ell=1}^t \Ocal\left(\frac{\eta_\ell c(d) d\ln\left(\frac{\ell}{\delta'}\right)  }{\rho \ell \gamma_\ell } \right)~, \label{equation::upper_bound_term_B}
\end{align}
where $(i)$ follows from Item 2. of Lemma~\ref{lemma::useful_properties}, and $(ii)$ follows from Item 3. of the same lemma. Notice that this upper bound holds deterministically. Let us now turn to handling term $\mathbf{A}$ now. We use a similar argument as Lemma 8 in~\cite{bartlett2008high}.

\begin{lemma}\label{lemma::bounding_term_A}
With probability at least $1-\delta$ simultaneously for all $t \in \mathbb{N}$,
\begin{align}\label{equation::upper_bound_term_A}
\mathbf{A} 
= \Ocal\left( \sum_{\ell=1}^t \frac{\eta_\ell d}{\rho} +   \sqrt{ \ln\left( \frac{t}{\delta}\right)\sum_{\ell=1}^t \frac{\eta_\ell^2 c(d) d}{ \gamma_\ell \rho^3}   } +
\max_{\ell \leq t } \frac{c(d)\eta_\ell}{\rho^2 \gamma_\ell}\ln \frac{t}{\delta}\right)~.
\end{align}
\end{lemma}

\begin{proof}
Recalling that $\widehat{\boldsymbol{\omega}}_\ell =  b_\ell \frac{r_\ell \left( \boldsymbol{\Sigma}_\ell \right)^{-1} \mathbf{a}_\ell}{\rho}$, we first observe that
\begin{align*}
       \mathbb{E}_{\mathbf{a} \sim p_\ell} \left[ \left(\widehat{r}_\ell(\mathbf{a})\right)^2 \Big| \mathcal{F}_\ell\right] &= \sum_{\mathbf{a} \in \mathcal{A}} p_\ell(\mathbf{a}) \widehat{\boldsymbol{\omega}}_\ell^\top \mathbf{a} \mathbf{a}^\top \widehat{\boldsymbol{\omega}}_\ell\\
     &= \widehat{\boldsymbol{\omega}}_\ell^\top \left(\sum_{\mathbf{a} \in \mathcal{A}} p_\ell(\mathbf{a}) \mathbf{a}\mathbf{a}^\top       \right) \widehat{\boldsymbol{\omega}}_\ell\\
     &= \frac{r_\ell^2 b_\ell^2}{\rho^2} \mathbf{a}_\ell^\top \Sigma_\ell^{-1} \Sigma_\ell \Sigma_{\ell}^{-1} \mathbf{a}_\ell  \\
     &= \frac{r_\ell^2 b_\ell }{\rho^2}\mathbf{a}_\ell^\top \Sigma_\ell^{-1} \Sigma_\ell \Sigma_{\ell}^{-1} \mathbf{a}_\ell\\
     &\leq \frac{ b_\ell \mathbf{a}_\ell^\top \Sigma_\ell^{-1} \mathbf{a}_\ell}{\rho^2}~.
\end{align*}
Summing over $\ell$ and multiplying by $4\eta_\ell$ yields
\begin{equation*}
  \mathbf{A} \leq \sum_{\ell=1}^t \frac{4\eta_\ell}{\rho^2} b_\ell \mathbf{a}_\ell^\top \Sigma_{\ell}^{-1} \mathbf{a}_\ell~.
\end{equation*}
Now, Item 2 of Lemma~\ref{lemma::useful_properties} implies the magnitude of each of the terms $\frac{4\eta_\ell}{\rho^2} b_\ell \mathbf{a}_\ell^\top \Sigma_{\ell}^{-1} \mathbf{a}_\ell$ is at most $4\frac{c(d)\eta_\ell}{\rho^2\gamma_\ell} $. Moreover, Item $3$ of Lemma~\ref{lemma::useful_properties} implies that for each term the conditional expectation $\mathbb{E}\left[ \frac{4\eta_\ell}{\rho^2} b_\ell \mathbf{a}_\ell^\top \Sigma_{\ell}^{-1} \mathbf{a}_\ell      \Big| \mathcal{F}_\ell^- \right]$ equals $\frac{4\eta_\ell d}{\rho}$.
As for the conditional variance, we can write
\begin{align*}
\mathrm{Var}\left[\frac{4\eta_\ell}{\rho^2} b_\ell \mathbf{a}_\ell^\top \Sigma_{\ell}^{-1} \mathbf{a}_\ell - \frac{4\eta_\ell d}{\rho}  \Big| \mathcal{F}_{\ell}^{-}\right] 
& \leq 
\frac{16\eta_\ell^2}{\rho^4}  \mathbb{E}\left[  b^2_\ell \left(\mathbf{a}_\ell^\top \Sigma_{\ell}^{-1} \mathbf{a}_\ell  \right)^2\Big| \mathcal{F}_\ell^-\right] \\
&\stackrel{(i)}{\leq} \frac{16\eta_\ell^2c(d)}{\gamma_\ell\rho^4}  \mathbb{E}\left[  b_\ell  \mathbf{a}_\ell^\top \Sigma_{\ell}^{-1} \mathbf{a}_\ell  | \mathcal{F}_\ell^{-}\right]\\
&\stackrel{(ii)}{=}
\frac{16\eta_\ell^2c(d)d}{\gamma_\ell\rho^3} 
\end{align*}
where $(i)$ follows from Item 2 in Lemma~\ref{lemma::useful_properties}, and $(ii)$ is from Item 3. 
An application of Lemma~\ref{lemma::super_simplified_freedman_norm}
concludes the proof
\end{proof}

We denote by $\mathcal{E}_{\mathbf{III}}$ the event where the bound of Lemma~\ref{lemma::bounding_term_A} holds. By the previous result, $\mathbb{P}\left( \mathcal{E}_{\mathbf{III}} \right) \geq 1-\delta$. Thus if $\mathcal{E}_{\mathbf{III}} $ holds, Equations~\ref{equation::upper_bound_term_B} and~\ref{equation::upper_bound_term_A} imply,
\begin{align*}
\mathbf{III} 
&=  
\sum_{\ell=1}^t \frac{\eta_\ell}{1-\gamma_\ell} \mathbb{E}_{\mathbf{a} \sim p_\ell}\left[ \widetilde{r}^2_\ell(\mathbf{a}) \right] \\
&= 
\Ocal\left(\sum_{\ell=1}^t \frac{\eta_\ell d}{\rho} +  \sqrt{ \ln\left( \frac{t}{\delta}\right)\sum_{\ell=1}^t \frac{\eta_\ell^2 c(d) d}{ \gamma_\ell \rho^3}   } + \max_{\ell \leq t} \frac{c(d)\eta_\ell}{\rho^2 \gamma_\ell}\ln \frac{t}{\delta} +
\sum_{\ell=1}^t \frac{\eta_\ell c(d) d\ln\left(\frac{\ell}{\delta'}\right)  }{\rho \ell \gamma_\ell } \right)~.
\end{align*}

\paragraph{Bounding Term $\mathbf{IV}$.}
Define $\mathrm{supp}(p_{E})  = \{ \mathbf{a} \in \mathcal{A}\,:\,  p_E(\mathbf{a}) > 0 \}$, and recall that
\begin{equation*}
     \mathbf{IV} = -\sum_{\ell=1}^{t} \frac{\gamma_\ell}{1-\gamma_\ell} \sum_{\mathbf{a} \in \mathcal{A}} p_{\mathrm{E}}(\mathbf{a}) \widetilde{r}_\ell(\mathbf{a})~.
\end{equation*}

Let $\mathbf{a}$ be any action in  $\mathrm{supp}(p_{E})$, and set in Lemma~\ref{lemma::weighted_lower_bound_fixed_action_high_prob}
$\alpha_{\ell} = \frac{\gamma_\ell}{1-\gamma_\ell}$. This implies that with probability at least $1-\delta$, 
\begin{align*}
\sum_{\ell=1}^t \frac{\gamma_\ell}{1-\gamma_\ell} \widetilde{r}_\ell(\mathbf{a}) &\geq \sum_{\ell=1}^t \frac{\gamma_\ell}{1-\gamma_\ell} \mathbf{a}^\top \omega_\ell - \Ocal\left(\sqrt{\frac{dt}{\rho} \ln \frac{t}{\delta'} } + \left( \frac{c(d)}{\rho} + 1  \right) \ln\frac{t}{\delta'}\right) \\
&\geq 
-2 \sum_{\ell=1}^t \gamma_\ell - \Ocal\left(\sqrt{\frac{dt}{\rho} \ln \frac{t}{\delta'} } + \left( \frac{c(d)}{\rho} + 1  \right) \ln\frac{t}{\delta'}\right)
\end{align*}
the last inequality following from $| \mathbf{a}^\top \boldsymbol{\omega}_\ell | \leq 1$ and $\frac{1}{1-\gamma_\ell} \leq 2$.

A simple union bound along with the fact that $\sum_{\mathbf{a} \in \mathrm{supp}(p_E)} p_E(\mathbf{a}) = 1$ implies that with probability at least $1-| \mathrm{supp}(p_E)|\delta$
\begin{align*}
\mathbf{IV} 
&\leq 2 \sum_{\ell=1}^t \gamma_\ell 
+ 
\Ocal\left(\sqrt{\frac{dt}{\rho} \ln \frac{t}{\delta'} } + \left( \frac{c(d)}{\rho} + 1  \right) \ln\frac{t}{\delta'}\right)~.
\end{align*}
Similar to before, we denote by $\mathcal{E}_{\mathbf{IV}}$ the event where this bound holds. By the previous result $\mathbb{P}\left(\mathcal{E}_{\mathbf{IV} }\right) \geq 1-| \mathrm{supp}(p_E)|\delta $.

\paragraph{Putting it all together.}
We plug the bounds so obtained on $\mathbf{I} -\mathbf{IV}$ back into Lemma \ref{l:basiclemma}, collect common terms, and overapproximate. We obtain that, when $\mathcal{E}_{\mathbf{I}}\cap \mathcal{E}_{\mathbf{II}}\cap \mathcal{E}_{\mathbf{III}} \cap \mathcal{E}_{\mathbf{IV}}$ holds,
\begin{align*}
\sum_{\ell = 1}^t \left(\bar{\mathbf{a}}^\top \boldsymbol{\omega}_\ell - r_\ell\right) 
&= \Ocal\Biggl(\frac{\ln(\mathcal{A})}{\eta_t} +  \sqrt{\frac{dt}{\rho}\ln \frac{t}{\delta'}} + \max_{\ell \leq t} \left(\frac{c(d)}{\rho \gamma_\ell} + 1\right) \ln \frac{t}{\delta'} +\sum_{\ell=1}^t \gamma_\ell + \sum_{\ell=1}^t \frac{\eta_\ell c(d) d\ln\left(\frac{\ell}{\delta'}\right)  }{\rho \ell \gamma_\ell }   \\
& \qquad\qquad + \sum_{\ell=1}^t \frac{\eta_\ell d}{\rho} + \sqrt{ \ln\left( \frac{t}{\delta}\right)\sum_{\ell=1}^t \frac{\eta_\ell^2 c(d) d}{ \gamma_\ell \rho^3}   } +  \max_{\ell \leq t} \frac{c(d)\eta_\ell}{\rho^2 \gamma_\ell}\ln \frac{t}{\delta}\Biggl)~,
\end{align*}
with $\delta' = \frac{\delta}{|\mathcal{A}|}$.

Now, recall the restriction on $\eta_t$ as in (\ref{equation::constraint_eta}). In order to fulfill this requirement, we set
$$
\eta_t = \Ocal\left(\frac{\rho \gamma_t}{c(d) + \frac{c(d)}{ \sqrt{dt}} \sqrt{\rho \ln |\mathcal{A}| \ln \frac{t}{\delta} }}\right)~.
$$
This gives
\begin{align*}
\sum_{\ell = 1}^t  \left(\bar{\mathbf{a}}^\top \boldsymbol{\omega}_\ell - r_\ell\right)
&= \Ocal\Biggl(
\frac{c(d)\,\ln|\mathcal{A}| }{\rho\gamma_t}\left( 1 + 2 \sqrt{ \frac{\rho\ln |\mathcal{A}| }{dt} \ln \frac{t}{\delta}  }\right) + \max_{\ell \leq t} \left(\frac{c(d)}{\rho \gamma_\ell} + 1\right) \ln |\mathcal{A}| \ln \frac{t}{\delta} \\
&\qquad\qquad +\sum_{\ell=1}^t \gamma_\ell + \sum_{\ell=1}^t \frac{ \ln |\mathcal{A}| \ln\left( \frac{\ell}{\delta}\right)}{\ell}+  \sum_{\ell=1}^t \frac{\gamma_\ell d}{c(d)} + \sqrt{\frac{dt \ln |\mathcal{A}|}{\rho}\ln \frac{12t^2}{\delta}} \\
&\qquad\qquad +\sqrt{\ln \frac{t}{\delta} \sum_{\ell=1}^t \frac{\gamma_\ell d}{c(d) \rho} } + \frac{1}{\rho}\ln \frac{t}{\delta}\Biggl)~.
\end{align*}
We now set $\gamma_\ell$ so as to satisfy (\ref{e:gammarestriction}):
\[
\gamma_\ell =  \min\left\{\sqrt{\frac{ c(d) \ln|\mathcal{A}|\ln \frac{\ell}{\delta}}{\rho\,\ell}},\frac{1}{2}\right\}~.
\]
Under the assumption that $c(d) \geq d$ (see below) this gets

\begin{align*}
\sum_{\ell = 1}^t  \left(\bar{\mathbf{a}}^\top \boldsymbol{\omega}_\ell - r_\ell\right) 
&=
\Ocal\Biggl(\sqrt{  \frac{ c(d) t \ln|\mathcal{A}|}{\rho}  \ln \frac{t}{\delta} } + \rho \ln|\mathcal{A}| \ln \frac{t}{\delta} \ln t + \sqrt{ \frac{dt \ln|\mathcal{A}|}{\rho} \ln \frac{t}{\delta} } \\
&\qquad\qquad + \left( \frac{\ln|\mathcal{A}|}{c(d)} \ln \frac{t}{\delta}  \right)^{1/4} t^{1/4}\left(\frac{1}{\rho} \right)^{3/4} + \frac{1}{\rho} \ln \frac{t}{\delta} \Biggl)~.
\end{align*}
Now, from \cite{bubeck2012regret} (Ch. 5 therein), it is known that with John's exploration, the smallest eigenvalue of $\boldsymbol{\Sigma}_t$ is at least $\frac{\gamma_t}{d}$, so that the function $c(d)$ in Lemma \ref{lemma::useful_properties} is $\leq d$. Moreover, the support of John's exploration distribution has size at most $d(d+1)/2 + 1 \leq 2d^2$. Combining with the last displayed equation, and taking a final union bound so as to make the events $\mathcal{E}_{\mathbf{I}}, \mathcal{E}_{\mathbf{II}}, \mathcal{E}_{\mathbf{III}},$
and $\mathcal{E}_{\mathbf{IV}}$ jointly hold concludes the proof of Theorem \ref{theorem::main_weighted_geomhedge_constantprob_simplified}.

\subsection{Exp4 algorithm for finite policy classes}\label{ssa:exp4}
Consider now the case of the Exp4  algorithm from \cite{a+03}. The algorithm operates with a finite set of policies $\Pi$. An anytime high probability regret guarantee for a biased version of Exp4 can be derived by following a similar pattern as in Section \ref{ssa:geomhedge}, but it can also be derived, e.g., by modifying the high-probability analysis for Exp3 contained in \cite{neu2015explore}. The proof is omitted.
\begin{corollary}\label{cor:exp4}
Let the complexity $R(\Pi)$ of the policy space $\Pi$ be defined as $R(\Pi) = \sqrt{|\Acal|\log |\Pi|}$. Then a version of the Exp4 algorithm from \cite{a+03} exists that is h-stable in that, for $t \rightarrow \infty$ and constant $\rho$ independent of $t$, its regret $\regret(t)$ satisfies
\begin{equation*}
\regret(t) = \Ocal\left(R(\Pi)\,\sqrt{\frac{t}{\rho}  \ln \frac{t}{\delta} }\right)~,
\end{equation*}
with probability at least $1-\delta$,
where the big-oh hides terms in $t$ which are lower order than $\sqrt{t\log t}$ as $t \rightarrow \infty$.
\end{corollary}
Regarding extendability and the ability to handle policy removals, this is fairly immediate for Exp4, and we omit the trivial details.

\section{Adversarial Regret Balancing and Elimination}\label{sa:adv}
In this section, we provide the proof of \textsf{Arbe}'s regret bound for adversarial environments. For convenience, we restate the main theorem here:
\arberegret*

The proof of this regret bound relies on the following regret bound of $\textsf{Arbe}$ in between restarts:
\begin{lemma}[Regret per Epoch of \textsf{Arbe}]
\label{lem:epoch_balancing_regret_bound}
Consider a run of \pref{alg:adversarial_epoch_balancing} with \textsf{Arbe}$(\delta, s, t_0)$  let $T \in \bbN \cup \{\infty\}$ be the round when the algorithm restarts ($T=\infty$ if there is no restart). Then the regret against $\Pi_M$ is bounded with probability at least $1 - \poly(M)\delta$ for all $t \in [T]$ simultaneously as 
\begin{align*}
    \regret_{\MM}([t_0+1, t], \Pi_M) & = \Ocal\left(\left(M R(\wt \Pi_{i_\star})  + \frac{R(\wt \Pi_{i_\star})^2}{R(\wt \Pi_{1})}  \sqrt{ i_\star}\right) \sqrt{(t - t_0) \ln \frac{t}{\delta}} + M \ln \frac{\ln t}{\delta}\right).
\end{align*}
Further, if $s \geq i^\star$, then the algorithm does not restart, i.e., $T = \infty$.
\end{lemma}
With this result, \pref{thm:arbe_main_regret} can be proven quickly:
\begin{proof}[Proof of Theorem~\ref{thm:arbe_main_regret}]
Denote by $t_0 = 0$ and $t_i$ the round of the $i$-th restart and $\infty$ if it does not exist for $i \in [i_\star]$. By \pref{lem:epoch_balancing_regret_bound}, there can be at most $i_\star - 1$ restarts and thus $t_{i_\star} = \infty$. The total regret of $\textsf{Arbe}(\delta, 1, 0)$ can be decomposed into the regret between two restarts
\begin{align*}
    \regret_{\MM}(t, \Pi_M) = \sum_{i = 1}^{i_\star} \regret([t_{i-1} + 1, \min\{t_i, t \}], \Pi_M)~.
\end{align*}
We can now plug in the bound from \pref{lem:epoch_balancing_regret_bound} for each term on the RHS which gives
\begin{align}
    \regret_{\MM} (t, \Pi_M) = 
    \sum_{s = 1}^{i_\star}
    \Ocal\left(\left(M R(\wt \Pi_{i_\star})  + \frac{R(\wt \Pi_{i_\star})^2}{R(\wt \Pi_{1})}  \sqrt{ i_\star}\right) \sqrt{\bar t_s \ln \frac{t}{\delta}} + M \ln \frac{\ln t}{\delta}\right)~.
\end{align}
where $\bar t_s = \max\{\min\{t_s, t\} - t_{s-1}, 0\}$. We can bound this further using Jensen's inequality and the fact that $\sum_{s=1}^{i_\star} \bar t_s = t$ as
\begin{align}
    \regret_{\MM}(t, \Pi_M) = 
    \Ocal\left(\left(\frac{R(\wt \Pi_{i_\star})}{R(\wt \Pi_{1})}\sqrt{ i_\star} + M\right) R(\wt \Pi_{i_\star})\sqrt{i_\star t \ln \frac{t}{\delta}} + M i_\star \ln \frac{\ln t}{\delta}\right)
\end{align}
When $t \geq i_\star$ the last term is dominated by the others and hence, the proof is complete.
\end{proof}

\subsection{Proof of \pref{lem:epoch_balancing_regret_bound}}

We first show that there are at most $i_\star$ restarts. This is because, due to extendability and h-stability of learners above $i_\star$, the elimination test can never trigger for them. The following lemma makes this argument formal:
\begin{lemma}\label{lem:no_istar_elimination}
With probability at least $1 - \poly(M)\delta$, the elimination test in \pref{eqn:algorithm_misspecification_test_condition} never triggers for $i,j$ with $i_\star \leq i < j \leq M$.
\end{lemma}

\begin{proof}
Let $i,j \in [i_\star, M] = \{i_\star, i_\star + 1, \dots, M\}$ and $t \in \bbN$ and denote by $t_0$ the round of the last restart before $t$.
By definition of $i_\star$ in \pref{sec:ms_setting}, base learner $i$ is h-stable and extendable. We consider the event where the regret bound of $i$ holds and where the statements in \pref{lem:basic_concentration_adv} hold. This happens for all $i \geq i_\star$ with probability at least $1 - \poly(M)\delta$.
Since $i \geq i_\star$ is h-stable and extendable, we have
\begin{align*}
    \max_{\pi \in \wt \Pi_i} \sum_{\ell = t_0 + 1}^t \bbE_{a \sim \pi}[r_\ell(a, x_\ell)] - r_\ell(a^i_\ell, x_\ell) \leq c R(\wt \Pi_{i}) \sqrt{\frac{t - t_0}{\rho_i} \ln \frac{t}{\delta}} 
\end{align*}
where $c$ is an absolute constant. Since the extended policy class $\wt \Pi_i$ includes an action $\wt a_j$ that always follows base learner $\alg[j]$, i.e., $r_\ell(\wt a_j, x_\ell) = r_\ell(a^j_\ell, x_\ell)$, we have in particular
\begin{align*}
    \sum_{\ell = t_0 + 1}^t \left[r_\ell(a^j_\ell, x_\ell) - r_\ell(a^i_\ell, x_\ell)\right] \leq c R(\wt \Pi_{i}) \sqrt{\frac{t - t_0}{\rho_i} \ln \frac{t}{\delta}}~.
\end{align*}
Using now \pref{eqn:conc2} from \pref{lem:basic_concentration_adv}, we have
\begin{align*}
    \tildecrew_j(t_0, t) - \tildecrew_i(t_0, t) 
    &\leq \conc[j](t_0, t) + \conc[i](t_0, t)
    +\sum_{\ell = t_0 + 1}^t \left[r_\ell(a^j_\ell, x_\ell) - r_\ell(a^i_\ell, x_\ell)\right]\\
    &\leq \conc[j](t_0, t) + \conc[i](t_0, t) + c R(\wt \Pi_{i}) \sqrt{\frac{t - t_0}{\rho_i} \ln \frac{t}{\delta}}
\end{align*}
and thus, the test in \pref{eqn:algorithm_misspecification_test_condition} does not trigger for $i$ and $j$ in round $t$.
\end{proof}

We are now ready to do the proof of \pref{lem:epoch_balancing_regret_bound}.
\begin{proof}[Proof of Lemma~\ref{lem:epoch_balancing_regret_bound}]
By \pref{lem:no_istar_elimination} there are at most $i_\star$ restarts and $s \leq i_\star$ at all times. 
The regret of \textsf{Arbe} in rounds $[t_0+1, t]$ against any policy $\pi' \in \Pi_M$ can be written as
\begin{align*}
     \sum_{\ell = t_0 + 1}^t \left[ \bbE_{a \sim \pi'} [r_\ell(a, x_\ell)] - r_\ell(a_\ell, x_\ell)\right]
     = \sum_{i=s}^M \sum_{\ell = t_0 + 1}^t  \left[ \rho_i \bbE_{a \sim \pi'} [r_\ell(a, x_\ell)] - \indicator{b_\ell = i} r_\ell(a_\ell, x_\ell)\right]~.
\end{align*}
For $i \geq i_\star$, we can bound the summand on the RHS directly using h-stability and extendability of $i$ as
\begin{align*}
    \sum_{\ell = t_0 + 1}^t &\left[ \rho_i \bbE_{a \sim \pi'} [r_\ell(a, x_\ell)] - \indicator{b_\ell = i} r_\ell(a_\ell, x_\ell)\right]\\
    &= 
        \rho_i \sum_{\ell = t_0 + 1}^t \left[ \bbE_{a \sim \pi'} [r_\ell(a, x_\ell)] - \frac{\indicator{b_\ell = i} r_\ell(a_\ell^i, x_\ell)}{\rho_i}\right]\\
    &\leq  
        \rho_i \sum_{\ell = t_0 + 1}^t \left[ \bbE_{a \sim \pi'} [r_\ell(a, x_\ell)] - r_\ell(a_\ell^i, x_\ell)\right] + \rho_i \conc[i](t_0, t) \tag{\pref{lem:basic_concentration_adv}}\\
        &\leq c R(\wt \Pi_i) \sqrt{(t - t_0)\rho_i  \ln \frac{t}{\delta}} + \rho_i \conc[i](t_0, t) 
        \tag{h-stability of $\alg[i]$}\\
        &= 
        \Ocal\left(R(\wt \Pi_i) \sqrt{(t - t_0)\rho_i  \ln \frac{t}{\delta}} +  \sqrt{(t - t_0) \rho_i \ln \frac{ \ln t}{\delta}}  +  \ln \frac{\ln t}{\delta}\right) \tag{definition of $\conc[i]$}\\
        &= 
        \Ocal\left(R(\wt \Pi_i) \sqrt{(t - t_0)\rho_i  \ln \frac{t}{\delta}}  +  \ln \frac{\ln t}{\delta}\right) \tag{since $R(\wt \Pi_i) \geq 1$}\\
        &= 
        \Ocal\left( \sqrt{\frac{t - t_0}{\sum_{j=s}^M R(\wt \Pi_j)^{-2}} \ln \frac{t}{\delta}}  +  \ln \frac{\ln t}{\delta}\right) \tag{definition of $\rho_i$}\\
        &= 
        \Ocal\left( \sqrt{\frac{t - t_0}{R(\wt \Pi_{i_\star})^{-2}} \ln \frac{t}{\delta}}  +  \ln \frac{\ln t}{\delta}\right) \tag{since $s \leq i_\star$}\\
        &= 
        \Ocal\left(R(\wt \Pi_{i_\star}) \sqrt{(t - t_0) \ln \frac{t}{\delta}}  +  \ln \frac{\ln t}{\delta}\right)~.\numberthis \label{eqn:hstable_regret_contribution}
\end{align*}
For $i < i_\star$, we cannot rely on h-stability and extendability and instead have to use the fact that the misspecification test in \pref{eqn:algorithm_misspecification_test_condition} did not trigger until the last round where there can be at most a regret of $1$. This allows us to bound
\begin{align*}
    &\sum_{\ell = t_0 + 1}^t  \left[ \rho_i \bbE_{a \sim \pi'} [r_\ell(a, x_\ell)] - \indicator{b_\ell = i} r_\ell(a_\ell, x_\ell)\right]\\
    &= 
        \rho_i \sum_{\ell = t_0 + 1}^t \left[ \bbE_{a \sim \pi'} [r_\ell(a, x_\ell)] - \frac{\indicator{b_\ell = i} r_\ell(a_\ell^i, x_\ell)}{\rho_i}\right]\\
    &\leq 
    \rho_i \sum_{\ell = t_0 + 1}^t \left[ \bbE_{a \sim \pi'} [r_\ell(a, x_\ell)] - \frac{\indicator{b_\ell = i_\star} r_\ell(a_\ell^{i_\star}, x_\ell)}{\rho_{i_\star}}\right]\\
    & \qquad+ \rho_i \conc[i](t_0, t) + \rho_i \conc[i_\star](t_0, t) + R(\wt \Pi_i) \sqrt{ \rho_i (t - t_0) \ln \frac{t}{\delta}} + 1 \tag{misspecification test failed}\\
        &\leq 
    \rho_i \sum_{\ell = t_0 + 1}^t \left[ \bbE_{a \sim \pi'} [r_\ell(a, x_\ell)] - r_\ell(a_\ell^{i_\star}, x_\ell)\right] \tag{\pref{lem:basic_concentration_adv}}\\
    & \qquad+ 2\rho_i \conc[i_\star](t_0, t) + R(\wt \Pi_i) \sqrt{ \rho_i (t - t_0) \ln \frac{t}{\delta}} + 1 \tag{$\conc[i_\star](t_0, t) \geq \conc[i](t_0, t)$} \\
   & \leq \rho_i c R(\wt \Pi_{i_\star}) \sqrt{\frac{t - t_0}{\rho_{i_\star}} \ln \frac{t}{\delta}} +  2\rho_{i} \conc[i_\star](t_0, t) + R(\wt \Pi_i) \sqrt{ \rho_i (t - t_0) \ln \frac{t}{\delta}} + 1 \tag{h-stability of $\alg[i_\star]$}\\
   & = \Ocal \left(  \rho_i  R(\wt \Pi_{i_\star}) \sqrt{\frac{t - t_0}{\rho_{i_\star}} \ln \frac{t}{\delta}} + \sqrt{\rho_i} \sqrt{\frac{\rho_i}{\rho_{i_\star}} (t - t_0) \ln \frac{t}{\delta}} +  R(\wt \Pi_i)\sqrt{\rho_i} \sqrt{  (t - t_0) \ln \frac{t}{\delta}} + \ln \frac{\ln t}{\delta} \right) \tag{definition of $\conc[i_\star]$}\\
   & = \Ocal \left(  \left( \frac{\rho_i  R(\wt \Pi_{i_\star})}{\sqrt{\rho_{i_\star}}} +  R(\wt \Pi_i)\sqrt{\rho_i}\right) \sqrt{  (t - t_0) \ln \frac{t}{\delta}} + \ln \frac{\ln t}{\delta} \right)
   \\
    & = \Ocal \left(  \left( \frac{R(\wt \Pi_{i_\star})^2}{R(\wt \Pi_{i})} +  R(\wt \Pi_i)\right) \sqrt{ \rho_i (t - t_0) \ln \frac{t}{\delta}} + \ln \frac{\ln t}{\delta} \right)
    \\
    &=  \Ocal \left(  \frac{R(\wt \Pi_{i_\star})^2}{R(\wt \Pi_{i})}  \sqrt{ \rho_i (t - t_0) \ln \frac{t}{\delta}} + \ln \frac{\ln t}{\delta} \right) \tag{$R(\wt \Pi_i) \leq R(\wt \Pi_{i_\star})$}
\end{align*}
Finally, combining both bounds yields
\begin{align*}
    \regret_{\MM}([t_0 + 1, t], \Pi_M)
    &= \Ocal\left(\left(\sum_{i = i_\star}^M R(\wt \Pi_{i_\star})  + \sum_{i = s}^{i_\star -1} \frac{R(\wt \Pi_{i_\star})^2}{R(\wt \Pi_{i})}  \sqrt{ \rho_i}\right) \sqrt{(t - t_0) \ln \frac{t}{\delta}} + M \ln \frac{\ln t}{\delta}\right)\\
    &= \Ocal\left(\left(M R(\wt \Pi_{i_\star})  + \frac{R(\wt \Pi_{i_\star})^2}{R(\wt \Pi_{1})}  \sqrt{ i_\star}\right) \sqrt{(t - t_0) \ln \frac{t}{\delta}} + M \ln \frac{\ln t}{\delta}\right)~.
\end{align*}
\end{proof}

\subsection{Concentration Bounds on Reward Sequences}
\label{sec:concentration_bounds}
\begin{lemma}\label{lem:basic_concentration_adv}
With probability at least $1 - \poly(M)\delta$, the following inequalities hold for all base learners $i \in [M]$ for all time steps $t \in \bbN$ where $\alg[i]$ was not eliminated yet
\begin{align}
    \left| \sum_{\ell = t_0 + 1}^t \left[\bbE_{a \sim \pi^i_\ell}\left[r_\ell(a, x_\ell) \right]
    - r_\ell(a^i_\ell, x_\ell) \right] \right| &= \Ocal\left( \sqrt{(t - t_0) \ln \frac{ \ln (t - t_0)}{\delta}} \right)
    \label{eqn:conc1}
    \\
    \left| \sum_{\ell = t_0 + 1}^t \left[\frac{\indicator{b_\ell = i} r_\ell(a_\ell, x_\ell) }{\rho_i}
    - r_\ell(a^i_\ell, x_\ell) \right] \right| &= \Ocal\left( \sqrt{\frac{t - t_0}{\rho_i} \ln \frac{ \ln (t- t_0)}{\delta}}  + \frac{1}{\rho_i} \ln \frac{\ln (t - t_0)}{\delta}\right)
    \label{eqn:conc2}
\end{align}
where $t_0$ is the time of the last restart of \pref{alg:adversarial_epoch_balancing} before $t$ and $\rho_i$ is the probability with which learner $i$ is chosen in round $t$. 
\end{lemma}
\begin{proof}
There can be at most $M$ restarts of \pref{alg:adversarial_epoch_balancing}. Thus, we can prove the concentration bounds for a single restart and a single base learner $i$ and obtain the statement for all restarts with a union bound over $M^2$. 

Consider first \pref{eqn:conc1} and let 
\begin{align*}
    \cF_\ell = \sigma\left(\{r_j, x_j, \pi_1^j, \{a^i_j\}_{i \in [M]}, b_j\}_{j \in [t_0 + 1, \ell - 1]} \cup  \{r_\ell, x_\ell, \pi_\ell^i, \{a^k_\ell\}_{k \in [M] \setminus\{i\}}\}\right)
\end{align*} be the sigma field of all previous reward functions, contexts and actions as well as the context and action in the current round. Further let $\tau$ be the stopping time w.r.t. $\{\cF_\ell\}$ of when the algorithm restarts, and denote 
\begin{align*}
    X_\ell = \indicator{\ell \geq \tau}\left(\bbE_{a \sim \pi^i_\ell}\left[r_\ell(a, x_\ell) \right]
    - r_\ell(a^i_\ell, x_\ell)\right)~.
\end{align*} 
The sequence $\{X_\ell\}_{\ell > t_0}$ is a martingale difference sequence w.r.t. $\{\cF_\ell\}_{\ell > t_0}$ and $X_\ell \in [-1, 1]$ almost surely for all $\ell$. Then by \pref{lem:hoeffding_lil} (setting $m = 1$ and $a_t = -1, b_t = 1$) with probability at least $1 - \delta$, we have for all $t > t_0$
\begin{align*}
   \sum_{\ell = t_0 + 1}^t X_\ell 
   &\leq 1.44 \sqrt{(t - t_0) \left( 1.4 \ln \ln \left(4 \left(t - t_0 \right)\right) + \ln \frac{5.2}{\delta}\right)} = \Ocal\left( \sqrt{(t - t_0) \ln \frac{ \ln (t-t_0)}{\delta}} \right)
\end{align*}
We can apply the same argument to $-\sum_{\ell = t_0 + 1}^t X_\ell$ which proves \pref{eqn:conc1}.

Consider now \pref{eqn:conc2} and let 
\begin{align*}
    \cF_\ell = \sigma\left(\{r_j, x_j, \pi_1^j, \{a^i_j\}_{i \in [M]}, b_j\}_{j \in [t_0 + 1, \ell - 1]} \cup  \{r_\ell, x_\ell, \pi_\ell^i, \{a^k_\ell\}_{k \in [M]}\}\right).
\end{align*}
Again, let $\tau$ be the stopping time w.r.t. $\{\cF_\ell\}_{\ell}$ of when the algorithm restarts and denote 
\begin{align*}
    X_\ell = \indicator{\ell \geq \tau}\left[\frac{\indicator{b_\ell = i} r_\ell(a_\ell^i, x_\ell) }{\rho_i}
    - r_\ell(a^i_\ell, x_\ell) \right]
\end{align*} 
which is a martingale difference sequence w.r.t. $\{\cF_\ell\}_{\ell > t_0}$. We have $X_\ell \leq \frac{1}{\rho_i}$ almost surely and $\bbE_[X_\ell^2 | \cF_\ell] \leq \frac{1}{\rho_i}$. By \pref{lem:bernstein_lil} (with $m = 1 / \rho_i$), this implies that with probability at least $1 - \delta$ for all $t > t_0$
\begin{align*}
    \sum_{\ell = t_0 + 1}^t X_\ell  &\leq  \underset{=: \conc[i](t_0, t)}{\underbrace{1.44 \sqrt{\frac{t - t_0}{\rho_i} \left( 1.4 \ln \ln \left(4 (t-t_0)\right) + \ln \frac{5.2}{\delta}\right)} + \frac{0.41}{\rho_i}  \left( 1.4 \ln \ln \left(4 (t - t_0)\right) + \ln \frac{5.2}{\delta}\right)}}~.
\end{align*}
We define the RHS as the precise definition of $\conc[i](t_0, t)$ used in \pref{alg:adversarial_epoch_balancing}. Note that
\begin{align}
    \conc[i](t_0, t) =  \Ocal\left( \sqrt{\frac{t - t_0}{\rho_i} \ln \frac{ \ln (t - t_0)}{\delta}}  + \frac{1}{\rho_i} \ln \frac{\ln (t - t_0)}{\delta}\right)
\end{align}
are required. Finally, we can apply the exact same argument to $-\sum_{\ell = t_0 + 1}^t X_\ell $ which finishes the proof.
\end{proof}

\section{Adversarial Regret Balancing and Elimination with Best of Both Worlds Regret}\label{app:bestofboth}

In this section, we provide the proofs of the main regret bound for \textsf{Arbe-Gap} in \pref{thm:arbegap_main} and describe the \textsf{Arbe-GapExploit} subroutine in detail.
We restate the theorem her for convenience:
\mainthmmodselbbw*

We break the proof of this statement into several parts based on the phases of the algorithm. The first phase end when \textsf{Arbe-Gap} calls \textsf{Arbe-GapExploit} and the second phase are all rounds played by \textsf{Arbe-GapExploit}. Finally, in case \textsf{Arbe-GapExploit} was called but terminated at some point, we have a final phase where we simply execute \textsf{Arbe}. For the regret in this final phase, we can directly use the guarantees of \textsf{Arbe}. The behavior in the first two phases is analyzed below. The following lemma characterizes the regret and pseudo-regret in the first phase until \textsf{Arbe-GapExploit} is called. It also ensures that if the environment is stochastic, the inputs of \textsf{Arbe-GapExploit} are correct, i.e., $\wh \pi$ is the optimal policy and $\wh \Delta$ is an accurate estimate of its gap.

\begin{restatable}[Guarantee for First Phase]{lemma}{mainresultgapestimation}
\label{lem:phase1_guarantees_combined}
Consider a run of \pref{alg:gap_estimation_balanced} with inputs $t_0 = 0$, arbitrary policy policy $\wh \pi \in \Pi_M$ and $M$ base learners $\alg[1], \dots, \alg[M]$. Further, let $t_{\mathrm{gap}} \in \bbN \cup \{\infty\}$ be the round where the $\textsf{Arbe-GapExploit}$ subroutine is called. Then with probability at least $1 - \poly(M)\delta$, the following conditions hold for all rounds $t \in [2i_\star, t_{\mathrm{gap}}]$. The regret is bounded as 
\begin{align*}
     \regret_{\bbM}(t, \Pi_M) =
    \Ocal\left(\left(M   + \frac{R(\wt \Pi_{i_\star})}{R(\wt \Pi_{1})}  \sqrt{ i_\star}\right) R(\wt \Pi_{i_\star}) \sqrt{t (\ln(t) + i_\star) \ln \frac{t}{\delta}}\right).
\end{align*}
If $\bbB$ is stochastic and there is a unique policy $\pi_\star$ with gap $\Delta > 0 $, then
the gap estimator $\widehat{\Delta}$ and policy $\wh \pi$ passed onto $\textsf{Arbe-GapExploit}$ satisfy $ \frac{\Delta}{2} \leq \widehat{\Delta} \leq \Delta $ and $\wh \pi = \pi_\star$. Further, 
the pseudo-regret is bounded as
\begin{align*}
    \pseudoregret_{\bbM}(t, \Pi_M) &= \Ocal\left(
    \frac{ R(\wt \Pi_{i_\star})^2 R(\wt \Pi_M)^2}{R(\wt \Pi_1)^2}\frac{M^2 i_\star}{\Delta} \ln \left(\frac{M R(\wt \Pi_M)}{\Delta \delta }\right)
     \ln \frac{t}{\delta}\right).
\end{align*}
\end{restatable}

The proof of this result can be found in \pref{app:proofs_ballancing_bww}. 
To characterize the regret and pseudo-regret of the second phase, we use the following main properties of the \textsf{Arbe-GapExploit} routine in \pref{alg:exploitation_subroutine}. It ensures that the regret and pseudoregret are well controlled and that the routine never terminates if the environment was indeed stochastic with a gap.
\begin{restatable}[Guarantee for Second Phase]{lemma}{mainresultexploitation}
\label{lem:main_bound_exploitation_guarantee_combined}
Let $\alg$ be an h-stable learner with policy class $\Pi_{\alg}$. Then the regret of \pref{alg:exploitation_subroutine} against $\Pi_{\alg} \cup \{\wh \pi\}$ is bounded with probability at least $1 - \Ocal(\delta)$ for all rounds $t > t_0$ that the algorithm has not terminated yet as
\begin{align*}
        \regret_{\bbM}([t_0 + 1, t], \Pi_{\alg} \cup \{\wh \pi\}) = \Ocal\left( \frac{R(\Pi_{\alg})^2}{\wh \Delta} \left(\ln(t) \ln \frac{t}{\delta}+ \ln \frac{R(\Pi_{\alg})}{\wh \Delta \delta} \right) + \sqrt{(t - t_0) \ln(t) \ln \frac{t}{\delta} }\right).
\end{align*}
Further, if the environment is stochastic with an optimal policy $\pi_\star$ that has a gap $\Delta$ compared to the best policy in $\Pi_{\alg}$ and the inputs satisfy $\wh \pi = \pi_\star$ and $\wh \Delta \leq \Delta \leq 2 \wh \Delta$, then with probability at least $1 - \Ocal(\delta)$ the pseudo-regret of \pref{alg:exploitation_subroutine} is bounded in all rounds $t > t_0$ as
\begin{align*}
    \pseudoregret_{\bbM}([t_0+1, t], \{ \wh \pi \} \cup \Pi_{\alg}) = \Ocal\left( \frac{R(\Pi_{\alg})^2}{\Delta} \left(\ln(t) \ln \frac{t}{\delta} + \ln \frac{R(\Pi_{\alg})}{\Delta \delta}\right)\right)
\end{align*}
and the algorithm never terminates.
\end{restatable}
The proof of this statement is available in \pref{app:proofs_gapexploit}. While combining \pref{lem:phase1_guarantees_combined} and \pref{lem:main_bound_exploitation_guarantee_combined}  gives the desired pseudo-regret guarantee in \pref{thm:arbegap_main} above for stochastic environments fairly directly, the bound on the regret in \pref{thm:arbegap_main} requires more work. \textsf{Arbe-GapExploit} guarantees only guarantees that the regret is of order $\tilde \Ocal( R(\Pi_{M} \setminus \{ \wh \pi\})^2 / \wh \Delta + \sqrt{t})$ while we would like a bound that does not scale with $ R(\Pi_{M} \setminus \{ \wh \pi\}$ in our final result. To achieve that, we will use the following lemma which states that the length of the initial phase $t_{gap}$ has to be sufficiently large as a function of the gap estimate $\wh \Delta$. This will allow us to absorb the $R(\Pi_{M} \setminus \{ \wh \pi\})^2 / \wh \Delta$ term into the regret of the first phase.
\begin{lemma}\label{lem:gap_tgap_relationship}
Consider a run of \pref{alg:gap_estimation_balanced} with inputs $t_0 = 0, n=1$, arbitrary policy $\wh \pi \in \Pi_M$ and $M$ base learners $\alg[1], \dots, \alg[M]$ with $1 \leq R(\wt \Pi_1) \leq R( \wt \Pi_2) \leq \dots \leq R(\wt \Pi_M)$ where $\wt \Pi_i$ is the extended version of policy class $\Pi_i$ with $(M+1-i)$ additional actions. 
Let $t_{gap}$ be the round where  $\textsf{ArbeGap-Exploit}$ was called with gap estimate $\wh \Delta$. Then with probability at least $1 - \poly(M) \delta$
\begin{align*}
    \wh \Delta = \Omega\left( \frac{R(\wt \Pi_{M})^2}{R(\wt \Pi_{i_\star})} \sqrt{\frac{ \ln \frac{t_{gap}}{\delta}}{t_{gap}}} 
    \right).
\end{align*}

\end{lemma}
\begin{proof}
\textsf{ArbeGap} calls \textsf{ArbeGap-Exploit} as soon as $2W(t_0, t) \leq \wh \Delta_t$. Hence, when the call happened in round $t_{gap}$, we must have $2W(t_0, t_{gap}) \leq \wh \Delta$ or, plugging in the definition of $W(t_0, t_{gap})$ with an appropriate absolute constant $c$
\begin{align*}
\frac{cR(\wt \Pi_{M})}{\sqrt{\rho_M}} \sqrt{\frac{\ln \frac{n(k)}{\delta}}{k}} + \frac{c}{\rho_M}   \frac{\ln \frac{n\ln (k)}{\delta}}{k} \leq \wh \Delta,
\end{align*}
where $k = t_{gap} - t_0$. 
We can further lower-bound the LHS as
\begin{align*}
\frac{cR(\wt \Pi_{M})}{\sqrt{\rho_M}} \sqrt{\frac{\ln \frac{nk}{\delta}}{k}} + \frac{c}{\rho_M}   \frac{\ln \frac{n\ln (k)}{\delta}}{k} 
&\geq \frac{cR(\wt \Pi_{M})}{\sqrt{\rho_M}} \sqrt{\frac{\ln \frac{k}{\delta}}{k}} \\
&\geq c R(\wt \Pi_{M})^2 \sqrt{\sum_{i=s}^{M} R(\wt \Pi_{i})^{-2}}  \sqrt{\frac{\ln \frac{k}{\delta}}{k}} \\
&\geq c \frac{R(\wt \Pi_{M})^2}{R(\wt \Pi_{i_\star})^{2}}  \sqrt{\frac{\ln \frac{k}{\delta}}{k}} 
\end{align*}
where the last equation holds because $i_\star$ is never eliminated with high probability \pref{lem:arbe_gap_restarts}. Finally, since the function on the RHS is monotonically decreasing in $k$, we can further lower-bound this quantity by replacing $k$ with $t_{gap} \geq k$. Hence, we have $\wh \Delta \geq c \frac{R(\wt \Pi_{M})^2}{R(\wt \Pi_{i_\star})^{2}}  \sqrt{\frac{\ln \frac{t_{gap}}{\delta}}{t_{gap}}} $.
Reordering terms gives the desires statement.
\end{proof}

We are now ready to state the proof of \pref{thm:arbegap_main}:
\begin{proof}[Proof of Theorem~\ref{thm:arbegap_main}]
We first consider stochastic environments with a gap and apply \pref{lem:phase1_guarantees_combined}. We denote by $t_{gap}$ the round where \textsf{Arbe-Gap} calls \textsf{Arbe-GapExploit}. For all $t \leq t_{gap}$, the pseudo-regret is bounded as 
\begin{align*}
    \pseudoregret_{\bbM}(t, \Pi_M) &= \Ocal\left(
    \frac{ R(\wt \Pi_{i_\star})^2 R(\wt \Pi_M)^2}{R(\wt \Pi_1)^2}\frac{M^2 i_\star}{\Delta} \ln \left(\frac{M R(\wt \Pi_M)}{\Delta \delta }\right)
     \ln \frac{t}{\delta}\right).
\end{align*}
Further,  \textsf{Arbe-GapExploit} can only be called with $\Pi_{\alg} = \Pi_{M} \setminus \{ \pi_\star\}$, $\wh \pi = \pi_\star$ and $\wh \Delta $ that satisfies $\wh \Delta \leq \Delta \leq 2 \wh \Delta$. This allows us to apply \pref{lem:main_bound_exploitation_guarantee_combined} to bound the pseudo-regret of any round played by \textsf{Arbe-GapExploit} as
\begin{align*}
    \pseudoregret_{\bbM}([t_{gap}+1, t], \Pi_M) = \Ocal\left( \frac{R(\Pi_{M})^2}{\Delta} \left(\ln(t) \ln \frac{t}{\delta} + \ln \frac{R(\Pi_{M})}{\Delta \delta}\right)\right).
\end{align*}
It further tells us that, with high probability, \textsf{Arbe-GapExploit} will never return. Hence, we can bound the pseudo-regret of both phases to get a bound on the total pseudo-regret after any number of rounds
\begin{align*}
    \pseudoregret_{\bbM}(t, \Pi_M) &= 
    \Ocal\Bigg(
    \frac{ R(\wt \Pi_{i_\star})^2 R(\wt \Pi_M)^2}{R(\wt \Pi_1)^2}\frac{M^2 i_\star}{\Delta} \ln \left(\frac{M R(\wt \Pi_M)}{\Delta \delta }\right)
     \ln \frac{t_{gap}}{\delta} \\
     & \qquad + \frac{R(\Pi_{M})^2}{\Delta} \left(\ln(t) \ln \frac{t}{\delta} + \ln \frac{R(\Pi_{M})}{\Delta \delta}\right)\Bigg)\\
     &= 
    \Ocal\Bigg(\frac{R(\Pi_{M})^2}{\Delta} \ln(t) \ln \frac{t}{\delta} + 
    \frac{ R(\wt \Pi_{i_\star})^2 R(\wt \Pi_M)^2}{R(\wt \Pi_1)^2}\frac{M^2 i_\star}{\Delta} \ln^2 \left(\frac{M R(\wt \Pi_M)}{\Delta \delta }\right) \Bigg)
\end{align*}
where we upper-bounded $t_{gap}$ using a crude upper-bound $\Ocal\left( \frac{R(\wt \Pi_M)^6M^5}{\Delta^4 \delta^2}\right)$ of the bound in \pref{lem:arbe_gap_phase1_length} which gives that $\ln(t_{gap}) = \Ocal(\ln( MR(\wt \Pi_M) / \Delta \delta))$.

We now move on to the regret bound in any environment. Again, \pref{lem:phase1_guarantees_combined} gives us a regret bound that holds with high probability for any round $t \leq t_{gap}$ of 
\begin{align*}
     \regret_{\bbM}(t, \Pi_M) =
    \Ocal\left(\left(M   + \frac{R(\wt \Pi_{i_\star})}{R(\wt \Pi_{1})}  \sqrt{ i_\star}\right) R(\wt \Pi_{i_\star}) \sqrt{t (\ln(t) + i_\star) \ln \frac{t}{\delta}}\right).
\end{align*}
If $t$ falls into a round that is played by the routine \textsf{Arbe-Gap}, then \pref{lem:main_bound_exploitation_guarantee_combined} the regret after $t_{gap}$ and before $t$ is bounded as
\begin{align*}
        \regret_{\bbM}([t_{gap} + 1, t], \Pi_{M}) = \Ocal\left( \frac{R(\Pi_{M})^2}{\wh \Delta} \left(\ln(t) \ln \frac{t}{\delta}+ \ln \frac{R(\Pi_{M})}{\wh \Delta \delta} \right) + \sqrt{(t - t_{gap}) \ln(t) \ln \frac{t}{\delta} }\right).
\end{align*}
Using \pref{lem:gap_tgap_relationship}, we can further bound $\frac{1}{\wh \Delta} = \Ocal\left(
\frac{R(\wt \Pi_{i_\star})}{R( \wt \Pi_M)^2} \sqrt{\frac{t_{gap}}{\ln(1 / \delta)}}\right)$ and plugging this into the bound above gives
\begin{align*}
        &\regret_{\bbM}([t_{gap} + 1, t], \Pi_{M})\\
        &= \Ocal\left( R(\wt \Pi_{i_\star})\sqrt{t_{gap}} \left(\ln^{3/2}(t) \ln^{1/2} \frac{t}{\delta}+ \ln (t_{gap} R(\Pi_{M})) \right) + \sqrt{(t - t_{gap}) \ln(t) \ln \frac{t}{\delta} }\right)\\
        &= \Ocal\left( R(\wt \Pi_{i_\star})\sqrt{t}\ln^{3/2}(t) \ln^{1/2} \frac{t}{\delta} + \sqrt{t \ln(t) \ln \frac{t}{\delta} }\right).
\end{align*}
Here, we also used the fact that $t_{gap} = \Omega(R(\Pi_M))$ since the test in Line~\ref{lin:gap_test} of \pref{alg:gap_estimation_balanced} can only trigger when $W(t_0, t) \leq R(\Pi_M)^2$ which is only possible after at least $\Omega(R(\Pi_M))$ rounds. Finally, if $t$ falls into a round after \textsf{Arbe-GapExploit} returned (in round $t_{adv}$, then the regret since the return can be bounded using the regret bound of \textsf{Arbe} in \pref{thm:arbe_main_regret} as
\begin{align*}
\regret_{\bbM}([t_{adv} + 1, t], \Pi_{M}) =
    \Ocal\left(  \left(\frac{R(\widetilde{\Pi}_{i_\star})}{R(\widetilde{\Pi}_{1})}\postdeadline{\sqrt{i_\star}} + \postdeadline{M} \right) R(\widetilde{\Pi}_{i_\star})\sqrt{ i_\star  t \ln \frac{t}{\delta} } \right)~,
\end{align*}
Combining the bounds from all three possible phases gives the following bound that holds for all $t > M^2$ as
\begin{align*}
    \regret_{\bbM}(t, \Pi_M) &=
    \Ocal\Bigg(\left(M   + \frac{R(\wt \Pi_{i_\star})}{R(\wt \Pi_{1})}  \sqrt{ i_\star}\right) R(\wt \Pi_{i_\star}) \sqrt{t (\ln(t) + i_\star) \ln \frac{t}{\delta}}\\
    & \qquad
    +
     R(\wt \Pi_{i_\star})\sqrt{t} \ln^{3/2}(t) \ln^{1/2} \frac{t}{\delta}
    \Bigg)\\
    &=
    \Ocal\left(\left(\frac{R(\wt \Pi_{i_\star})}{R(\wt \Pi_{1})}  \sqrt{ i_\star} + M + \ln(t)\right) R(\wt \Pi_{i_\star}) \sqrt{t (\ln(t) + i_\star) \ln \frac{t}{\delta}}\right)
\end{align*}
\end{proof}

\subsection{Description of the Second Phase: Gap Exploitation}
\begin{algorithm2e}[t]
\textbf{Input:} current round $t_0$, learner $\alg$ with policy class $\Pi_{\alg}$, candidate policy $\widehat{\pi}$, gap estimate $\widehat \Delta$, failure probability $\delta$ \\ %
Initialize $k_0 =  \Theta\left(\frac{ R^2(\Pi_{\alg})
}{\wh \Delta^2 }
\ln \frac{R(\Pi_{\alg})}{\widehat{\Delta} \delta}   \right)$ 
\\
\For{epoch $e=0, 1, 2 \ldots $ } {
    Set next epoch length $k_{e+1} = 2 k_{e}$ and final round $t_{e+1} = t_e + k_e$\\
    Set learner probability $ \rho^e =  \Theta\left( \frac{ R^2(\Pi_{\alg})
}{k_e \widehat{\Delta}^2 }
\ln \frac{k_e}{\delta_e} \right)$ where $\delta_e = \frac{\delta}{(e+1)^2}$\\
    
    Restart $\alg$ with failure probability $\delta_e$\\
    \For{round $t = t_e + 1, t_e + 2, \dots, t_{e+1}$}{

    Set $\pi_{t}^1$ as the current policy of $\alg$ and $\pi_t^0 = \wh \pi$\\
    Sample $b_t \sim \operatorname{Bernoulli}(\rho^e)$\\
    Get context $x_t$ and compute $a_t^i \sim \pi_t^i(\cdot | x_t)$ for $i \in \{0, 1\}$\\
    Play action $a_t = a_t^{b_t}$ and receive reward $r_t(a_t, x_t)$\\
    Update learner $\alg$ with reward $\frac{b_t r_t(a_t, x_t)}{\rho^e}$\\
    
    Set $V(t) = \Theta\left( R(\Pi_{\alg})\sqrt{\frac{\ln \frac{t - t_e}{\delta_e}}{\rho^e(t - t_e)} } + \frac{\ln \frac{\ln(t - t_e)}{\delta_e}}{t - t_e}\right)$ \\ %
    
    \If{$ \frac{\tildecrew_0(t_e + 1, t) - \tildecrew_1(t_e + 1, t)}{t - t_e}   <   \widehat{\Delta} - V(t)
    $}{
    \label{line::exploit_test_one}
       \textbf{return} \tcp{environment is adversarial}
    }
    
    \If{
    $\frac{\tildecrew_0(t_e + 1, t) - \tildecrew_1(t_e + 1, t)}{t - t_e} > 4 \wh \Delta + V(t)$
    }{
    \label{line::exploit_test_two}
         \textbf{return} \tcp{environment is adversarial}
    }
}
}

\caption{\textsf{Arbe-GapExploit}}
\label{alg:exploitation_subroutine}
\end{algorithm2e}

We present the \textsf{Arbe-GapExploit} algorithm of the Gap Exploitation phase as a general procedure that takes a focus policy $\widehat{\pi}$, a policy class $\Pi_{\mathbb{A}}$ and a gap estimate $\widehat \Delta$ and is tasked with testing the hypothesis `$\widehat \pi $ is the optimal policy, and has a gap of order $\Theta(\widehat \Delta)$', incurring in small regret while doing so. 

We use \textsf{Arbe-GapExploit} with input policy class $\Pi_{\mathbb{A}} = \Pi_M$. If it ever returns, the learner can conclude the environment is adversarial and thus, start playing \textsf{Arbe} with the policy classes $\Pi_s, \ldots, \Pi_M$ not yet eliminated by the misspecification tests during \textsf{Arbe-Gap}. We develop results for the more general case when the input algorithm and policy class equal $\alg$ and $\Pi_{\alg}$. In case the environment is stochastic with gap $\Delta$ and $\wh \Delta = \Theta(\Delta)$, we show \textsf{Arbe-GapExploit} has a pseudo regret of order $\Ocal\left( \frac{R(\Pi_{\alg})^2}{ \Delta} \left(\ln(t) \ln \frac{t}{\delta}+ \ln \frac{R(\Pi_{\alg})}{ \Delta \delta} \right) \right)$ at time $t$ (see \pref{lem:arbe_gap_exploit_regregretstochastic}). Similarly we show that \textsf{Arbe-GapExploit} has an adversarial regret rate of order $\Ocal\left( \frac{R(\Pi_{\alg})^2}{\wh \Delta} \left(\ln(t) \ln \frac{t}{\delta}+ \ln \frac{R(\Pi_{\alg})}{\wh \Delta \delta} \right) + \sqrt{(t - t_0) \ln(t) \ln \frac{t}{\delta} }\right)$ at time $t$ (see \pref{lem:arbe_gap_exploit_regregretadversarial}). The adversarial rate consists of a poly-logarithmic factor with an upfront multiplier of order $\frac{R^2(\Pi_\mathbb{A})}{\widehat \Delta}$ plus a factor scaling with $\sqrt{t - t_0}$. Since in our case $\wh \Delta = \Omega\left( \frac{R(\wt \Pi_{M})^2}{R(\wt \Pi_{i_\star})} \sqrt{\frac{M \ln \frac{t_{gap}}{\delta}}{t_{gap}}} \right)$ (see \pref{lem:gap_tgap_relationship}) and $\Pi_{\alg} = \Pi_M$, the adversarial regret has an upper bound of the form $\widetilde{\Ocal}\left( R(\wt \Pi_{i_\star}) \sqrt{t_\mathrm{gap}} + \sqrt{ t-t_\mathrm{gap} } \right)$ (where $\widetilde \Ocal$ hides polylogarithmic factors) thus satisfying a model selection guarantee.  
    
In \textsf{Arbe-GapExploit} we divide time into epochs indexed from $e = 0, 1, \ldots$ of length $k_e = k_0 \cdot 2^{e}$ where $k_0 =\Theta\left(  \frac{  R^2(\widetilde{\Pi}_{\alg}) \log \frac{R(\widetilde \Pi_{\alg})}{\widehat{\Delta} \delta}   }{\widehat{\Delta}^2 } \right)$. During epoch $e$, learner $\alg$ is sampled with probability $\rho^e = \Theta\left( \frac{ R^2(\widetilde \Pi_{\alg})\log\left(\frac{ k_e}{\delta_e}\right) }{k_e\widehat \Delta^2 } \right)$. We define $k_0$ and $\rho^e$ so that for all epochs $1-\rho^e \geq \frac{1}{2}$. Thus for all $t \in \{ t_e + 1, \cdots, t_{e+1}\}$, it holds that $\left| \tildecrew_{0}(t_e+1, t) - \sum_{\ell =t_e+1}^{t} \bbE_{a \sim \wh\pi}\left[ r_\ell(a, x_\ell)\right]\right| = \Ocal\left(\sqrt{(t-t_e) \ln \frac{t-t_e}{\delta_e}}\right)$. By the h-stability of $\alg$, and using a concentration argument, we prove that $\tildecrew_{1}(t_e+1, t)$ can be used to estimate $\max_{\pi \in \Pi_{\alg} \backslash \{ \widehat \pi \}} \sum_{\ell =t_e+1}^{t} \bbE_{a \sim \pi}\left[ r_\ell(a, x_\ell)\right] $ up to $V(t) = \Theta\left( R(\Pi_{\alg})\sqrt{\frac{\ln \frac{t - t_e}{\delta_e}}{\rho^e(t - t_e)} } + \frac{\ln \frac{\ln(t - t_e)}{\delta_e}}{t - t_e}\right) $ accuracy. 
 
When the environment is stochastic,  $\wh \pi$ is the optimal policy and $\Delta/2 \leq \wh \Delta \leq \Delta$ it follows that 
$$ 
(t-t_e)\wh \Delta\ \leq \sum_{\ell =t_e+1}^{t} \bbE_{a \sim \wh\pi, x \sim \Dcal}\left[ r(a, x)\right] -  \max_{\pi \in \Pi_{\alg} \backslash \{ \widehat \pi \}} \sum_{\ell =t_e+1}^{t} \bbE_{a \sim \pi, x \sim \Dcal}\left[ r(a, x)\right] \leq 4\wh \Delta~.
$$ 
Therefore when the condition in line~\ref{line::exploit_test_one} or~\ref{line::exploit_test_two} of \pref{alg:exploitation_subroutine} trigger, we would have found evidence that 
$$ 
\sum_{\ell =t_e+1}^{t} \bbE_{a \sim \wh\pi, x \sim \Dcal}\left[ r(a, x)\right] -  \max_{\pi \in \Pi_{\alg} \backslash \{ \widehat \pi \}} \sum_{\ell =t_e+1}^{t} \bbE_{a \sim \pi, x \sim \Dcal}\left[ r(a, x)\right] < (t-t_e)\wh \Delta
$$ 
or 
$$
\sum_{\ell =t_e+1}^{t} \bbE_{a \sim \wh\pi, x \sim \Dcal}\left[ r(a, x)\right] -  \max_{\pi \in \Pi_{\alg} \backslash \{ \widehat \pi \}} \sum_{\ell =t_e+1}^{t} \bbE_{a \sim \pi, x \sim \Dcal}\left[ r(a, x)\right] >  4(t-t_e)\wh \Delta
$$ 
thus indicating the environment cannot be stochastic.

The observations above imply that in case the environment is stochastic the tests of lines~\ref{line::exploit_test_one} and~\ref{line::exploit_test_two} in \pref{alg:exploitation_subroutine} do not trigger. Let us jump to the task of analyzing the regret of \textsf{Arbe-GapExploit} in stochastic environments. We will assume $t$ lies in epoch $e$. By the h-stability of $\alg$, the sum of its pseudo-rewards in a stochastic environment satisfies 
\begin{equation*}\sum_{\ell=t_e+1}^t \bbE_{a \sim \pi^1_\ell, x\sim \Dcal} \left[r(a,x)\right] + \underbrace{\Ocal\left(R(\Pi_{\alg} \backslash \{ \wh \pi\}) \sqrt{\frac{(t-t_e)}{\rho^e} \ln \frac{t-t_e}{\delta_e} }\right)}_{\mathbf{I}}\geq \max_{\pi \in \Pi_{\alg} \backslash \{ \wh \pi \}} \sum_{\ell=t_e+1}^t \bbE_{a \sim \pi, x\sim \Dcal} \left[ r(a,x)\right]~.
\end{equation*}
In this case, pseudo-regret is only incurred when $b_t = 1$. From $t_e+1$ to $t$ the variable $b_t$ equals $1$ an average of $\rho^e(t-t_e)$ times. Let us use the notation $\pi_\star' = \max_{\pi \in \Pi_{\alg} \backslash \{ \wh \pi\}} \bbE_{a \sim \pi, x \sim \Dcal }\left[ r(x,a)\right] $. The regret collected during these rounds can be upper bounded by \begin{small}\begin{equation*}
\Ocal\left(\underbrace{\Delta \cdot \rho^e (t-t_e)}_{\text{Regret of $\pi_\star'$ w.r.t. $\wh \pi$}} + \underbrace{\rho^e \times \mathbf{I}}_{\text{Regret of $\alg$ w.r.t. $\pi_\star'$ }} \right)= \Ocal\left(\Delta \rho^e (t-t_e) + R(\Pi_{\alg} \backslash \{ \wh \pi\})) \sqrt{\rho^e(t-t_e) \ln \frac{t-t_e}{\delta_e} }\right).\end{equation*}\end{small}
Substituting in the value of $\rho^e$ and using the fact that $\wh \Delta = \Theta(\Delta)$ when the environment is stochastic allow us to write
\begin{align*}
\Delta \cdot \rho^e (t-t_e) + \rho^e \cdot \mathbf{I} &= \Ocal\left( \frac{R^2(\Pi_{\alg})}{k_e \Delta} (t-t_e) \ln \left(\frac{k_e}{\delta_e}\right) + \frac{R^2(\Pi_{\alg})}{\Delta}\sqrt{\frac{t-t_e}{k_e} \ln \left( \frac{t-t_e}{\delta_e}\right) \ln\left( \frac{k_e}{\delta_e}\right) } \right) \\
&=\Ocal\left( \frac{R^2(\Pi_{\alg})}{ \Delta} \ln \left(\frac{k_e}{\delta_e}\right) \right)~.
\end{align*}
Summing over all epochs $e' \leq e$, and using $\sum_{e'=0}^e \ln \left( \frac{k_{e'}}{\delta_{e'}} \right) = \Ocal\left( \ln(t) \ln\left( \frac{t}{\delta}\right) + \ln\left(  \frac{R(\Pi_{\alg})}{\Delta \delta} \right) \right)$ we conclude that
\begin{align*}
    \pseudoregret_{\bbM}([t_0+1, t], \{ \wh \pi \} \cup \Pi_{\alg}) = \Ocal\left( \frac{R(\Pi_{\alg})^2}{\Delta} \left(\ln \frac{R(\Pi_{\alg})}{\Delta \delta} + \ln(t) \ln \frac{t}{\delta}\right) \right)
\end{align*}
and thus the proof sketch of \pref{lem:arbe_gap_exploit_regregretstochastic}.

We now bound the adversarial regret of \textsf{Arbe-GapExploit} during the timesteps before Lines~\ref{line::exploit_test_one} or~\ref{line::exploit_test_two} of \pref{alg:exploitation_subroutine} trigger. In this case, the h-stability of $\alg$ implies,
\begin{equation}\label{equation::supporting_inequality_explanation_exploit}
    \sum_{\ell=t_e+1}^t \bbE_{a \sim \pi_\ell^1} \left[r_\ell(a,x_\ell)\right] + \underbrace{\Ocal\left(R(\Pi_{\alg} \backslash \{ \wh \pi\}) \sqrt{\frac{(t-t_e)}{\rho^e} \ln \frac{t-t_e}{\delta_e} }\right)}_{\mathbf{II}(t,t_e)}\geq \max_{\pi \in \Pi_{\alg} \backslash \{ \wh \pi \}} \sum_{\ell=t_e+1}^t \bbE_{a \sim \pi} \left[ r_\ell(a,x_\ell)\right].
\end{equation}
While Lines~\ref{line::exploit_test_one} or~\ref{line::exploit_test_two} of \pref{alg:exploitation_subroutine} have not triggered, we can certify that with high probability 
\begin{equation}\label{equation::supporting_inequality_explanation_exploit_one}
\wh \Delta (t-t_e) -V_e(t)(t-t_e) \leq    \sum_{\ell=t_e+1}^t \bbE_{a \sim \wh \pi} \left[r_\ell(a,x_\ell)\right] - \sum_{\ell=t_e+1}^t \bbE_{a \sim \pi^1_\ell} \left[r_\ell(a,x_\ell)\right] \leq 4\wh \Delta (t-t_e) + V_e(t)(t-t_e).
\end{equation}
In the above, we used the notation $V_e(t)$ to denote the $V(t)$ of epoch $e$. Hence the regret collected during rounds $\{t_e+1, \ldots, t\}$ can be upper bounded by the sum of three terms
\begin{align*}
& \underbrace{ \max\left(\sum_{\ell=t_e+1}^t \bbE_{a \sim \wh \pi} \left[r_\ell(a,x_\ell)\right] - \sum_{\ell=t_e+1}^t \bbE_{a \sim \pi_\ell} \left[r_\ell(a,x_\ell)\right] , 0\right)}_{\mathbf{B}} \\
&+\underbrace{ \max\left( \max_{\pi \in \Pi_{\alg} \backslash \{ \wh \pi \}} \sum_{\ell=t_e+1}^t \bbE_{a \sim \pi} \left[ r_\ell(a,x_\ell)\right] - \sum_{\ell=t_e+1}^t \bbE_{a \sim \wh \pi} \left[r_\ell(a,x_\ell)\right] , 0\right)}_{\mathbf{C}}\\ &+\underbrace{\sum_{\ell=t_e+1}^t \bbE_{a \sim \wh \pi} \left[r_\ell(a,x_\ell)\right] - r_\ell(a_\ell, x_\ell)}_{\mathbf{D}}~.
\end{align*}
From $t_e+1$ to $t$ the variable $b_t$ equals $1$ an average of $\rho^e(t-t_e)$ times. Note that $V(t) (t-t_e) \approx \Ocal\left( \wh \Delta \sqrt{k_e (t-t_e)}  \right)$. We can bound $\mathbf{B}, \mathbf{C}, \mathbf{D}$ individually as follows:
\begin{itemize}
\item $\mathbf{B}$ is the Pseudo-Regret of $\alg$ w.r.t $\widehat{\pi}$, and can be upper bounded by $\Ocal\left(\wh\Delta \cdot \rho^e (t-t_e) + \wh \Delta \cdot \rho^e \sqrt{ (t-t_e) k_e}\right) = \Ocal\left( \wh \Delta \cdot \rho^e \sqrt{ (t-t_e) k_e}\right)$ as a consequence of multiplying the right hand side of Equation~\ref{equation::supporting_inequality_explanation_exploit_one} by $\rho^e$.
\item $\mathbf{C}$ is the Pseudo-Regret of $\wh \pi$ w.r.t. $\Pi_{\alg}\backslash \{ \wh\pi \}$, and can be upper bounded by $\mathbf{II}(t,t_e) - \widehat \Delta (t-t_e) + V_e(t)(t-t_e)$ as a consequence of combining Equations~\ref{equation::supporting_inequality_explanation_exploit} and the left hand side of~\ref{equation::supporting_inequality_explanation_exploit_one}. 
\item $\mathbf{D}$ is the difference between sample rewards vs. Pseudo-Rewards and can be bound by Hoeffding's inequality. 
\end{itemize}
Let $t_e' = \min(t, t_e)$. Summing the upper bound $\mathbf{C}$ over all epochs $e' \leq e$, we choose the multiplier $c$ in $\rho^e = \frac{ cR^2(\Pi_{\alg})
}{k_e \widehat{\Delta}^2 }
\ln \frac{k_e}{\delta_e}$ such that\footnote
{
Increasing the value of $c$ implies we have to set $k_0$ to be larger. This has the only effect of increasing the constant on the RHS of Eq.~\ref{equation::upper_bound_support_termII_sketch_exploit}. 
} 
\begin{align}
    \sum_{e' \leq e}\mathbf{II}(t_{e'}', t_{e'}) - \wh \Delta (t_{e'}'-t_{e'}) + V_{e'}(t_{e'}')(t_{e'}'-t_{e'}) &=    \sum_{e' \leq e} \Ocal\left( \wh \Delta \sqrt{\frac{k_e(t_{e'}'-t_{e'})}{c}} \right)  - \wh \Delta (t_{e'}'-t_{e'}) \notag\\  
    &\leq   \Ocal\left(\wh \Delta k_0\right). \label{equation::upper_bound_support_termII_sketch_exploit} 
\end{align}

Combining the bounds for $\mathbf{B}$ and $\mathbf{D}$,
\begin{equation*}
    \Ocal\left(\underbrace{\wh\Delta \cdot \rho^e (t-t_e)}_{\mathbf{B}}  + \underbrace{\sqrt{ (t-t_e) \ln \left( \frac{t-t_e}{\delta_e} \right)  } }_{\mathbf{D}}\right)= \Ocal\left(\frac{R^2(\Pi_{\alg})}{ \wh \Delta} \ln \left(\frac{k_e}{\delta_e}\right) + \sqrt{ (t-t_e) \ln \left( \frac{t-t_e}{\delta_e} \right)  }\right).
\end{equation*}

Summing over all epochs and using the upper bounds 
\begin{align*}
    \sum_{e'=0}^e \ln \left( \frac{k_{e'}}{\delta_{e'}} \right) &= \Ocal\left( \ln(t) \ln\left( \frac{t}{\delta}\right) + \ln\left(  \frac{R(\Pi_{\alg})}{\wh\Delta \delta} \right) \right)\\
    \sum_{e'=0}^e  \sqrt{ (t_{e'}'-t_{e'}) \ln \left( \frac{t'_{e'}-t_{e'}}{\delta_{e'}} \right)  } &= \Ocal\left( \sqrt{(t - t_0) \ln(t) \ln \frac{t}{\delta} }  \right)
\end{align*}
and the bound from Eq.~\ref{equation::upper_bound_support_termII_sketch_exploit} along with the observation $\wh \Delta k_0  = \Ocal\left( \frac{ R^2(\Pi_{\alg})
}{\wh \Delta }
\ln \frac{R(\Pi_{\alg})}{\widehat{\Delta} \delta}\right)$ allows us to conclude,
\begin{align*}
        \regret_{\bbM}([t_0 + 1, t], \Pi_{\alg} \cup \{\wh \pi\}) = \Ocal\left( \frac{R(\Pi_{\alg})^2}{\wh \Delta} \left(\ln(t) \ln \frac{t}{\delta}+ \ln \frac{R(\Pi_{\alg})}{\wh \Delta \delta} \right) + \sqrt{(t - t_0) \ln(t) \ln \frac{t}{\delta} }\right)~,
\end{align*}
and with it we finish the proof sketch of \pref{lem:arbe_gap_exploit_regregretadversarial}. Combining these results finalizes the proof of \pref{lem:main_bound_exploitation_guarantee_combined}.

\subsection{Analysis of the First Phase}\label{app:proofs_ballancing_bww}
We will use the notation $t_\mathrm{gap}$ to denote the (random) time when the \textsf{Arbe-Gap} estimation phase ends (see Algorithm~\ref{alg:gap_estimation_balanced}).

\mainresultgapestimation*

\subsubsection{Adversarial Guarantees}
We first start by bounding the number of restarts of the algorithm:
\begin{lemma}[Number of Restarts of \textsf{Arbe-Gap}]
\label{lem:arbe_gap_restarts}
Consider a run of \pref{alg:gap_estimation_balanced} with inputs $t_0 = 0$, arbitrary policy policy $\wh \pi \in \Pi_M$ and $M$ base learners $\alg[1], \dots, \alg[M]$. Then for any total number of rounds $t$, there are at most $\ln(t)$ restarts due to a change in candidate policy (Line~\ref{lin:policy_test} in \pref{alg:gap_estimation_balanced}) up to that round $t$. Further, with probability at least $1 - \poly(M)\delta$, there are at most $i_\star - 1$ restarts due an elimination of a base learner (Line~\ref{lin:arbegap2} in \pref{alg:gap_estimation_balanced}).
\end{lemma}
\begin{proof}
We first show the bound on the number of restarts due to changes in the candidate policy $\wh \pi$. Let $t_1, t_2, \dots$ be the rounds at which a restart is triggered in Line~\ref{lin:policy_test} of \pref{alg:gap_estimation_balanced} and $\wh \pi_1, \wh \pi_2, \dots$ be the candidate policies selected at those restarts. For each restart $i$, we know that $\wh \pi_i$ was selected in at least $\frac{3t_i}{4}$ of the first $t_i$ rounds and therefore also $t_{i+1}$ rounds. However, since the policy changed from $\wh \pi_i$ to $\wh \pi_{i+1}$ at round $t_{i+1}$, we also know that $\wh \pi_i$ can only be selected at most $\frac{t_{i+1}}{4}$ of the first $t_{i+1}$ rounds. Combining both bounds yields
\begin{align*}
    \frac{3t_i}{4} \leq \frac{t_{i+1}}{4}
\end{align*}
and thus $t_{i+1} \geq 3 t_i$ holds for all $i$. Since also $t_1 \geq 9$ by the condition in the algorithm,  up to round $t$, there can only be $\log_3(t)-1 \leq \ln(t) - 1 \leq \ln(t)$ restarts.

Finally, the number of restarts due to base learner elimination is bounded by $i_\star - 1$ since this condition can never trigger for $i \geq i_\star$ by \pref{lem:no_istar_elimination} (which also holds for $\textsf{Arbe-Gap}$).
\end{proof}

The following lemma now bounds the regret within each restart:

\begin{lemma}\label{lem:gap_estimation_single_epoch}
Consider a run of \pref{alg:gap_estimation_balanced} with \textsf{Arbe-Gap}$(\delta, s, t_0, \widehat{\pi})$ where $s \leq i_\star$ and let $T \in \bbN \cup \{\infty\}$ be the round when the algorithm restarts or calls \textsf{Arbe-GapExploit} ($T=\infty$ if there is no restart or transition to the second phase). Then the regret against $\Pi_M$ is bounded with probability at least $1 - \poly(M)\delta$ for all $t \in [t_0 + 1, T] = \{t_0 + 1, t_0 + 2, \dots, T\}$ simultaneously as 
\begin{align*}
    \regret_{\bbM}([t_0+1, t], \Pi_M) & = \Ocal\left(\left(M R(\wt \Pi_{i_\star})  + \frac{R(\wt \Pi_{i_\star})^2}{R(\wt \Pi_{1})}  \sqrt{ i_\star}\right) \sqrt{(t - t_0) \ln \frac{t}{\delta}} + M \ln \frac{\ln t}{\delta}\right).
\end{align*}
\end{lemma}

\begin{proof}
The regret of \textsf{Arbe-Gap} in rounds $[t_0+1, t]$ against any policy $\pi' \in \Pi_M$ can be written as
\begin{align*}
     \sum_{\ell = t_0 + 1}^t \left[ \bbE_{a \sim \pi'} [r_\ell(a, x_\ell)] - r_\ell(a_\ell, x_\ell)\right]
     = \sum_{i=s}^{M+1} \sum_{\ell = t_0 + 1}^t  \left[ \rho_i \bbE_{a \sim \pi'} [r_\ell(a, x_\ell)] - \indicator{b_\ell = i} r_\ell(a_\ell, x_\ell)\right]~.
\end{align*}
For $i \in [s, M]$, we can follow the analysis of \textsf{Arbe} and apply the arguments in the proof of \pref{lem:epoch_balancing_regret_bound} verbatim. This yields with probability at least $ 1 - \poly(M)\delta$
\begin{align*}
    \sum_{i=s}^{M} \sum_{\ell = t_0 + 1}^t  \left[ \rho_i \bbE_{a \sim \pi'} [r_\ell(a, x_\ell)] - \indicator{b_\ell = i} r_\ell(a_\ell, x_\ell)\right] \qquad \qquad \qquad\qquad\qquad\\
    \qquad =\Ocal\left(\left(M R(\wt \Pi_{i_\star})  + \frac{R(\wt \Pi_{i_\star})^2}{R(\wt \Pi_{1})}  \sqrt{ i_\star}\right) \sqrt{(t - t_0) \ln \frac{t}{\delta}} + M \ln \frac{\ln t}{\delta}\right).
    \numberthis\label{eqn:1Mregret}
\end{align*}
It only remains to bound the regret contribution of the special base learner $\alg[M+1]$, which does not exist in \textsf{Arbe}. To do so, we will use the fact that the gap test in Line~\ref{lin:gap_test} can only trigger in the last round before a restart happens. This allows us to relate the regret of $\alg[M+1]$ to that of $\alg[M]$. For the regret of $\alg[M]$, we again use the arguments in the proof of \pref{lem:epoch_balancing_regret_bound} (\pref{eqn:hstable_regret_contribution} specifically) verbatim to show
\begin{align}
    \sum_{\ell = t_0 + 1}^t  \left[ \rho_M \bbE_{a \sim \pi'} [r_\ell(a, x_\ell)] - \indicator{b_\ell = M} r_\ell(a_\ell, x_\ell)\right]
     = \Ocal\left(R(\wt \Pi_{i_\star}) \sqrt{(t - t_0) \ln \frac{t}{\delta}}  +  \ln \frac{\ln t}{\delta}\right).
     \label{eqn:Mregret}
\end{align}

The regret contribution of $\alg[M+1]$ can now be bounded as
\begin{align*}
   & \sum_{\ell = t_0 + 1}^t  \left[ \rho_{M+1} \bbE_{a \sim \pi'} [r_\ell(a, x_\ell)] - \indicator{b_\ell = M+1} r_\ell(a_\ell, x_\ell)\right] 
   \\
    &=   \rho_{M+1}\left( \left[ \sum_{\ell = t_0 + 1}^t   \bbE_{a \sim \pi'} [r_\ell(a, x_\ell)] \right] - \tildecrew_{M+1}(t_0+1, t)\right) \tag{definition of $\tildecrew_{M+1}$}
    \\
    &= 
    \rho_{M+1}\left( \left[ \sum_{\ell = t_0 + 1}^t   \bbE_{a \sim \pi'} [r_\ell(a, x_\ell)] \right] - \tildecrew_{M}(t_0+1, t)\right) \\
    &\quad + 
    \rho_{M+1}\left(\tildecrew_{M}(t_0+1, t) - \tildecrew_{M+1}(t_0+1, t)\right)
    \\%
        &\leq  
    \rho_{M+1}\left( \left[ \sum_{\ell = t_0 + 1}^t   \bbE_{a \sim \pi'} [r_\ell(a, x_\ell)] \right] - \tildecrew_{M}(t_0+1, t)\right) \\ &\quad +  
    7\rho_{M+1} W(t_0, t) (t - t_0) + 1 \tag{gap test not triggered at $t-1$}
    \\%
            &\leq  
    \frac{\rho_{M+1}}{\rho_M} \Ocal\left(R(\wt \Pi_{i_\star}) \sqrt{(t - t_0) \ln \frac{t}{\delta}}  +  \ln \frac{\ln t}{\delta}\right)  +  
    7\rho_{M+1} W(t_0, t) (t - t_0) + 1 \tag{\pref{eqn:Mregret}}
    \\%
                &\leq  
    \Ocal\left(R(\wt \Pi_{i_\star}) \sqrt{(t - t_0) \ln \frac{t}{\delta}}  +  \ln \frac{\ln t}{\delta}\right)  +  
    7\rho_{M} W(t_0, t) (t - t_0)  \tag{$\rho_M \geq \rho_{M+1}$}
    \\%
                    &=  
    \Ocal\left(R(\wt \Pi_{i_\star}) \sqrt{(t - t_0) \ln \frac{t}{\delta}}  +  \ln \frac{\ln t}{\delta}\right) 
     + \Ocal\left(  R( \widetilde \Pi_{M} )\sqrt{\rho_M (t-t_0)\ln \frac{t}{\delta}}  \right)
     \tag{definition of $W$}    
     \\%
                    &= 
    \Ocal\left(R(\wt \Pi_{i_\star}) \sqrt{(t - t_0) \ln \frac{t}{\delta}}  +  \ln \frac{\ln t}{\delta} +
    \sqrt{\frac{t-t_0}{\sum_{i=s}^{M+1} R(\wt \Pi_i)^{-2}}\ln \frac{t}{\delta}}  \right)
    \tag{definition of $\rho_M$}\\ 
                        &=  
    \Ocal\left(R(\wt \Pi_{i_\star}) \sqrt{(t - t_0) \ln \frac{t}{\delta}}  +  \ln \frac{\ln t}{\delta} \right)
    \tag{$s \leq i_\star$}
\end{align*}
Note that $\rho_M \geq \rho_{M+1}$ holds without loss of generality since $\wt \Pi_{M+1}$ contains 2 fewer policies than $\wt \Pi_M$. Finally, the desired statement follows by combining the previous display with \pref{eqn:1Mregret}.
\end{proof}

Equipped with the previous two lemmas, we can now prove the regret bound of the first phase for adversarial environments:
\begin{lemma}\label{lem:arbegap_adversarial_regret}
Consider a run of \pref{alg:gap_estimation_balanced} with inputs $t_0 = 0$, arbitrary policy policy $\wh \pi \in \Pi_M$ and $M$ base learners $\alg[1], \dots, \alg[M]$. Further, let $t_{\mathrm{gap}} \in \bbN \cup \{\infty\}$ be the round where the $\textsf{Arbe-GapExploit}$ subroutine is called. Then with probability at least $1 - \poly(M)\delta$, the following conditions hold for all rounds $t \in [2 i_\star, t_{\mathrm{gap}}]$. The regret is bounded as 
\begin{align*}
     \regret_{\bbM}(t, \Pi_M) = 
    \Ocal\left(\left(M R(\wt \Pi_{i_\star})  + \frac{R(\wt \Pi_{i_\star})^2}{R(\wt \Pi_{1})}  \sqrt{ i_\star}\right) \sqrt{t (\ln(t) + i_\star) \ln \frac{t}{\delta}}\right).
\end{align*}
\end{lemma}
\begin{proof}
Let $\tau_0, \tau_1, \dots$ be the rounds where $\textsf{Arbe-Gap}$ restarts or eventually calls $\textsf{Arbe-GapExploit}$. By convention, we set $\tau_0 = 0$ and $\tau_i = \infty$ when there are less than $i$ total calls to $\textsf{Arbe-Gap}$. 
By \pref{lem:arbe_gap_restarts}, there are at most $\ln(t) + i_\star$ calls of $\textsf{Arbe-Gap}$ up to round $t$ with probability at least $ 1 - \poly(M) \delta$ for all $t \in \bbN$ jointly.
Further, by \pref{lem:gap_estimation_single_epoch}, the regret in each of these calls is bounded with probability $1 - \poly(M) \delta$ as well. If we were to apply a naive union bound over all $\ln(t) + i_\star$, then our failure probability would increase at a rate of $\ln(t)$. However, we can easily avoid this by choosing the absolute constants in the definition of $\conc[i]$ appropriately. A factor of 3 larger is sufficient. This ensures that each of these terms is effectively at least as large as if we had invoked them with $\frac{\delta}{n^2}$ instead of $\delta$ in the $n$-th restart of $\textsf{Arbe-Gap}$. We now illustrate why this is true. Let $c'$ be the absolute constant such that 
\begin{align*}
    \conc[i](t_0, t) = c' \sqrt{\frac{t - t_0}{\rho_i} \ln \frac{\ln t}{\delta}} + \frac{c'}{\rho_i}\ln \frac{\ln t}{\delta}
\end{align*} satisfies the necessary concentration bounds (see \pref{sec:concentration_bounds}) for a single restart of $\textsf{Arbe-Gap}$. Now, we have 
\begin{align*}
     3c' \sqrt{\frac{t - t_0}{\rho_i} \ln \frac{\ln t}{\delta}} + \frac{3c'}{\rho_i}\ln \frac{\ln t}{\delta}
     &\geq 
     c' \sqrt{\frac{t - t_0}{\rho_i} \ln \frac{(\ln t)^9}{\delta^9}} + \frac{c'}{\rho_i}\ln \frac{(\ln t)^3}{\delta^3}\\
     &\geq 
     c' \sqrt{\frac{t - t_0}{\rho_i} \ln \frac{(\ln t)^3}{\delta}} + \frac{c'}{\rho_i}\ln \frac{(\ln t)^3}{\delta}\\
     &\geq 
     c' \sqrt{\frac{t - t_0}{\rho_i} \ln \frac{\ln t}{\frac{\delta}{\ln^2 t}}} + \frac{c'}{\rho_i}\ln \frac{\ln t}{\frac{\delta}{\ln^2 t}}\\     
     &\geq 
     c' \sqrt{\frac{t - t_0}{\rho_i} \ln \frac{\ln t}{\frac{\delta}{n^2}}} + \frac{c'}{\rho_i}\ln \frac{\ln t}{\frac{\delta}{n^2}}
\end{align*}
where the last step holds because the number of calls $n$ to $\textsf{Arbe-Gap}$ due to a change in candidate policy  at round $t$ is bounded as $\ln(t)$.
Hence, with this choice of constant, we can ensure that the statement of \pref{lem:gap_estimation_single_epoch} holds with probability at least $1 - \poly(M)\delta$ jointly for all calls of $\textsf{Arbe-Gap}$ (for the remaining $i_\star \leq M$ restarts possible due to elimination of a base learner, we apply a standard union bound).

Now, just as in the proof of \pref{thm:arbe_main_regret}, we write the regret of $\textsf{Arbe-Gap}$ using $\bar t_s = \max\{\min\{\tau_s, t\} - \tau_{s-1}, 0\}$ as 
\begin{align*}
    &\regret_{\MM}(t, \Pi_M)\\ 
    &= \sum_{i = 1}^{\ln(t) + i_\star} \regret([\tau_{i-1} + 1, \min\{\tau_i, t \}], \Pi_M)\\
    &= \sum_{i = 1}^{\ln(t) + i_\star}
    \Ocal\left(\left(M R(\wt \Pi_{i_\star})  + \frac{R(\wt \Pi_{i_\star})^2}{R(\wt \Pi_{1})}  \sqrt{ i_\star}\right) \sqrt{\bar t_i \ln \frac{t}{\delta}} + M \ln \frac{\ln t}{\delta}\right) \tag{\pref{lem:gap_estimation_single_epoch}}\\
    &= 
    \Ocal\left(\left(M R(\wt \Pi_{i_\star})  + \frac{R(\wt \Pi_{i_\star})^2}{R(\wt \Pi_{1})}  \sqrt{ i_\star}\right) \sqrt{(\ln(t) + i_\star)\sum_{i = 1}^{\ln(t) + i_\star} \bar t_i \ln \frac{t}{\delta}}\right) \\
    & \quad + \Ocal \left(M (\ln(t) + i_\star) \ln \frac{\ln t}{\delta}\right) 
    \tag{Jensen's inequality}
    \\
    &= 
    \Ocal\left(\left(M R(\wt \Pi_{i_\star})  + \frac{R(\wt \Pi_{i_\star})^2}{R(\wt \Pi_{1})}  \sqrt{ i_\star}\right) \sqrt{t (\ln(t) + i_\star) \ln \frac{t}{\delta}}\right) \\
    &\quad + \Ocal\left(M (\ln(t) + i_\star) \ln \frac{\ln t}{\delta}\right) \tag{$\sum_{s} \bar t_s \leq t$}\\          &= 
    \Ocal\left(\left(M R(\wt \Pi_{i_\star})  + \frac{R(\wt \Pi_{i_\star})^2}{R(\wt \Pi_{1})}  \sqrt{ i_\star}\right) \sqrt{t (\ln(t) + i_\star) \ln \frac{t}{\delta}}\right) \tag{$t \geq 2i_\star$ by assumption}
\end{align*}
This concludes the proof.
\end{proof}

\subsubsection{Stochastic Guarantees}
As a first step, we show that $\textsf{Arbe-Gap}$ always maintains valid confidence bounds on the gap of the candidate policy in a stochastic environment:
\begin{lemma}[Confidence bounds on the gap]\label{lem:valid_gap_bounds}
Consider a run of \pref{alg:gap_estimation_balanced} with inputs $n=1, t_0 = 0$, arbitrary policy policy $\wh \pi \in \Pi_M$ and $M$ base learners $\alg[1], \dots, \alg[M]$ with $1 \leq R(\wt \Pi_1) \leq R( \wt \Pi_2) \leq \dots \leq R(\wt \Pi_M)$ where $\wt \Pi_i$ is the extended version of policy class $\Pi_i$ with $(M+1-i)$ additional actions. Assume the environment $\bbB$ is stochastic and there is a policy $\pi_\star$ with gap $\Delta > 0$. Then with probability at least $1 - \poly(M)\delta$ in all rounds $t \in \bbN$
\begin{align*}
\wh \Delta_t \leq \Delta_{\wh \pi} \leq \wh \Delta_t + 2 \textrm{W}(t_0, t) 
\end{align*}
where $\textrm{W}(t_0, t)$ is the term used in the definition of $\wh \Delta_t$ in the algorithm with
\begin{align*}
  \textrm{W}(t_0, t) = \Theta\left( \frac{R(\wt \Pi_{M})}{\sqrt{\rho_M}} \sqrt{\frac{\ln \frac{n(t - t_0)}{\delta}}{t - t_0}} + \frac{1}{\rho_M}   \frac{\ln \frac{n\ln (t - t_0)}{\delta}}{t - t_0} \right),
\end{align*}
and $t_0$ and $n$ are the time and number of the last restart and before $t$ and $\Delta_{\wh \pi} = \indicator{ \wh \pi = \pi_\star} \Delta$ is the gap of the candidate policy $\wh \pi$ in round $t$.
\end{lemma}
\begin{proof}
First, consider a single restart of \textsf{Arbe-Gap}. We note that both base learner $\alg[M]$ and $\alg[M+1]$ are h-stable on their respective policy classes $\wt \Pi_M$ and $\wt \Pi_{M+1}$ by assumption (removing a single policy usually does not impede $h$-stability, see e.g. \pref{app:geometric_hedge}). Further, denote by $\pi_\star^i = \argmin_{\pi \in \wt \Pi_i} \bbE_{a \sim \pi, x \sim \Dcal}[r(a, x)]$ a best policy in policy class $ \wt \Pi_i$. We can use these insights and definition to derive the following lower-bound on $\tildecrew_{i}(t_0, t)$ for $i \in \{M, M+1\}$ that holds uniformly with probability at least $1 - \poly(M)\delta$
\begin{align*}
&\tildecrew_i(t_0, t) \\
&\geq \sum_{\ell = t_0 + 1}^t r_\ell(a^i_\ell, x_\ell) - \Ocal\left( \sqrt{\frac{t - t_0}{\rho_i} \ln \frac{ \ln (t- t_0)}{\delta}}  + \frac{1}{\rho_i} \ln \frac{\ln (t - t_0)}{\delta}\right) \tag{\pref{lem:basic_concentration_adv}}\\
&\geq 
\max_{\pi' \in \wt \Pi_i} \sum_{\ell = t_0 + 1}^t \bbE_{a \sim \pi'}[r_\ell(a, x_\ell)] - \Ocal\left( R(\wt \Pi_{i}) \sqrt{\frac{t - t_0}{\rho_i} \ln \frac{t- t_0}{\delta}}  + \frac{1}{\rho_i} \ln \frac{\ln (t - t_0)}{\delta}\right) \tag{h-stability of $\alg[i]$}\\
&\geq 
\sum_{\ell = t_0 + 1}^t \bbE_{a \sim \pi^i_\star}[r_\ell(a, x_\ell)] - \Ocal\left( R(\wt \Pi_{i}) \sqrt{\frac{t - t_0}{\rho_i} \ln \frac{ t- t_0}{\delta}}  + \frac{1}{\rho_i} \ln \frac{\ln (t - t_0)}{\delta}\right)
\tag{$\pi_\star^i \in \wt \Pi_i$}\\
&\geq 
\sum_{\ell = t_0 + 1}^t \bbE_{a \sim \pi^i_\star, x \sim \Dcal}[r(a, x)] - \Ocal\left( R(\wt \Pi_{i}) \sqrt{\frac{t - t_0}{\rho_i} \ln \frac{ t- t_0}{\delta}}  + \frac{1}{\rho_i} \ln \frac{\ln (t - t_0)}{\delta}\right)\\
&\quad - \Ocal\left( \sqrt{(t - t_0) \ln \frac{\ln t}{\delta}} \right)
\tag{\pref{lem:xr_concentration}}\\
&=
\sum_{\ell = t_0 + 1}^t \bbE_{a \sim \pi^i_\star, x \sim \Dcal}[r(a, x)] - \Ocal\left( R(\wt \Pi_{i}) \sqrt{\frac{t - t_0}{\rho_i} \ln \frac{ t- t_0}{\delta}}  + \frac{1}{\rho_i} \ln \frac{\ln (t - t_0)}{\delta}\right)~.
\end{align*}
Conversely, using similar concentration arguments, we can derive the following upper-bound for $\tildecrew_{i}(t_0, t)$ for $i \in \{M, M+1\}$ that holds uniformly with probability at least $1 - \poly(M)\delta$:
\begin{align*}
    &\tildecrew_i(t_0, t) \\
    &\leq \sum_{\ell = t_0 + 1}^t r_\ell(a^i_\ell, x_\ell) + \Ocal\left( \sqrt{\frac{t - t_0}{\rho_i} \ln \frac{ \ln (t- t_0)}{\delta}}  + \frac{1}{\rho_i} \ln \frac{\ln (t - t_0)}{\delta}\right) \tag{\pref{lem:basic_concentration_adv}}\\
    &\leq \sum_{\ell = t_0 + 1}^t \bbE_{a \sim \pi_\ell^i}[r_\ell(a, x_\ell)] + \Ocal\left( \sqrt{\frac{t - t_0}{\rho_i} \ln \frac{ \ln (t- t_0)}{\delta}}  + \frac{1}{\rho_i} \ln \frac{\ln (t - t_0)}{\delta}\right) \tag{\pref{lem:basic_concentration_adv}}\\
    &\leq \sum_{\ell = t_0 + 1}^t \bbE_{a \sim \pi_\ell^i, x \sim \Dcal}[r(a, x)] +  \Ocal\left( \sqrt{\frac{t - t_0}{\rho_i} \ln \frac{ \ln (t- t_0)}{\delta}}  + \frac{1}{\rho_i} \ln \frac{\ln (t - t_0)}{\delta}\right)
    \tag{\pref{lem:xr_concentration}}\\
    &\leq \sum_{\ell = t_0 + 1}^t \bbE_{a \sim \pi_\star^i, x \sim \Dcal}[r(a, x)] +  \Ocal\left( \sqrt{\frac{t - t_0}{\rho_i} \ln \frac{ \ln (t- t_0)}{\delta}}  + \frac{1}{\rho_i} \ln \frac{\ln (t - t_0)}{\delta}\right)
    \tag{def. of $\pi^i_\star$}
\end{align*}
Note that $\Delta_{\wh \pi} = \bbE_{a \sim \pi_\star^M, x \sim \Dcal}[r(a, x)] - \bbE_{a \sim \pi_\star^{M+1}, x \sim \Dcal}[r(a, x)]$ which is either $\Delta$ if $\wh \pi = \pi_\star$ or $0$ if $\wh \pi \neq \pi^\star$. Combining the bounds on $\tildecrew_{i}$ with a union bound, we can derive the following deviation bound
\begin{align*}
&- \Ocal\left( \left(\frac{R(\wt \Pi_{M})}{\sqrt{\rho_M}} + \frac{1}{\sqrt{\rho_{M+1}}}\right) \sqrt{\frac{\ln \frac{t - t_0}{\delta}}{t - t_0}} + \left(\frac{1}{\rho_M} + \frac{1}{\rho_{M+1}}\right)  \frac{\ln \frac{\ln (t - t_0)}{\delta}}{t - t_0} \right)
\\
& \qquad \quad \leq  \frac{\tildecrew_{M}(t_0, t) - \tildecrew_{M+1}(t_0, t)}{t - t_0}
    - \Delta_{\wh \pi}
    \leq\\
    & +\Ocal\left( \left(\frac{R(\wt \Pi_{M+1})}{\sqrt{\rho_{M+1}}} + \frac{1}{\sqrt{\rho_{M}}}\right) \sqrt{\frac{\ln \frac{t - t_0}{\delta}}{t - t_0}} + \left(\frac{1}{\rho_M} + \frac{1}{\rho_{M+1}}\right)  \frac{\ln \frac{\ln (t - t_0)}{\delta}}{t - t_0} \right)~.
\end{align*}
We can further simplify those bounds by noting that $R(\wt \Pi_{M+1}) \leq R(\wt \Pi_{M})$ and thus also $\rho_{M+1} \geq \rho_{M}$ since $\wt \Pi_{M}$ is identical to $\wt \Pi_{M+1}$ but contains two more policies. Thus, we can bound the magnitude of the upper and lower bound further by 
\begin{align*}
    \Ocal\left( \frac{R(\wt \Pi_{M})}{\sqrt{\rho_M}} \sqrt{\frac{\ln \frac{t - t_0}{\delta}}{t - t_0}} + \frac{1}{\rho_M}   \frac{\ln \frac{\ln (t - t_0)}{\delta}}{t - t_0} \right).
\end{align*} 
We now rebind $\delta$ by $\frac{\delta}{n^2}$ and apply a union bound over all restarts of \textsf{Arbe-Gap}. Thus, we can choose a constant in the definition of 
\begin{align*}
    \textrm{W}(t_0, t) = \Theta\left( \frac{R(\wt \Pi_{M})}{\sqrt{\rho_M}} \sqrt{\frac{\ln \frac{n(t - t_0)}{\delta}}{t - t_0}} + \frac{1}{\rho_M}   \frac{\ln \frac{n\ln (t - t_0)}{\delta}}{t - t_0} \right)
\end{align*} 
large enough so that
\begin{align*}
   - \textrm{W}(t_0, t) \leq  \frac{\tildecrew_{M}(t_0, t) - \tildecrew_{M+1}(t_0, t)}{t - t_0}
    - \Delta_{\wh \pi}
    \leq \textrm{W}(t_0, t)
\end{align*}
holds for all $t$ in all possible restarts of \textsf{Arbe-Gap} with the desired $1 - \poly(M)\delta$ probability.
\end{proof}
The lemma above immediately implies the correctness of the first phase, in the sense that if the algorithm moves on to the second phase in a stochastic environment, the candidate policy has to be optimal and the gap estimate is accurate up to a multiplicative factor:

\begin{corollary}\label{cor:gap_estimate_correctness}
Consider a run of \pref{alg:gap_estimation_balanced} with inputs $n=1, t_0 = 0$, arbitrary policy policy $\wh \pi \in \Pi_M$ and $M$ base learners $\alg[1], \dots, \alg[M]$ with $1 \leq R(\wt \Pi_1) \leq R( \wt \Pi_2) \leq \dots \leq R(\wt \Pi_M)$ where $\wt \Pi_i$ is the extended version of policy class $\Pi_i$ with $(M+1-i)$ additional actions. Assume the environment $\bbB$ is stochastic and there is a policy $\pi_\star$ with gap $\Delta > 0$. Then with probability at least $1 - \poly(M)\delta$ the policy $\wh \pi$ and gap estimate $\wh \Delta$ passed to \textsf{Arbe-GapExploit} satisfy
\begin{align*}
    \wh \pi = \pi_\star \qquad \textrm{and} \qquad \wh \Delta \leq \Delta \leq 2\wh \Delta.    
\end{align*}
\end{corollary}
\begin{proof}
The statement follows from \pref{lem:valid_gap_bounds} and the condition in Line~\ref{lin:gap_test} of \pref{alg:gap_estimation_balanced}. First, note that the test cannot trigger when $\wh \pi \neq \pi_\star$ since $\wh \Delta_t \leq 0$ in this case. Second, since $\wh \Delta$ satisfies $2 W(t_0, t) \leq \wh \Delta$ and $\wh \Delta \leq \Delta \leq \wh \Delta + 2 W(t_0, t)$ when the test triggers, we have
\begin{align*}
    \wh \Delta \leq \Delta \leq \wh \Delta + 2 W(t_0, t) \leq 2 \wh \Delta~,
\end{align*}
as claimed.
\end{proof}

We now move on to show that if there is a policy with a gap, the alorithm has to identify it within a certain number of rounds:
\begin{lemma}[\textsf{Arbe-Gap} selects the right candidate policy]
\label{lem:arbe_gap_policy_identification}
Consider a run of \pref{alg:gap_estimation_balanced} with inputs $t_0 = 0$, arbitrary policy policy $\wh \pi \in \Pi_M$ and $M$ base learners $\alg[1], \dots, \alg[M]$ with $1 \leq R(\wt \Pi_1) \leq R( \wt \Pi_2) \leq \dots \leq R(\wt \Pi_M)$ where $\wt \Pi_i$ is the extended version of policy class $\Pi_i$ with $(M+1-i)$ additional actions. Assume the environment $\bbB$ is stochastic and there is a policy $\pi_\star$ with gap $\Delta > 0$. Then with probability at least $1 - \poly(M)\delta$ the number of rounds until $\pi_\star$ is always chosen as the candidate policy $\wh \pi$ is bounded as 
\begin{align*}
\Ocal\left( \left( M^2   + \frac{R(\wt \Pi_{i_\star})^2}{R(\wt \Pi_{1})^2}   i_\star\right) \frac{R(\wt \Pi_{i_\star})^2 i_\star}{\Delta^2} \ln^2 \frac{M R(\wt \Pi_{i_\star})}{\Delta \delta}\right).
\end{align*}
\end{lemma}
\begin{proof}
By \pref{lem:arbegap_adversarial_regret}, with probability at least $1 - \poly(M)\delta$ the regret of \textsf{Arbe-Gap} is bounded for all rounds $t \in [2 i_\star, t_{\mathrm{gap}}]$ as
\begin{align*}
     \regret_{\bbM}(t, \Pi_M) &= 
    \Ocal\left(\left(M R(\wt \Pi_{i_\star})  + \frac{R(\wt \Pi_{i_\star})^2}{R(\wt \Pi_{1})}  \sqrt{ i_\star}\right) \sqrt{t (\ln(t) + i_\star) \ln \frac{t}{\delta}}\right)\\
    &=
    \Ocal\left(\left(M R(\wt \Pi_{i_\star}) \sqrt{ i_\star}  + \frac{R(\wt \Pi_{i_\star})^2}{R(\wt \Pi_{1})}   i_\star\right) \sqrt{t} \ln \frac{t}{\delta}\right)~.
\end{align*}
By the concentration argument in \pref{lem:regret_pseudoregret_concentration}, the same bound can be established for the pseudo-regret
\begin{align}
    \pseudoregret_{\bbM}(t, \Pi_M) &\leq  c\left(M R(\wt \Pi_{i_\star}) \sqrt{ i_\star}  + \frac{R(\wt \Pi_{i_\star})^2}{R(\wt \Pi_{1})}   i_\star\right) \sqrt{t} \ln \frac{t}{\delta}
    \label{eqn:pseudoregb}
\end{align}
for some sufficiently large absolute constant $c$. Now denote $\gamma = c\left(M R(\wt \Pi_{i_\star}) \sqrt{ i_\star}  + \frac{R(\wt \Pi_{i_\star})^2}{R(\wt \Pi_{1})}   i_\star\right)$ and consider the value 
\begin{align*}
    t' = \frac{16^2  \gamma^2}{\Delta^2} \ln^2 \frac{8 \gamma}{\Delta \delta}~.
\end{align*}
Then by the properties of $\frac{\ln(t)}{\sqrt{t}}$ investigated in \pref{lem:concentration_threshold}, we can bound for $t \geq t'$
\begin{align*}
    \pseudoregret_{\bbM}(t, \Pi_M)
    &\leq \gamma \sqrt{t} \ln \frac{t}{\delta} \tag{\pref{eqn:pseudoregb}}\\
    &= \gamma t \frac{\ln(t / \delta)}{\sqrt{t}}
    \leq \gamma t \frac{\Delta}{4 \gamma} \tag{\pref{lem:concentration_threshold}}\\
    &= \frac{t}{4} \Delta~.
\end{align*}
We have shown that the adversarial regret rate implies that the pseudo-regret for rounds $t \geq t'$ has to be bounded by $\frac{t}{4} \Delta$. Since each policy but $\pi^\star$ incurs a pseudo-regret at least $\Delta$ per round, \textsf{Arbe-Gap} has to select $\pi^\star$ in at least $\frac{3}{4}t$ among all $t$ rounds to satisfy this pseudo-regret bound. As a result, a switch of the candidate policy to $\pi^\star$ would be triggered if it is not already the candidate policy. Further, no other policy can be selected more than a quarter of the times, thus the candidate policy has to be $\pi^\star$ in all rounds $t \geq t'$.
\end{proof}

\begin{lemma}[Number of Rounds in the First Phase]
\label{lem:arbe_gap_phase1_length}
Consider a run of \pref{alg:gap_estimation_balanced} with inputs $t_0 = 0, n=1$, arbitrary policy $\wh \pi \in \Pi_M$ and $M$ base learners $\alg[1], \dots, \alg[M]$ with $1 \leq R(\wt \Pi_1) \leq R( \wt \Pi_2) \leq \dots \leq R(\wt \Pi_M)$ where $\wt \Pi_i$ is the extended version of policy class $\Pi_i$ with $(M+1-i)$ additional actions. Assume the environment $\bbB$ is stochastic and there is a policy $\pi_\star$ with gap $\Delta > 0$. Then with probability at least $1 - \poly(M)\delta$ the number of rounds until the algorithm enters the second phase by calling $\textsf{ArbeGap-Exploit}$ is bounded as
\begin{align*}
\Ocal \left( \left(\frac{ R(\wt \Pi_M)^4}{R(\wt \Pi_1)^2 } 
    +  M  R(\wt \Pi_{i_\star})^2  \right)
    \frac{Mi_\star }{\Delta^2}
    \ln^2 \frac{M R(\wt \Pi_{i_\star})}{\Delta \delta}
    \right)~.
\end{align*}
\end{lemma}
\begin{proof}
By \pref{lem:arbe_gap_policy_identification}, after a certain number of rounds $t_{pol}$, the candidate policy has to be $\pi_\star$ at all rounds. Hence, there can be no restarts due to candidate policy switches anymore. According to \pref{lem:arbe_gap_restarts}, there can only be up to $i_\star$ restarts after round $t_{pol}$ due to elimination of a base learner. We will in the following show that if \textsf{ArbeGap} has been (re)started with candidate policy $\pi_\star$ and there are no other restarts in the meantime, it has to switch to the second phase within a certain number of rounds $k$. The total number of rounds in the first phase, is then bounded by
\begin{align*}
    t_{pol} + i_\star \cdot k~.
\end{align*}
We will now show a bound on $k$. By \pref{lem:valid_gap_bounds}, we have at all times that $\wh \Delta_t \leq \Delta$ (since $\wh \pi = \pi_\star$ by assumption) and the algorithm moves on to the second phase as soon as $2W(t_0, t) \leq \wh \Delta_t$. Note that the condition of $\wh \Delta_t \leq \Delta \leq 1 \leq R(\wt \Pi_M)^2$ is always satisfies in stochastic environments. Hence,the algorithm cannot stay in the first phase if $2W(t_0, t) \leq \Delta$ or, plugging in the definition of $W(t_0, t)$ with an appropriate absolute constant $c$
\begin{align*}
\frac{cR(\wt \Pi_{M})}{\sqrt{\rho_M}} \sqrt{\frac{\ln \frac{n(t - t_0)}{\delta}}{t - t_0}} + \frac{c}{\rho_M}   \frac{\ln \frac{n\ln (t - t_0)}{\delta}}{t - t_0} \leq \Delta.
\end{align*}
Hence, we can obtain an value for the bound $k$ by identifying a value that satisfies
\begin{align*}
    \frac{\ln \frac{k}{\delta / n}}{k} \leq \frac{\Delta \rho_M}{2c} \quad \textrm{and}
    \frac{\ln \frac{k}{\delta / n}}{k} \leq \frac{\Delta^2 \rho_M}{4c^2 R(\wt \Pi_M)^2}~.
\end{align*}
Since $\Delta \in (0, 1]$ and $c, R(\wt \Pi_M) \geq 1$ without loss of generality, it is sufficient to only consider the condition on the right. By \pref{lem:concentration_threshold2}, we can set $k$ as
\begin{align*}
    k = \frac{16c^2 R(\wt \Pi_M)^2}{\Delta^2 \rho_M} \ln \left( \frac{2n}{\delta}\frac{4c^2 R(\wt \Pi_M)^2}{\Delta^2 \rho_M} \right)
    = \Ocal \left( \frac{M R(\wt \Pi_M)^4}{R(\wt \Pi_1)^2 \Delta^2} \ln \frac{ R(\wt \Pi_M)}{\delta \Delta} \right)~,
\end{align*}
where we used the fact that $n \leq \ln  t_{pol} + i_\star = \Ocal\left(\frac{R(\wt \Pi_M)}{\delta \Delta}\right)$.
Hence, the total length of the first phase can be at most
\begin{align*}
    t_{pol} + i_\star \cdot k
    &= \Ocal \left( \frac{i_\star M R(\wt \Pi_M)^4}{R(\wt \Pi_1)^2 \Delta^2} \ln \frac{ R(\wt \Pi_M)}{\Delta \delta } 
    + \left( M^2   + \frac{R(\wt \Pi_{i_\star})^2}{R(\wt \Pi_{1})^2}   i_\star\right) \frac{R(\wt \Pi_{i_\star})^2 i_\star}{\Delta^2} \ln^2 \frac{M R(\wt \Pi_{i_\star})}{\Delta \delta}
    \right)\\
    &= \Ocal \left( \left(\frac{ R(\wt \Pi_M)^4}{R(\wt \Pi_1)^2 } 
    +  M  R(\wt \Pi_{i_\star})^2  \right)
    \frac{Mi_\star }{\Delta^2}
    \ln^2 \frac{M R(\wt \Pi_{i_\star})}{\Delta \delta}
    \right)\\
    &= \Ocal \left( \frac{M^2 i_\star R(\wt \Pi_M)^4}{R(\wt \Pi_1)^2 \Delta^2} \ln^2 \frac{M R(\wt \Pi_M)}{\Delta \delta }
    \right)~,
\end{align*}
as claimed.
\end{proof}

\begin{lemma}[Pseudo-Regret of the First Phase]
\label{lem:arbe_gap_phase1_pseudoregret}
Consider a run of \pref{alg:gap_estimation_balanced} with inputs $t_0 = 0, n=1$, arbitrary policy $\wh \pi \in \Pi_M$ and $M$ base learners $\alg[1], \dots, \alg[M]$ with $1 \leq R(\wt \Pi_1) \leq R( \wt \Pi_2) \leq \dots \leq R(\wt \Pi_M)$ where $\wt \Pi_i$ is the extended version of policy class $\Pi_i$ with $(M+1-i)$ additional actions. Assume the environment $\bbB$ is stochastic and there is a policy $\pi_\star$ with gap $\Delta > 0$.
Let $t_{gap}$ be the round where  $\textsf{ArbeGap-Exploit}$ was called. Then with probability at least $1 - \poly(M)\delta$, the pseudo-regret in all rounds $t \in [2 i_\star, t_{gap}]$ is bounded as
\begin{align*}
\pseudoregret_{\bbM}(t, \Pi_M) =
\Ocal\left(
    \frac{ R(\wt \Pi_{i_\star})^2 R(\wt \Pi_M)^2}{R(\wt \Pi_1)^2}\frac{M^2 i_\star}{\Delta} \ln \left(\frac{M R(\wt \Pi_M)}{\Delta \delta }\right)
     \ln \frac{t}{\delta}\right).
\end{align*}
\end{lemma}
\begin{proof}
First, we bound the pseudo-regret by regret through a simple concentration argument in \pref{lem:regret_pseudoregret_concentration}
\begin{align*}
    \pseudoregret_{\bbM}(t, \Pi_M) \leq \regret_{\bbM}(t, \Pi_M) + \Ocal\left(\sqrt{t \ln \frac{\ln t}{\delta}} \right).
\end{align*}
Next, we bound the regret by \pref{lem:arbegap_adversarial_regret} which holds for any environment (adversarial but also stochastic). This yields
\begin{align*}
    \pseudoregret_{\bbM}(t, \Pi_M) = \Ocal\left(\left(M R(\wt \Pi_{i_\star})  + \frac{R(\wt \Pi_{i_\star})^2}{R(\wt \Pi_{1})}  \sqrt{ i_\star}\right) \sqrt{t (\ln(t) + i_\star) \ln \frac{t}{\delta}}\right)
\end{align*}
and finally, we use the bound on $t \leq t_{gap}$, the length of the first phase from \pref{lem:arbe_gap_phase1_length} to replace $t ( \ln(t) + i_\star)$ above which gives
\begin{align*}
    &\pseudoregret_{\bbM}(t, \Pi_M)\\
    &\ = \Ocal\left(\left(M R(\wt \Pi_{i_\star})  + \frac{R(\wt \Pi_{i_\star})^2}{R(\wt \Pi_{1})}  \sqrt{ i_\star}\right) 
    \frac{ R(\wt \Pi_M)^2}{R(\wt \Pi_1) } 
    \frac{M\sqrt{i_\star} }{\Delta}
     \ln \left(\frac{M R(\wt \Pi_M)}{\Delta \delta }\right)
    \sqrt{
    (\ln(t) + i_\star) \ln \frac{t}{\delta}}\right)\\
    &\ = 
    \Ocal\left(\left(M   + \frac{R(\wt \Pi_{i_\star})}{R(\wt \Pi_{1})}  \sqrt{ i_\star}\right) 
    \frac{M i_\star R(\wt \Pi_{i_\star}) R(\wt \Pi_M)^2}{R(\wt \Pi_1) \Delta} \ln \left(\frac{M R(\wt \Pi_M)}{\Delta \delta }\right)
     \ln \frac{t}{\delta}\right)\\
     &\ = 
    \Ocal\left(
    \frac{M^2 i_\star R(\wt \Pi_{i_\star})^2 R(\wt \Pi_M)^2}{R(\wt \Pi_1)^2 \Delta} \ln \left(\frac{M R(\wt \Pi_M)}{\Delta \delta }\right)
     \ln \frac{t}{\delta}\right)~,
\end{align*}
as claimed.
\end{proof}

\subsection{Analysis of \textsf{Arbe-GapExploit}}
\label{app:proofs_gapexploit}
In this section we prove the following result:
\mainresultexploitation*

\subsubsection{Guarantees for stochastic environments}
\begin{lemma}[\pref{alg:exploitation_subroutine} does not terminate in stochastic environemnts]\label{lem:exploit_phase_confband}
Assume the environment is stochastic and there is an optimal policy $\pi_\star$ with a gap $\Delta$ compared to the best policy in $\Pi_{\alg}$. If \pref{alg:exploitation_subroutine} is called with inputs $\wh \pi = \pi_\star$ and $\wh \Delta \leq \Delta \leq 2 \wh \Delta$, then with probability at least $1 - \Ocal(\delta)$ it never terminates.
\end{lemma}
\begin{proof}
Let $e$ be an arbitrary epoch. By following the steps of \pref{lem:basic_concentration_adv}, we can show that with probability at least $1 - \Ocal(\delta_e)$ for all time steps in epoch $e$
\begin{align*}
   \left| \tildecrew_1(t_e+1, t) - \sum_{\ell = t_e + 1}^t r_\ell(a_\ell^1, x_\ell) \right| &\leq \Ocal \left(\sqrt{\frac{t - t_e}{\rho^e} \ln \frac{\ln (t - t_e)}{\delta_e}} + \frac{1}{\rho^e} \ln \frac{\ln (t - t_e)}{\delta_e}\right)\\
      \left| \tildecrew_0(t_e+1, t) - \sum_{\ell = t_e + 1}^t r_\ell(a_\ell^0, x_\ell) \right| &\leq \Ocal \left(\sqrt{\frac{t - t_e}{1 - \rho^e} \ln \frac{\ln (t - t_e)}{\delta_e}} + \frac{1}{1 - \rho^e} \ln \frac{\ln (t - t_e)}{\delta_e}\right)
\end{align*}
and for $i \in \{0, 1\}$
\begin{align*}
    \left| \sum_{\ell = t_e + 1}^t r_\ell(a_\ell^i, x_\ell) - \sum_{\ell = t_e + 1}^t \bbE_{a \sim \pi_\ell^i}[r_\ell(a, x_\ell)\right| \leq \Ocal \left(\sqrt{(t - t_e) \ln \frac{\ln (t - t_e)}{\delta_e}}\right).
\end{align*}
Combining these bound together with \pref{lem:regret_pseudoregret_concentration}, we have with $\pi_{\star, \alg}$ being the best policy in $\Pi_{\alg}$
\begin{align*}
    &\tildecrew_0(t_e + 1, t) - \tildecrew_1(t_e + 1, t) \\
    & \leq 
     \sum_{\ell = t_e + 1}^t \left[\bbE_{a \sim \pi_\ell^0, x \sim \Dcal}[r(a, x)] - \bbE_{a \sim \pi_\ell^1, x \sim \Dcal}[r(a, x) ] \right] \\
     & \qquad +
     \Ocal \left(\sqrt{\frac{t - t_e}{\rho^e} \ln \frac{\ln (t - t_e)}{\delta_e}} + \frac{1}{\rho^e} \ln \frac{\ln (t - t_e)}{\delta_e}\right)
     \\
     & = \sum_{\ell = t_e + 1}^t \left[\bbE_{a \sim \pi_\star, x \sim \Dcal}[r(a, x)] - 
     \bbE_{a \sim \pi_{\star, \alg}, x \sim \Dcal}[r(a, x)] \right] + \Ocal \left(\sqrt{\frac{t - t_e}{\rho^e} \ln \frac{\ln (t - t_e)}{\delta_e}}\right) 
     \\ & \qquad
     +
     \sum_{\ell = t_e + 1}^t \left[
     \bbE_{a \sim \pi_{\star, \alg}, x \sim \Dcal}[r(a, x)] 
     -\bbE_{a \sim \pi_\ell^1, x \sim \Dcal}[r(a, x) ] \right] + \Ocal \left(\frac{1}{\rho^e} \ln \frac{\ln (t - t_e)}{\delta_e}\right)\\
     & = \Delta (t - t_e) + \pseudoregret_{\alg}([t_e+1, t], \Pi_{\alg})
     +
     \Ocal \left(\sqrt{\frac{t - t_e}{\rho^e} \ln \frac{\ln (t - t_e)}{\delta_e}} + \frac{1}{\rho^e} \ln \frac{\ln (t - t_e)}{\delta_e}\right)
     \\
     & \leq \Delta (t - t_e) + \regret_{\alg}([t_e+1, t], \Pi_{\alg}) +
     \Ocal \left(\sqrt{\frac{t - t_e}{\rho^e} \ln \frac{\ln (t - t_e)}{\delta_e}} + \frac{1}{\rho^e} \ln \frac{\ln (t - t_e)}{\delta_e}\right) \tag{\pref{lem:regret_pseudoregret_concentration}}\\
     & \leq \Delta (t - t_e) + \Ocal\left( R(\Pi_{\alg}) \sqrt{\frac{t - t_e}{\rho^e}\ln \frac{t - t_e}{\delta_e}} + \frac{1}{\rho^e} \ln \frac{\ln (t - t_e)}{\delta_e}\right)
     \tag{$\alg$ is h-stable}\\
     & \leq 2 \wh \Delta (t - t_e) + \Ocal\left( R(\Pi_{\alg}) \sqrt{\frac{t - t_e}{\rho^e}\ln \frac{t - t_e}{\delta_e}} + \frac{1}{\rho^e} \ln \frac{\ln (t - t_e)}{\delta_e}\right)
     \tag{$\Delta \leq 2 \wh \Delta$}
\end{align*}
where we used the fact that $\rho^e \leq 1/2 \leq 1 - \rho^e$ by the choice of constants (see \pref{lem:exploit_regret_polclass}).
This chain of inequalities holds with probability $1 - \Ocal(\delta_e)$. Combining this with a union bound, this implies that with this probability at least $1 - \Ocal(\delta)$, the test Line~\ref{line::exploit_test_two} never triggers. Similarly, we can lower-bound the same term as
\begin{align*}
    &\tildecrew_0(t_e + 1, t) - \tildecrew_1(t_e + 1, t) \\
    & \geq 
     \sum_{\ell = t_e + 1}^t \left[\bbE_{a \sim \pi_\ell^0, x \sim \Dcal}[r(a, x)] - \bbE_{a \sim \pi_\ell^1, x \sim \Dcal}[r(a, x) ] \right] \\
     & \qquad -
     \Ocal \left(\sqrt{\frac{t - t_e}{\rho^e} \ln \frac{\ln (t - t_e)}{\delta_e}} + \frac{1}{\rho^e} \ln \frac{\ln (t - t_e)}{\delta_e}\right)
     \\
         & \geq 
     \sum_{\ell = t_e + 1}^t \left[\bbE_{a \sim \pi_\star, x \sim \Dcal}[r(a, x)] - \bbE_{a \sim \pi_{\star, \alg}, x \sim \Dcal}[r(a, x) ] \right] \tag{$\pi^0_\ell = \pi_\star$}\\
     & \qquad -
     \Ocal \left(\sqrt{\frac{t - t_e}{\rho^e} \ln \frac{\ln (t - t_e)}{\delta_e}} + \frac{1}{\rho^e} \ln \frac{\ln (t - t_e)}{\delta_e}\right)
     \\
         & \geq 
     \Delta (t - t_e) -
     \Ocal \left(\sqrt{\frac{t - t_e}{\rho^e} \ln \frac{\ln (t - t_e)}{\delta_e}} + \frac{1}{\rho^e} \ln \frac{\ln (t - t_e)}{\delta_e}\right)     \\
         & \geq 
     \wh \Delta (t - t_e) -
     \Ocal \left(\sqrt{\frac{t - t_e}{\rho^e} \ln \frac{\ln (t - t_e)}{\delta_e}} + \frac{1}{\rho^e} \ln \frac{\ln (t - t_e)}{\delta_e}\right) \tag{$\wh \Delta \leq \Delta$}
\end{align*}
Hence, after combining these statements with a union bound, this implies that with this probability at least $1 - \Ocal(\delta)$, the test Line~\ref{line::exploit_test_one} never triggers. A final union bound for both tests completes the proof.
\end{proof}

\begin{lemma}[Pseudo-regret of \pref{alg:exploitation_subroutine}]\label{lem:arbe_gap_exploit_regregretstochastic}
Assume the environment is stochastic and there is an optimal policy $\pi_\star$ with a gap $\Delta$. If \pref{alg:exploitation_subroutine} is called with inputs $\wh \pi = \pi_\star$ and $\wh \Delta = \Theta(\Delta)$, then the pseudo-regret of the algorithm is bounded with probability at least $1 - \Ocal(\delta)$ as
\begin{align*}
    \pseudoregret_{\bbM}([t_0+1, t], \{ \wh \pi \} \cup \Pi_{\alg}) = \Ocal\left( \frac{R(\Pi_{\alg})^2}{\Delta} \left(\ln \frac{R(\Pi_{\alg})}{\Delta \delta} + \ln(t) \ln \frac{t}{\delta}\right) \right)
\end{align*}
for all rounds $t \geq t_0$ where the algorithm has not terminated.
\end{lemma}

\begin{proof}
The pseudo-regret of \pref{alg:exploitation_subroutine} can be decomposed into the regret in each epoch as
\begin{align*}
    \pseudoregret_{\bbM}([t_0+1, t], \{ \wh \pi \} \cup \Pi_{\alg})
    = \sum_{e=0}^{j(t)} \pseudoregret_{\bbM}([t_e+1, t'_e], \{ \wh \pi \} \cup \Pi_{\alg})
\end{align*}
where $t'_e = \min\{t_{e+1}, t\}$ and $j(t) = \min\{ e \in \bbN \colon t_{e+1} \geq t\}$ is the epoch at time $t$. We consider the regret in each epoch separately as
\begin{align*}
    &\pseudoregret_{\bbM}([t_e+1, t'_e], \{ \wh \pi \} \cup \Pi_{\alg})\\
    &= \sum_{\ell = t_e+1}^{t'_{e}} \left(\bbE_{a \sim \pi^\star, x \sim \Dcal}[r(a, x)] - \bbE_{a \sim \pi_t, x \sim \Dcal}[r(a, x)]\right) \\
     &= \sum_{\ell = t_e+1}^{t'_{e}} b_\ell \left(\bbE_{a \sim \pi^\star, x \sim \Dcal}[r(a, x)] - \bbE_{a \sim \pi_\ell^1, x \sim \Dcal}[r(a, x)]\right) \tag{$\wh \pi = \pi^\star$} \\
     &\leq \sum_{\ell = t_e+1}^{t'_{e}} b_\ell  \left(\Delta + \bbE_{a \sim \pi^\star_{\alg}, x \sim \Dcal}[r(a, x)] - \bbE_{a \sim \pi_\ell^1, x \sim \Dcal}[r(a, x)]\right)
\end{align*}
where $\pi^\star_{\alg}$ is the best policy in $\Pi_{\alg}$ which incurs a pseudo-regret of at most $\Delta$ per round. We now apply a concentration argument.
Denote $Y_\ell = \indicator{\ell > t_e} \left(\Delta + \bbE_{a \sim \pi^\star_{\alg}, x \sim \Dcal}[r(a, x)] - \bbE_{a \sim \pi_\ell, x \sim \Dcal}[r(a, x)]\right)$ and let $\Fcal_\ell$ be the sigma-field that includes $\{\pi_{t}^1\}_{t > t_0}^{\ell}$, $\{b_t\}_{t > t_0}^{\ell - 1}$ and $t_e$. Note that $Y_\ell$ is $\Fcal_\ell$-measurable and $\sum_{\ell = t_e+1}^{t'_{e}} (b_\ell - \rho^e) Y_\ell$ is a martingale difference sequence w.r.t. $\Fcal_{\ell}$. Since $|(b_\ell - \rho^e) Y_\ell| \leq 2$ and the sequence of conditional variance is bounded as $V_\ell = \sum_{\ell = t_e+1}^{t'_{e}} \rho^e (1 - \rho^e) Y_\ell \leq \sum_{\ell = t_e+1}^{t'_{e}} \rho^e Y_\ell$, we can apply \pref{lem:bernstein_lil} and get with probability at least $1 - \delta_e$ for all rounds $t_e \leq \tilde t \leq t_{e+1}$
\begin{align*}
    \sum_{\ell = t_e+1}^{\tilde t} (b_\ell - \rho^e) Y_\ell 
    &\leq \Ocal\left(\sqrt{\sum_{\ell = t_e+1}^{\tilde t} \rho^e Y_\ell \ln \frac{\ln(\tilde t - t_e)}{\delta_e} } +\ln \frac{\ln(\tilde t - t_e)}{\delta_e} \right)\\
    & \leq \sum_{\ell = t_e+1}^{\tilde t} \rho^e Y_\ell + \leq \Ocal\left(\ln \frac{\ln(\tilde t - t_e)}{\delta_e} \right) \tag{AM-GM inequality}
\end{align*}
This holds in particular for $\tilde t = t'_e$ and with shorthand $k'_e = t'_e - t_e$, this gives
\begin{align*}
    &\pseudoregret_{\bbM}([t_e+1, t'_e], \{ \wh \pi \} \cup \Pi_{\alg})\\
    &\leq 2 \rho^e \sum_{\ell = t_e+1}^{t'_{e}}  \left(\Delta + \bbE_{a \sim \pi^\star_{\alg}, x \sim \Dcal}[r(a, x)] - \bbE_{a \sim \pi_\ell, x \sim \Dcal}[r(a, x)]\right) 
    + \Ocal\left(\ln \frac{\ln(k'_e)}{\delta_e} \right)
    \\
    & = 2 \rho^e k_e' \Delta + \rho^e \pseudoregret_{\alg}([t_e+1, t'_e], \Pi_{\alg}) + \Ocal\left(\ln \frac{\ln(k'_e)}{\delta_e} \right)\\
    & \leq  2 \rho^e k_e' \Delta + \rho^e \regret_{\alg}([t_e+1, t'_e], \Pi_{\alg}) + \Ocal\left(\rho^e \sqrt{k_e' \ln \frac{\ln(k_e')}{\delta_e}} + \ln \frac{\ln(k'_e)}{\delta_e} \right) \tag{\pref{lem:regret_pseudoregret_concentration}}\\
    & \leq  2 \rho^e k_e' \Delta + \Ocal\left(\rho^e R(\Pi_{\alg}) \sqrt{\frac{k_e'}{\rho^e} \ln \frac{k_e'}{\delta_e}} + \ln \frac{\ln(k'_e)}{\delta_e} \right) \tag{$\alg$ is h-stable}\\
    & \leq  2 \rho^e k_e' \Delta + \Ocal\left( R(\Pi_{\alg}) \sqrt{\rho^e k_e' \ln \frac{k_e'}{\delta_e}} + \ln \frac{\ln(k'_e)}{\delta_e} \right)\\
        & \leq  \Ocal\left( \frac{R(\Pi_{\alg})^2 \Delta}{ \wh \Delta^2} \ln \frac{k_e}{\delta_e} +  R(\Pi_{\alg}) \sqrt{ \frac{R(\Pi_{\alg})^2}{\wh \Delta^2} \ln \left(\frac{k_e}{\delta_e}\right) \ln \frac{k_e'}{\delta_e}} + \ln \frac{\ln(k'_e)}{\delta_e} \right) \tag{definition of $\rho^e$ and $k_e' \leq k_e$}\\
         & \leq  \Ocal\left( \frac{R(\Pi_{\alg})^2}{\Delta} \ln \frac{k_e}{\delta_e}  \right) \tag{$\wh \Delta = \Theta(\Delta)$}
\end{align*}
Now, we can combine this pseudo-regret bound across all epochs and get
\begin{align*}
    \pseudoregret_{\bbM}([t_0+1, t], \{ \wh \pi \} \cup \Pi_{\alg})
    &\leq \sum_{e=0}^{j(t)} \Ocal\left( \frac{R(\Pi_{\alg})^2}{\Delta} \ln \frac{k_e}{\delta_e}  \right) \\
    &\leq \Ocal\left( j(t) \frac{R(\Pi_{\alg})^2}{\Delta} \ln \frac{t j(t)}{\delta}  \right) \\
    & = \Ocal \left( \frac{R(\Pi_{\alg})^2}{\Delta} \ln(t) \ln \frac{t}{\delta}  \right) \tag{$j(t) = \Ocal(\ln(t))$}
\end{align*}
when $j(t) \geq 1$ since $k_e \leq k_{j(t)} \leq 2t$ in this case. In the other case, where $j(t) = 0$, we have
\begin{align*}
     \pseudoregret_{\bbM}([t_0+1, t], \{ \wh \pi \} \cup \Pi_{\alg}) \leq \Ocal\left( \frac{R(\Pi_{\alg})^2}{\Delta} \ln \frac{k_0}{\delta_0}  \right)
     \leq \Ocal\left( \frac{R(\Pi_{\alg})^2}{\Delta} \ln \frac{R(\Pi_{\alg})}{\Delta \delta}  \right).
\end{align*}
Hence, combining both cases gives the final bound
\begin{align*}
     \pseudoregret_{\bbM}([t_0+1, t], \{ \wh \pi \} \cup \Pi_{\alg}) = \Ocal\left( \frac{R(\Pi_{\alg})^2}{\Delta} \left(\ln \frac{R(\Pi_{\alg})}{\Delta \delta} + \ln(t) \ln \frac{t}{\delta}\right) \right).
\end{align*}
\end{proof}

\subsubsection{Exploitation Subroutine Guarantees for Adversarial Environments }

\begin{lemma}\label{lem:exploit_regret_polclass}
Assume the absolute constant in the length of the initial epoch $k_0$ of \pref{alg:exploitation_subroutine} is chosen large enough. Then
the regret of \pref{alg:exploitation_subroutine} against $\Pi_{\alg}$ is bounded with probability at least $1 - \Ocal(\delta)$ for all rounds $t$ that the algorithm has not terminate yet as
\begin{align*}
        \regret_{\bbM}([t_0 + 1, t], \Pi_{\alg}) = \Ocal( \wh \Delta k_0).
\end{align*}
\end{lemma}
\begin{proof}
Consider any round $t > t_0$ before the test in Line~\ref{line::exploit_test_one} or~\ref{line::exploit_test_two} triggers. 
The total reward in relevant rounds can be decomposed into epochs $e = 0, \dots, j(t)$ as follows where $j(t) = \max\{ e \in \bbN \colon t_e < t\}$ is the epoch of round $t$ and $t'_e = \min\{t_{e+1}, t\}$.
\begin{align*}
    \sum_{\ell = t_0 + 1  }^{t} r_\ell(a_\ell, x_\ell)
    = \sum_{e = 0}^{j(t)} \sum_{\ell = t_e + 1}^{t'_{e}} r_\ell(a_\ell, x_\ell)
\end{align*}
Further, let $k'_e = t'_e - t_e$ be the number of rounds in the $e$-th epoch and consider now a single epoch $e$. We can write
\begin{align*}
   & \sum_{\ell = t_e + 1}^{t_e'} r_\ell(a_\ell, x_\ell) \\
   &= \rho^e \tildecrew_1(t_e + 1, t_e') + (1 - \rho^e) \tildecrew_0(t_e + 1, t_e') \\
    &\geq \rho^e \tildecrew_1(t_e + 1, t_e')  + (1 - \rho^e) \left(\tildecrew_1(t_e + 1, t_e') + \wh \Delta k_e' - V(t_e') k_e' \right) \tag{first test did not trigger}\\
        &\geq \tildecrew_1(t_e + 1, t_e')  +(1 - \rho^e) \left( \wh \Delta k_e' - V(t_e')k_e' \right) \\
        &\geq \sum_{\ell = t_e + 1}^{t_e'} r_\ell(a_\ell^1, x_\ell)  +(1 - \rho^e) \left( \wh \Delta k_e' - V(t_e')k_e' \right)
        - \Ocal\left(\sqrt{\frac{k_e'}{\rho^e} \ln \frac{\ln(k_e')}{\delta_e} } + \frac{1}{\rho^e} \ln \frac{\ln(k_e')}{\delta_e}\right)
        \tag{see proof of \pref{lem:exploit_phase_confband}}\\
                &\geq \max_{\pi \in \Pi_{\alg}} \sum_{\ell = t_e + 1}^{t_e'} \bbE_{a \sim \pi}[r_\ell(a, x_\ell)] 
                 - \Ocal\left(R(\Pi_{\alg}) \sqrt{\frac{k_e'}{\rho^e} \ln \frac{k_e'}{\delta_e} }  + \frac{1}{\rho^e} \ln \frac{\ln(k_e')}{\delta_e}\right)
                 \\
                 & \qquad +(1 - \rho^e) \left( \wh \Delta k_e' - V(t_e')k_e' \right)
        \tag{$\alg$ is h-stable}\\
                        &\geq \max_{\pi \in \Pi_{\alg}} \sum_{\ell = t_e + 1}^{t_e'} \bbE_{a \sim \pi}[r_\ell(a, x_\ell)] +(1 - \rho^e)  \wh \Delta k_e'  - \Ocal\left(R(\Pi_{\alg}) \sqrt{\frac{k_e'}{\rho^e} \ln \frac{k_e'}{\delta_e} } + \frac{1}{\rho^e} \ln \frac{\ln(k_e')}{\delta_e}\right)
        \tag{definition of $V(t_e')$}
\end{align*}
Hence, by rearranging terms, we get that the regret in a single epoch is bounded as
\begin{align*}
    &\regret_{\bbM}([t_e + 1, t'_e], \Pi_{\alg}) \leq (\rho^e - 1)  \wh \Delta k_e' + \Ocal\left(R(\Pi_{\alg}) \sqrt{\frac{k_e'}{\rho^e} \ln \frac{k_e'}{\delta_e} } + \frac{1}{\rho^e} \ln \frac{\ln(k_e')}{\delta_e}\right).
\end{align*}
To further upper-bound these terms, we first derive useful expression for the inverse probability of playing $\alg$.
Here, we make the constants in the definition of the initial epoch length $k_0$ explicit. Specifically, we assume that $k_0 = \frac{c_0 R(\Pi_{\alg})^2}{\wh \Delta^2}\ln \frac{2 c_0 R(\Pi_{\alg})}{\wh \Delta \delta}$
where $c_0$ is the absolute constant. We have
\begin{align*}
    \frac{1}{\rho^e} &= 
    \frac{k_e \wh \Delta^2}{c_\rho R(\Pi_{\alg})^2 \ln \frac{k_e}{\delta_e}}
\end{align*}
Plugging this bound on the inverse probability back into the expression for the regret per epoch above, we get with $c$ as the constant in the $\Ocal$ notation
\begin{align*}
    \regret_{\bbM}([t_e + 1, t'_e], \Pi_{\alg}) &\leq (\rho^e - 1)  \wh \Delta k_e' + \wh \Delta  \sqrt{\frac{6 c k_e' k_e}{c_\rho}  } + \frac{6c \wh \Delta^2 k_e}{c_0 R(\Pi_{\alg})^2}\\
    &= \wh \Delta k_e' \left(- 1 + \rho^e + \sqrt{\frac{6 c k_e}{c_\rho k_e'}} + \frac{6c \wh \Delta k_e}{k_e' c_\rho R(\Pi_{\alg})^2}\right)\\
    &\leq \frac{\wh \Delta}{2} \left(- k_e' + \sqrt{\frac{12 c}{c_\rho} k_e k_e'} + \frac{12 c}{ c_\rho } \frac{\wh \Delta}{R(\Pi_{\alg})} k_e\right)
    .\tag{$\rho^e \leq 1/2$}
\end{align*}
In the last step, we used $\rho^e \leq \rho^0 \leq 1/2$ which we can establish by choosing the constants $c_\rho$ and $c_{0}$ in the definition of $\rho$ and $k_0$ appropriately. Specifically,
\begin{align*}
    \rho^e \leq \rho^0 = c_\rho \frac{R(\Pi_{\alg})^2 \ln \frac{k_0}{\delta_0}}{\wh \Delta^2 k_0} =
    \frac{c_\rho}{c_0} \frac{\ln \frac{k_0}{\delta_0}}{\ln \frac{R(\Pi_{\alg})}{\delta \wh \Delta}}
    \leq
    \frac{c_\rho}{c_0} \frac{\ln \frac{c_0 R(\Pi_{\alg})^3}{\delta^2 \wh \Delta^3}}{\ln \frac{c_0 R(\Pi_{\alg})}{\delta \wh \Delta}}
    \leq 3 \frac{c_\rho}{c_0}.
\end{align*}
Thus, choosing $c_0 \geq 6 c_\rho$ is sufficient to ensure $\rho^e \leq 1/2$.
Summing now the bound above over epochs gives
\begin{align*}
    \regret_{\bbM}([t_0 + 1, t], \Pi_{\alg})
    &\leq \frac{\wh \Delta}{2} \left(- t + t_0 + \sqrt{\frac{12 c}{c_\rho}}  \sum_{e=0}^{j(t)} \sqrt{k_e k_e'} + \frac{12 c}{ c_\rho } \frac{\wh \Delta}{R(\Pi_{\alg})} \sum_{e=0}^{j(t)} k_e\right)~.
\end{align*}
We now distinguish between two cases. First consider the case where $j(t) \geq 1$ and assume that $c_\rho \geq  9 \cdot 12c$. Since in this case, the number of rounds $t - t_0$ is at least half the sum of epoch lengths, $\sum_{e=0}^{j(t)} k_e$, we have
\begin{align*}
    \regret_{\bbM}([t_0 + 1, t], \Pi_{\alg})
    &\leq \frac{\wh \Delta}{2} \left(- t + t_0 + \frac{1}{3}  \sum_{e=0}^{j(t)} \sqrt{k_e k_e'} + \frac{1}{9}  \sum_{e=0}^{j(t)} k_e\right)\\
    & \leq \frac{\wh \Delta}{2} \left(- t + t_0 + \frac{4}{9} \sum_{e=0}^{j(t)} k_e\right)
    \leq \frac{\wh \Delta}{2} \left(- t + t_0 + \frac{8}{9} (t - t_0)\right) \leq 0~.
\end{align*}
In the other case, when $j(t) = 0$, we have
\begin{align*}
        \regret_{\bbM}([t_0 + 1, t], \Pi_{\alg})
    &\leq \frac{\wh \Delta}{2} \left(- k_0' + \sqrt{\frac{12 c}{c_\rho}}  \sqrt{k_0 k_0'} + \frac{12 c}{ c_\rho } \frac{\wh \Delta}{R(\Pi_{\alg})}  k_0\right) = \Ocal( \wh \Delta k_0).
\end{align*}
Hence, the regret against $\Pi_{\alg}$ is always bounded as
\begin{align*}
        \regret_{\bbM}([t_0 + 1, t], \Pi_{\alg}) = \Ocal( \wh \Delta k_0).
\end{align*}
\end{proof}

\begin{lemma}\label{lem:exploit_regret_whpi}
The regret of \pref{alg:exploitation_subroutine} against $\wh \pi$ is bounded with probability at least $1 - \Ocal(\delta)$ for all rounds $t$ that the algorithm has not terminated yet as
\begin{align*}
        \regret_{\bbM}([t_0 + 1, t], \{\wh \pi\}) = \Ocal\left( \frac{R(\Pi_{\alg})^2}{\wh \Delta} \left(\ln(t) \ln \frac{t}{\delta}+ \ln \frac{R(\Pi_{\alg})}{\Delta \delta} \right) + \sqrt{(t - t_0) \ln(t) \ln \frac{t}{\delta} }\right).
\end{align*}
\end{lemma}
\begin{proof}
Consider any round $t > t_0$ before the test in Line~\ref{line::exploit_test_one} or~\ref{line::exploit_test_two} triggers. 
The total reward in relevant rounds can be decomposed into epochs $e = 0, \dots, j(t)$ as follows where $j(t) = \max\{ e \in \bbN \colon t_e < t\}$ is the epoch of round $t$ and $t'_e = \min\{t_{e+1}, t\}$.
\begin{align*}
    \sum_{\ell = t_0 + 1  }^{t} r_\ell(a_\ell, x_\ell)
    = \sum_{e = 0}^{j(t)} \sum_{\ell = t_e + 1}^{t'_{e}} r_\ell(a_\ell, x_\ell)
\end{align*}
Further, let $k'_e = t'_e - t_e$ be the number of rounds in the $e$-th epoch and consider now a single epoch $e$. We can write
\begin{align*}
   & \sum_{\ell = t_e + 1}^{t_e'} r_\ell(a_\ell, x_\ell) \\
   &= (1 - \rho^e) \tildecrew_0(t_e + 1, t_e') + \rho^e \tildecrew_1(t_e + 1, t_e')  \\
    &\geq (1 - \rho^e)\tildecrew_0(t_e + 1, t_e') + \rho^e  \left(\tildecrew_0(t_e + 1, t_e')  - 4 \wh \Delta k_e' - V(t_e') k_e' \right) \tag{second test did not trigger}\\
        &\geq \tildecrew_0(t_e + 1, t_e')  - \rho^e \left(  4\wh \Delta k_e' + V(t_e')k_e' \right) \\
        &\geq \sum_{\ell = t_e + 1}^{t_e'} r_\ell(a_\ell^1, x_\ell)   - \rho^e \left(  4\wh \Delta k_e' + V(t_e')k_e' \right)
        - \Ocal\left(\sqrt{\frac{k_e'}{1 - \rho^e} \ln \frac{\ln(k_e')}{\delta_e} } + \frac{1}{1 - \rho^e} \ln \frac{\ln k_e'}{\delta_e}\right)
        \tag{see proof of \pref{lem:exploit_phase_confband}}\\
                &\geq \sum_{\ell = t_e + 1}^{t_e'} \bbE_{a \sim \wh \pi}[r_\ell(a, x_\ell)] 
                  - \rho^e \left(  4\wh \Delta k_e' + V(t_e')k_e' \right)
        - \Ocal\left(\sqrt{\frac{k_e'}{1 - \rho^e} \ln \frac{\ln(k_e')}{\delta_e} } + \frac{1}{1 - \rho^e} \ln \frac{\ln k_e'}{\delta_e}\right).
\end{align*}
The last step here follows since $\rho^e k_e' V(t_e') \leq \Ocal( R(\Pi_{\alg}) \sqrt{\rho^e k_e' \ln \frac{k_e'}{\delta_e}} + \ln \frac{\ln k_e'}{\delta_e})$
By rearranging terms, we can bound the regret against $\wh \pi$ in epoch $e$ as
\begin{align*}
    &\regret_{\bbM}([t_e + 1, t'_e], \{\wh \pi\}) \\
    &\leq \Ocal\left(\rho^e  k_e'(\wh \Delta + V(t_e')) + 
        \sqrt{\frac{k_e'}{1 - \rho^e} \ln \frac{k_e'}{\delta_e} } + \frac{1}{1 - \rho^e} \ln \frac{\ln k_e'}{\delta_e}\right)\\
    &\leq \Ocal\left( \frac{R(\Pi_{\alg})^2}{\wh \Delta} \ln \frac{k_e}{\delta_e} + \sqrt{\frac{k_e'}{1 - \rho^e} \ln \frac{k_e'}{\delta_e} } + \frac{1}{1 - \rho^e} \ln \frac{\ln k_e'}{\delta_e}\right)
    \tag{definition of $\rho^e$, $V(t'_e)$ and $k_e \geq k_e'$}\\    
    &\leq \Ocal\left( \frac{R(\Pi_{\alg})^2}{\wh \Delta} \ln \frac{k_e}{\delta_e} + \sqrt{\frac{k_e'}{1 - \rho^e} \ln \frac{k_e'}{\delta_e} } + \frac{1}{1 - \rho^e} \ln \frac{\ln k_e'}{\delta_e}\right)\\
        &\leq \Ocal\left( \frac{R(\Pi_{\alg})^2}{\wh \Delta} \ln \frac{k_e}{\delta_e} + \sqrt{k_e' \ln \frac{k_e'}{\delta_e} } + \ln \frac{\ln k_e'}{\delta_e}\right)
        \tag{$\rho^e \leq 1/2$}\\
        &\leq \Ocal\left( \frac{R(\Pi_{\alg})^2}{\wh \Delta} \ln \frac{k_e}{\delta_e} + \sqrt{k_e' \ln \frac{k_e'}{\delta_e} }\right).
\end{align*}
Note that $\rho^e \leq 1/2$ holds for appropriate constants in the definition of $k_0$ and $\rho^e$ (see \pref{lem:exploit_regret_polclass}).
We can now sum the regret over all epochs $e$ and get 
\begin{align*}
    &\regret_{\bbM}([t_0 + 1, t], \{\wh \pi\}) \leq \sum_{e=0}^{j(t)}  \Ocal\left( \frac{R(\Pi_{\alg})^2}{\wh \Delta} \ln \frac{k_e}{\delta_e} + \sqrt{k_e' \ln \frac{k_e'}{\delta_e} }\right)\\
    &\leq \sum_{e=0}^{j(t)}  \Ocal\left( \frac{R(\Pi_{\alg})^2}{\wh \Delta} \ln \frac{(j(t) + 1) k_{j(t)}}{\delta} + \sqrt{k_e' \ln \frac{(j(t) + 1) k_{j(t)}}{\delta} }\right)\\
    &\leq \Ocal\left( (j(t) + 1)\frac{R(\Pi_{\alg})^2}{\wh \Delta} \ln \frac{(j(t) + 1) k_{j(t)}}{\delta} + \sqrt{(j(t) + 1) (t - t_0) \ln \frac{(j(t) + 1) k_{j(t)}}{\delta} }\right).
\end{align*}
We now distinguish between two cases. First, $j(t) = 0$, in which case
\begin{align*}
    \regret_{\bbM}([t_0 + 1, t], \{\wh \pi\}) &= \Ocal\left( \frac{R(\Pi_{\alg})^2}{\wh \Delta} \ln \frac{R(\Pi_{\alg})}{\Delta \delta} + \sqrt{ k_0 \ln \frac{R(\Pi_{\alg})}{\wh \Delta \delta} }\right)\\
    &= \Ocal\left( \frac{R(\Pi_{\alg})^2}{\wh \Delta} \ln \frac{R(\Pi_{\alg})}{\wh \Delta \delta}\right)
\end{align*}
and the case where $j(t) > 0$, where  $j(t) = \Ocal(\ln (t - t_0)) = \Ocal( \ln (t))$ and $k_{j(t)} \leq 2 (t - t_0)$
\begin{align*}
    \regret_{\bbM}([t_0 + 1, t], \{\wh \pi\}) = \Ocal\left( \frac{R(\Pi_{\alg})^2}{\wh \Delta} \ln(t - t_0) \ln \frac{t - t_0}{\delta} + \sqrt{(t - t_0) \ln(t - t_0) \ln \frac{t - t_0}{\delta} }\right)\\
     = \Ocal\left( \frac{R(\Pi_{\alg})^2}{\wh \Delta} \ln(t) \ln \frac{t}{\delta} + \sqrt{(t - t_0) \ln(t) \ln \frac{t}{\delta} }\right).
\end{align*}
Combining both cases gives
\begin{align*}
    \regret_{\bbM}([t_0 + 1, t], \{\wh \pi\}) = \Ocal\left( \frac{R(\Pi_{\alg})^2}{\wh \Delta} \left(\ln(t) \ln \frac{t}{\delta}+ \ln \frac{R(\Pi_{\alg})}{\wh \Delta \delta} \right) + \sqrt{(t - t_0) \ln(t) \ln \frac{t}{\wh \delta} }\right).
\end{align*}
\end{proof}

\begin{lemma}[Regret in adversarial environments]\label{lem:arbe_gap_exploit_regregretadversarial}
Assume the absolute constant in the length of the initial epoch $k_0$ and sampling probabilities $\rho^e$ of \pref{alg:exploitation_subroutine} is chosen large enough. Then the regret of \pref{alg:exploitation_subroutine} against $\Pi_{\alg} \cup \{\wh \pi\}$ is bounded with probability at least $1 - \Ocal(\delta)$ for all rounds $t$ that the algorithm has not terminated yet as
\begin{align*}
        \regret_{\bbM}([t_0 + 1, t], \Pi_{\alg} \cup \{\wh \pi\}) = \Ocal\left( \frac{R(\Pi_{\alg})^2}{\wh \Delta} \left(\ln(t) \ln \frac{t}{\delta}+ \ln \frac{R(\Pi_{\alg})}{\wh \Delta \delta} \right) + \sqrt{(t - t_0) \ln(t) \ln \frac{t}{\delta} }\right).
\end{align*}
\end{lemma}
\begin{proof}
By combining \pref{lem:exploit_regret_polclass} and \pref{lem:exploit_regret_whpi}, we have
\begin{align*}
       & \regret_{\bbM}([t_0 + 1, t], \Pi_{\alg} \cup \{\wh \pi\})\\
       & \qquad = \Ocal\left(     \wh \Delta k_0 
    +  \frac{R(\Pi_{\alg})^2}{\wh \Delta} \left(\ln(t) \ln \frac{t}{\delta}+ \ln \frac{R(\Pi_{\alg})}{\wh \Delta \delta} \right) + \sqrt{(t - t_0) \ln(t) \ln \frac{t}{\delta} }\right)
\end{align*}
and by plugging in the definition of $k_0$, we get
\begin{align*}
            \regret_{\bbM}([t_0 + 1, t], \Pi_{\alg} \cup \{\wh \pi\}) &= \Ocal\left( \frac{R(\Pi_{\alg})^2}{\wh \Delta} \left(\ln(t) \ln \frac{t}{\delta}+ \ln \frac{R(\Pi_{\alg})}{\wh \Delta \delta} \right) + \sqrt{(t - t_0) \ln(t) \ln \frac{t}{\delta} }\right).
\end{align*}

\end{proof}

\subsection{Concentration Bounds in Stochastic Environments}
\begin{lemma}\label{lem:regret_pseudoregret_concentration}
In stochastic environments, the regret and pseudo-regret of any algorithm $\alg$ against a policy class $\Pi'$ satisfy with probability at least $1 - \delta$ for all rounds $t \in \bbN$
\begin{align}
\pseudoregret_{\alg}(t, \Pi') - \regret_{\alg}(t, \Pi') = \Ocal \left(\sqrt{t \ln \frac{\ln t}{\delta}} \right)~.
\end{align}
\end{lemma}
\begin{proof}
Let $\pi_\star \in \argmax_{\pi \in \Pi'}  \bbE_{a \sim \pi, x \sim \Dcal}[r(a, x)]$ be the best policy in $\Pi'$. Then
\begin{align*}
    &\pseudoregret_{\alg}(t, \Pi') - \regret_{\alg}(t, \Pi')\\
    &=
    \sum_{\ell=1}^t \bbE_{a \sim \pi_\star, x \sim \Dcal}[r(a, x)]
    - \sum_{\ell=1}^t \bbE_{a \sim \pi_\ell, x \sim \Dcal}[r(a, x)]\\
    &\quad - \left(\max_{\pi' \in \Pi'} \sum_{\ell = 1}^t \bbE_{a \sim \pi'}[r_\ell(a, x_\ell)]
    -
    \sum_{\ell=1}^t r_\ell(a_\ell, x_\ell)
    \right)\\
    &=\sum_{\ell=1}^t r_\ell(a_\ell, x_\ell) - \sum_{\ell=1}^t \bbE_{a \sim \pi_\ell, x \sim \Dcal}[r(a, x)]\\
    &\quad + \sum_{\ell=1}^t \bbE_{a \sim \pi_\star, x \sim \Dcal}[r(a, x)] -  \max_{\pi' \in \Pi'} \sum_{\ell = 1}^t \bbE_{a \sim \pi'}[r_\ell(a, x_\ell)]\\
        &\leq \sum_{\ell=1}^t r_\ell(a_\ell, x_\ell) - \sum_{\ell=1}^t \bbE_{a \sim \pi_\ell, x \sim \Dcal}[r(a, x)]\\
    &\quad + \sum_{\ell=1}^t \bbE_{a \sim \pi_\star, x \sim \Dcal}[r(a, x)] - \sum_{\ell = 1}^t \bbE_{a \sim \pi_\star}[r_\ell(a, x_\ell)] \tag{$\pi_\star \in \Pi'$}\\
            &\leq 2 \times 1.44 \sqrt{\max(2t , 2) \left( 1.4 \ln \ln \left(2 \left(\max\left(\frac{2t}{2} , 1 \right)\right)\right) + \ln \frac{5.2}{\delta}\right)} \tag{\pref{lem:hoeffding_lil}}\\
            &= \Ocal \left(\sqrt{t \ln \frac{\ln t}{\delta}} \right)~.
\end{align*}
Here, the last main step is to apply the time-uniform Hoeffding bound from \pref{lem:hoeffding_lil} to the the two differences individually. Both are martingale sequences that are bounded, i.e., $r_\ell(a_\ell, x_\ell) - \bbE_{a \sim \pi_\ell, x \sim \Dcal}[r(a, x)] \in [-1, 1]$ and $\bbE_{a \sim \pi_\star, x \sim \Dcal}[r(a, x)] - \bbE_{a \sim \pi_\star}[r_\ell(a, x_\ell)] \in [-1, 1]$. 
\end{proof}

\begin{lemma}\label{lem:xr_concentration}
Let $t_0 \in \bbN$ be a possibly random time and let $\{\pi_\ell\}_{\ell \geq t_0}$ be a possibly random sequence of policies $\pi_\ell \colon \Xcal \rightarrow \Delta_\Acal$ so that for all $\ell \geq t_0$, $\pi_\ell$ is independent of $\{(r_j, x_j)\}_{j \geq \ell}$ and $t_0$ is independent of $\{(r_j, x_j)\}_{j \geq t_0}$. Then with probability at least $1 - 2\delta$
\begin{align*}
    \left| \sum_{\ell = t_0 + 1}^{t} \left[\bbE_{a \sim \pi_\ell}[r_\ell(a, x_\ell)]
    - \bbE_{a \sim \pi_\ell, x \sim \Dcal}[r(a, x)] \right] 
    \right|
    &= \Ocal\left(\sqrt{(t - t_0) \ln \frac{\ln (t - t_0)}{\delta}}\right)~.
\end{align*}
\end{lemma}
\begin{proof}
Define the sigma-algebra $\Fcal_\ell = \sigma( \{r_j, x_j, \pi_j\}_{j \leq \ell} \cup \{\pi_\ell, t_0\})$ and let 
\begin{align*}
    Z_\ell = \begin{cases}
    0 & \textrm{ if } \ell \leq t_0\\
    \bbE_{a \sim \pi_\ell}[r_\ell(a, x_\ell)]
    - \bbE_{a \sim \pi_\ell, x \sim \Dcal}[r(a, x)] & \textrm{otherwise.}
    \end{cases}
\end{align*}
The sequence $\{Z_\ell\}_{\ell \in \bbN}$ is a martingale difference sequence w.r.t. $\Fcal_{\ell}$ and $Z_\ell \in [-\indicator{\ell > t_0},+\indicator{\ell > t_0}]$. As a result, we can apply \pref{lem:hoeffding_lil} and get that
\begin{align*}
    \sum_{\ell = t_0 + 1}^{t}& \left[\bbE_{a \sim \pi_\ell}[r_\ell(a, x_\ell)]
    - \bbE_{a \sim \pi_\ell, x \sim \Dcal}[r(a, x)] \right] = \sum_{\ell = 1}^t Z_t
    \\
    &\leq 1.44 \sqrt{\max(2(t-t_0) , 2) \left( 1.4 \ln \ln \left(2 \left(\max\left(\frac{2(t-t_0)}{2} , 1 \right)\right)\right) + \ln \frac{5.2}{\delta}\right)}
    \\
    &= \Ocal \left(\sqrt{(t - t_0) \ln \frac{\ln (t - t_0)}{\delta}}\right)
\end{align*}
holds for all $t$ with probability at least $1 - \delta$. Applying the same argument to $-Z_\ell$ and a union bound completes the proof.
\end{proof}

Notice that, because contexts are i.i.d., a simple anytime Hoeffding bound implies the random variable $\maxrew(t)$ (recall this quantity is defined as the maximum sum of pseudo-expectations over realized contexts) is larger than $t \mathbb{E}_{a \sim \pi_\star(\cdot |x), x\sim \mathcal{D} } \left[ r(a,x)  \right]$ (up to a factor of $\widetilde{\mathcal{O}}(\sqrt{t})$). 
\begin{lemma}\label{lemma::expected_max_rewards_close}
If the environment $\bbB$ is stochastic 
then with probability at least $1-\delta$ 
\begin{equation*}
     t \mathbb{E}_{a \sim \pi_\star(\cdot |x), x\sim \mathcal{D} } \left[ r(a,x)  \right]   \leq \max_{\pi \in \Pi} \sum_{\ell=1}^t \mathbb{E}_{a \sim \pi}\left[ r_\ell( a, x_\ell  )  \right] + \Ocal\left(\sqrt{ t \ln \frac{t}{\delta} } \right)
\end{equation*}
for all $t \in \mathbb{N}$.
\end{lemma}
\begin{proof}
Consider the martingale sequence $\{ Z_\ell \}_{\ell=1}^\infty$ defined as $Z_\ell =  \mathbb{E}_{a \sim \pi_\star(\cdot |x), x\sim \mathcal{D} } \left[ r(a,x)  \right] - \mathbb{E}_{a \sim \pi_\star(\cdot |x_t) } \left[ r(a,x)  | x_t\right] $. By definition $|Z_\ell| \leq 2$ for all $\ell$. Therefore by an anytime Hoeffding bound 
(Lemma~\ref{lem:hoeffding_lil} in Appendix \ref{sa:additional_lemmas}), with probability at least $1-\delta$,

\begin{equation*}
    \sum_{\ell=1}^t Z_\ell  = \Ocal\left(\sqrt{ t \ln \frac{t}{\delta} } \right)
\end{equation*}

for all $t\in\mathbb{N}$. Thus,

\begin{align*}
     t \mathbb{E}_{a \sim \pi_\star(\cdot |x), x\sim \mathcal{D} } \left[ r(a,x)  \right]   &\leq \sum_{\ell=1}^t \mathbb{E}_{a \sim \pi_\star}\left[ r_\ell( a, x_\ell  )   \right]+ \Ocal\left(\sqrt{ t \ln \frac{t}{\delta} } \right)\\
     &\leq \max_{\pi \in \Pi} \sum_{\ell=1}^t \mathbb{E}_{a \sim \pi}\left[ r_\ell( a, x_\ell  )  \right]+ \Ocal\left(\sqrt{ t \ln \frac{t}{\delta} } \right)
\end{align*}

\end{proof}

Lemma~\ref{lemma::expected_max_rewards_close} implies that an algorithm that an algorithm that competes with $\maxrew(t) $ in turn can compete against $ t \mathbb{E}_{a \sim \pi_\star(\cdot |x), x\sim \mathcal{D} } \left[ r(a,x)  \right]$. This fact will help us argue that an adversarial algorithm has good performance in a stochastic environment.

Let's show that when the contexts are i.i.d., the counterfactual reward of any algorithm that decides what policy to play before the context is revealed is (up to a $\sqrt{t \ln \frac{t}{\delta}}$ factor) upper bounded by $t \mathbb{E}_{a \sim \pi_\star(\cdot |x), x\sim \mathcal{D} } \left[ r(a,x)  \right]$,

\begin{lemma}\label{lemma::fixed_pi_star_optimal}
If the environment is stochastic with i.i.d. contexts and the algorithm decides on its policy $\pi_t$ before observing context $x_t$ then with probability at least $1-\delta$,
\begin{equation*}
     \sum_{\ell=1}^t \mathbb{E}_{a \sim \pi_\ell(\cdot | x_\ell)}\left[r(a, x_\ell) | x_\ell\right]   \leq t \mathbb{E}_{a \sim \pi_\star(\cdot |x), x\sim \mathcal{D} } \left[ r(a,x)  \right] + \Ocal\left(\sqrt{ t \ln \frac{t}{\delta} } \right)
\end{equation*}
for all $t \in \mathbb{N}$.
\end{lemma}

\begin{proof}
Similar to the proof of Lemma~\ref{lemma::expected_max_rewards_close}, consider the martingale sequence $\{ Z_\ell \}_{\ell=1}^\infty$ defined as $Z_\ell = \mathbb{E}_{a \sim \pi_\ell(\cdot | x_\ell) } \left[ r(a, x_\ell)  | x_\ell \right] - \mathbb{E}_{a \sim \pi_\ell( \cdot | x), x \sim \mathcal{D} } \left[ r(a, x)\right]$ . By definition $|Z_\ell| \leq 2$. Therefore by the anytime Hoeffding bound of Lemma~\ref{lem:hoeffding_lil}, with probability at least $1-\delta$
\begin{equation*}
    \sum_{\ell=1}^t Z_\ell = \sum_{\ell=1}^t \mathbb{E}_{a \sim \pi_\ell(\cdot | x_\ell)}\left[r(a, x_\ell) | x_\ell\right] - \sum_{\ell=1}^t  \mathbb{E}_{a \sim \pi_\ell(\cdot |x), x\sim \mathcal{D} } \left[ r(a,x)  \right] = \Ocal\left(\sqrt{ t \ln \frac{t}{\delta} } \right)
\end{equation*}
for all $t\in\mathbb{N}$. 
By definition $\pi_\star$ is the policy satisfying $\pi_\star = \argmax_{\pi} \mathbb{E}_{a \sim \pi, x \sim \mathcal{D} } \left[ r(a,x) \right]$ and therefore for any $\pi_\ell$, the inequality $\mathbb{E}_{a \sim \pi_\ell, x \sim \mathcal{D} } \left[ r(a,x) \right] \leq \mathbb{E}_{a \sim \pi_\star, x \sim \mathcal{D} } \left[ r(a,x) \right]$ holds. 
\end{proof}

Lemma~\ref{lemma::fixed_pi_star_optimal} implies that in the case of i.i.d. contexts a learner that selects a policy based on historical data (that is the learner selects a policy to play before the context is revealed) cannot do substantially better than playing $\pi_\star = \argmax_{\pi \in \Pi} \mathbb{E}_{a \sim \pi(\cdot |x), x\sim \mathcal{D} } \left[ r(a,x)  \right] $ during all timesteps. This will prove helpful when deriving bounds for the gap estimation phase.

\section{Additional Technical Lemmas}\label{sa:additional_lemmas}

\begin{lemma}[Time-uniform Hoeffding bound]\label{lem:hoeffding_lil}
Let $S_t = \sum_{i=1}^t Y_t$ be a martingale sequence w.r.t. some sigma algebra $\cF_t$ and let $Y_t \in [a_t, b_t]$ almost surely for $a_t, b_t$ measurable in $\cF_t$. Then with probability at least $1 - \delta$ for all $t \in \mathbb N$
\begin{align*}
    S_t &\leq  1.44 \sqrt{\max(W_t , m) \left( 1.4 \ln \ln \left(2 \left(\max\left(\frac{W_t}{m} , 1 \right)\right)\right) + \ln \frac{5.2}{\delta}\right)}
\end{align*}
where $W_t = \sum_{i=1}^t \frac{(b_i - a_i)^2}{4}$ and $m > 0$ arbitrary but fixed.
\end{lemma}
\begin{proof}
By entry ``Hoeffding I'' in Table~3 of \citet{howard2018uniform}, $S_t$ is a sub-$\psi_N$ process with variance process $W_t$. Further, by Proposition~2 in \citet{howard2020time}, this implies that it is also a sub-$\psi_P$ process with $c = 0$. We now apply \pref{lem:howard_bound} to achieve the desired result. 
\end{proof}

\begin{lemma}[Time-uniform Bernstein bound]
\label{lem:bernstein_lil}
Let $S_t = \sum_{i=1}^t Y_i$ be a martingale sequence w.r.t. a sigma algebra $\cF_t$ and let $Y_t \leq c$ a.s. for some parameter $c > 0$. Then with probability at least $1 - \delta$ for all $t \in \mathbb{N}$
\begin{align*}
    S_t &\leq  1.44 \sqrt{\max(W_t , m) \left( 1.4 \ln \ln \left(2 \left(\max\left(\frac{W_t}{m} , 1 \right)\right)\right) + \ln \frac{5.2}{\delta}\right)}\\
   & \qquad + 0.41 c  \left( 1.4 \ln \ln \left(2 \left(\max\left(\frac{W_t}{m} , 1\right)\right)\right) + \ln \frac{5.2}{\delta}\right)
\end{align*}
where $W_t = \sum_{i=1}^t \bbE[Y_i^2 | \cF_i]$ and $m > 0$ is arbitrary but fixed.
\end{lemma}
\begin{proof}
By entry ``Bennett'' in Table~3 of \citet{howard2018uniform}, $S_t$ is a sub-$\psi_P$ process with variance process $W_t$ and parameter $c = 0$. We now apply \pref{lem:howard_bound} to achieve the desired result. 
\end{proof}

\begin{lemma}[General concentration result]
\label{lem:howard_bound}
In the terminology of \citet{howard2018uniform}, let $S_t = \sum_{i=1}^t Y_i$ be a sub-$\psi_P$ process with parameter $c = 0$ and variance process $W_t$. Then with probability at least $1 - \delta$ for all $t \in \mathbb{N}$
\begin{align*}
    S_t &\leq  1.44 \sqrt{\max(W_t , m) \left( 1.4 \ln \ln \left(2 \left(\max\left(\frac{W_t}{m} , 1 \right)\right)\right) + \ln \frac{5.2}{\delta}\right)}\\
   & \qquad + 0.41 c  \left( 1.4 \ln \ln \left(2 \left(\max\left(\frac{W_t}{m} , 1\right)\right)\right) + \ln \frac{5.2}{\delta}\right)
\end{align*}
where $m > 0$ is arbitrary but fixed.
\end{lemma}
\begin{proof}
This bound follows from \citet{howard2018uniform} Theorem~1 with Equation~10 in that paper by setting $s=1.4$ and $\eta = 2$.
\end{proof}

\begin{lemma}\label{lemma:simplified_freedman}
Suppose $\{ X_t \}_{t=1}^\infty$ is a martingale difference sequence with $| X_t | \leq b$. Let 
\begin{equation*}
    \mathrm{Var}_\ell(X_\ell) = \mathbf{Var}( X_\ell | X_1, \cdots, X_{\ell-1})
\end{equation*}
Let $V_t = \sum_{\ell=1}^t \mathrm{Var}_\ell(X_\ell)$ be the sum of conditional variances of $X_t$.  Then we have that for any $\delta' \in (0,1)$ and $t \in \mathbb{N}$
\begin{equation*}
    \mathbb{P}\left(  \sum_{\ell=1}^t X_\ell >    2\sqrt{V_t}A_t + 3b A_t^2 \right) \leq \delta'~,
\end{equation*}
where $A_t = \sqrt{2 \ln \ln \left(2 \left(\max\left(\frac{V_t}{b^2} , 1\right)\right)\right) + \ln \frac{6}{\delta'}}$.
\end{lemma}

\begin{proof}
We are in a position to use \ref{lem:bernstein_lil} (with $c = b$). Let $S_t = \sum_{\ell=1}^t X_t$ and $W_t = \sum_{\ell=1}^t \mathrm{Var}_\ell(X_\ell)$. Let's set $m = b^2$. It follows that with probability $1-\delta'$ for all $t \in \mathbb{N}$

\begin{align*}
    S_t &\leq  1.44 \sqrt{\max(W_t , b^2) \left( 1.4 \ln \ln \left(2 \left(\max\left(\frac{W_t}{b^2} , 1 \right)\right)\right) + \ln \frac{5.2}{\delta'}\right)}\\
   & \qquad + 0.41 b  \left( 1.4 \ln \ln \left(2 \left(\max\left(\frac{W_t}{b} , 1\right)\right)\right) + \ln \frac{5.2}{\delta'}\right)\\
   &\leq 2 \sqrt{\max(W_t , b^2) \left( 2 \ln \ln \left(2 \left(\max\left(\frac{W_t}{b^2} , 1 \right)\right)\right) + \ln \frac{6}{\delta'}\right)}\\
   & \qquad + b  \left( 2 \ln \ln \left(2 \left(\max\left(\frac{W_t}{b^2} , 1\right)\right)\right) + \ln \frac{6}{\delta'}\right) \\
   &= 2\max(\sqrt{W_t}, b)A_t + bA_t^2\\
   &\leq 2\sqrt{W_t}A_t + 2bA_t + bA_t^2\\
   &\stackrel{(i)}{\leq} 2\sqrt{W_t}A_t + 3b A_t^2~,
\end{align*}
where $A_t = \sqrt{2 \ln \ln \left(2 \left(\max\left(\frac{W_t}{b^2} , 1\right)\right)\right) + \ln \frac{6}{\delta'}}$. Inequality $(i)$ follows because $A_t \geq 1$.

Setting $V_t = W_t$ we conclude the proof.
\end{proof}

In turn, a corollary of the previous lemma is the following.

\begin{lemma}\label{lemma:simplified_freedman_norm}
Suppose $\{ X_t \}_{t=1}^\infty$ is a martingale difference sequence with $| X_t | \leq b_t$ with $b_t$ a non-decreasing deterministic sequence. Let 
\begin{equation*}
    \mathrm{Var}_\ell(X_\ell) = \mathbf{Var}( X_\ell | X_1, \cdots, X_{\ell-1})
\end{equation*}
Let $V_t = \sum_{\ell=1}^t  \mathrm{Var}_\ell(X_\ell) 
$ be the sum of conditional variances of $X_t$. Then for any $\delta \in (0,1)$ and $t \in \mathbb{N}$ we have
\begin{equation*}
    \mathbb{P}\left(  \sum_{\ell=1}^t X_\ell >    2\sqrt{V_t}A_t + 3B_t A_t^2 \right) \leq \delta~,
\end{equation*}
where
$A_t(\delta) =2 \sqrt{ \ln \frac{12t^2}{\delta }}$.
\end{lemma}
\begin{proof}
For any $t$ define a martingale difference sequence $\{X_\ell^{(t)}\}_{\ell =1}^\infty$ as follows:
\begin{equation*}
    X_\ell^{(t)} = \begin{cases}
        X_{\ell} &\text{ if } \ell \leq t\\
        0 &\text{otherwise }
        \end{cases}
\end{equation*} 
We apply Lemma~\ref{lemma:simplified_freedman} with parameter $\delta' = \frac{\delta}{2t^2}$ and $b = B_t$, and then overapproximate. A union bound over $t \in \mathbb{N}$ gets the desired result. 
\end{proof}

\begin{lemma}\label{lemma::super_simplified_freedman_norm}
Suppose $\{ X_t \}_{t=1}^\infty$ is a martingale difference sequence with $| X_t | \leq b_t$ with $b_t$ a non-decreasing deterministic sequence. Let 
\begin{equation*}
    \mathrm{Var}_\ell(X_\ell) = \mathbf{Var}( X_\ell | X_1, \cdots, X_{\ell-1})
\end{equation*}
Let $V_t = \sum_{\ell=1}^t  \mathrm{Var}_\ell(X_\ell) 
$ be the sum of conditional variances of $X_t$. Then we have that for any $\delta \in (0,1)$ and $t \in \mathbb{N}$
\begin{equation*}
    \mathbb{P}\left(  \sum_{\ell=1}^t X_\ell >    4\sqrt{V_t \ln \frac{12t^2}{\delta}} + 12b_t \ln \frac{12t^2}{\delta} \right) \leq \delta~,
\end{equation*}

\end{lemma}

\begin{proof}

Notice that by definition $V_t(\mathbf{a}) \leq tb_t^2$. Therefore $\max( \frac{V_t(\mathbf{a})}{b^2_t} , 1) \leq t  $ and

\begin{equation*}
    A_t(\delta') \leq \sqrt{ 2\ln \ln 2t + \ln \frac{12t^2}{\delta'}  }\leq 2\sqrt{\ln \frac{12t^2}{\delta'}} := \tilde{A}_t(\delta')~. 
\end{equation*}
Substituting this upper bound in the statement of Lemma~\ref{lemma:simplified_freedman_norm} yields the result. 
\end{proof}

\begin{lemma}\label{lem:concentration_threshold}
Let $c, \delta \in (0, 1]$  and $t \geq \frac{16}{c^2} \ln^2 \frac{2}{c\delta}$. Then
\begin{align*}
    \frac{\ln \frac{t}{\delta}}{\sqrt{t}} \leq c
\end{align*}
\end{lemma}
\begin{proof}
The function $f(t) = \frac{\ln \frac{t}{\delta}}{\sqrt{t}}$ is non-increasing on $(e^2 \delta, \infty)$. Thus, it is sufficient to show that the inequality holds for $t = 16\frac{\ln^2 \frac{2}{c\delta}}{c^2}$. Since $\ln(x) \leq \frac{x}{2}$ for all $x \in \bbR_+$, we can upper-bound this value of $t$ as
\begin{align*}
    t \leq \frac{16}{c^4 \delta^2}.
\end{align*}
This implies
\begin{align*}
    \ln \frac{t}{\delta} \leq \ln \frac{16}{c^4 \delta^4} = 4 \ln \frac{2}{c \delta} 
    = c \sqrt{16 \frac{\ln^2 (2 / c / \delta)}{c^2}} = c \sqrt{t}
\end{align*}
which proves the claim.
\end{proof}

\begin{lemma}\label{lem:concentration_threshold2}
Let $c, \delta \in (0, 1]$  and $t \geq \frac{4}{c} \ln \frac{2}{c\delta}$. Then
\begin{align*}
    \frac{\ln \frac{t}{\delta}}{t} \leq c
\end{align*}
\end{lemma}
\begin{proof}
The function $f(t) = \frac{\ln \frac{t}{\delta}}{t}$ is non-increasing on $(e \delta, \infty)$. Thus, it is sufficient to show that the inequality holds for $t = 4\frac{\ln \frac{2}{c\delta}}{c}$. Since $\ln(x) \leq \frac{x}{2}$ for all $x \in \bbR_+$, we can upper-bound this value of $t$ as
\begin{align*}
    t \leq \frac{4}{c^2 \delta}.
\end{align*}
This implies
\begin{align*}
    \ln \frac{t}{\delta} \leq \ln \frac{4}{c^2 \delta^2} = 2 \ln \frac{2}{c \delta} 
    = \frac{c}{2} t \leq c t
\end{align*}
which proves the claim.
\end{proof}

\end{document}